\documentclass[twoside]{article}
\usepackage[accepted]{aistats2022}
\usepackage[T1]{fontenc}
\usepackage[round]{natbib}
\usepackage{booktabs} 
\usepackage{amsfonts}
\usepackage{graphicx}
\usepackage{subcaption}
\usepackage{float} 
\usepackage{amsthm}
\usepackage{bm}
\usepackage{dsfont}
\usepackage{algorithm}
\usepackage{algorithmic}
\usepackage{thm-restate}
\usepackage{color, colortbl}
\usepackage{xcolor}
\newtheorem{theorem}{Theorem}[section]
\newtheorem{lemma}[theorem]{Lemma}

\begin{document}


\runningauthor{Jinhang Zuo, Xutong Liu, Carlee Joe-Wong, John C.S. Lui, Wei Chen}

\twocolumn[
\aistatstitle{Online Competitive Influence Maximization}
% \aistatsauthor{Jinhang Zuo\textsuperscript{1}, Xutong Liu\textsuperscript{2}, Carlee Joe-Wong\textsuperscript{1}, John C.S. Lui\textsuperscript{2}, Wei Chen\textsuperscript{3}}
% \aistatsaddress{\textsuperscript{1}Carnegie Mellon University, \textsuperscript{2}The Chinese University of Hong Kong, \textsuperscript{3}Microsoft Research}]

\aistatsauthor{ Jinhang Zuo \And Xutong Liu \And Carlee Joe-Wong}
\aistatsaddress{ Carnegie Mellon University \And  The Chinese University of Hong Kong \And Carnegie Mellon University} 
\aistatsauthor{John C.S. Lui \And Wei Chen}
\aistatsaddress{The Chinese University of Hong Kong \And Microsoft Research} 
]

\begin{abstract}
Online influence maximization has attracted much attention as a way to maximize influence spread through a social network while learning the values of unknown network parameters. Most previous works focus on single-item diffusion. In this paper, we introduce a new Online Competitive Influence Maximization (OCIM) problem, where two competing items (e.g., products, news stories) propagate in the same network and influence probabilities on edges are unknown. We adopt a combinatorial multi-armed bandit (CMAB) framework for OCIM, but unlike the non-competitive setting, the important monotonicity property (influence spread increases when influence probabilities on edges increase) no longer holds due to the competitive nature of propagation, which brings a significant new challenge to the problem. We provide a nontrivial proof showing that the Triggering Probability Modulated (TPM) condition for CMAB still holds in OCIM, which is instrumental for our proposed algorithms OCIM-TS and OCIM-OFU to achieve sublinear Bayesian and frequentist regret, respectively. We also design an OCIM-ETC algorithm that requires less feedback and easier offline computation, at the expense of a worse frequentist regret bound. Experimental evaluations demonstrate the effectiveness of our algorithms.
\end{abstract}

\section{Introduction}\label{sec:intro}
Influence maximization, motivated by viral marketing applications, has been extensively studied since \cite{kempe2003maximizing} formally defined it as a stochastic optimization problem: given a social network $G$ and a budget $k$, how should a set of $k$ seed nodes in $G$ be chosen such that the expected number of final activated nodes under a given diffusion model is maximized? They proposed the well-known Independent Cascade (IC) and Linear Threshold (LT) diffusion models, and gave a greedy algorithm that outputs a $(1-1/e-\epsilon)$-approximate solution for any $\epsilon>0$.
However, they only considered a single item (e.g., product, idea) propagating in the network. In reality, different items could propagate concurrently in the same network, interfering with each other and leading to competition during propagation. 
Several competitive diffusion models~\citep{carnes2007maximizing,bharathi2007competitive, budak2011limiting,he2012influence,ivanov2017content} have been proposed for this setting. We use a Competitive Independent Cascade (CIC) model \citep{chen2013information}, which extends the classical IC model to multi-item influence diffusion. 
We consider the competitive influence maximization problem between two items from the ``follower's perspective'': 
given the seed nodes of the competitor's item, the follower's item chooses a set of nodes so as to maximize the expected number of nodes activated by the follower's item, referred to as the influence spread of the item. 

\begin{table*}[t]
	\caption{Summary of the proposed algorithms.}	\resizebox{2\columnwidth}{!}{
	\label{tab:algorithms}
	\begin{tabular}{ccccc}
		\toprule
		\textbf{Algorithm}& \textbf{No Prior?}& \textbf{Offline computation} & \textbf{Feedback} & \textbf{Regret}\\
		\midrule
		\textit{OCIM-TS} & $\times$ & Standard & Full propagation & Bayes. $O(\sqrt{T\ln T})$\\
		\textit{OCIM-OFU} & \checkmark & Hard & Full propagation & Freq. $O(\sqrt{T\ln T})$\\
		\textit{OCIM-ETC} & \checkmark & Standard & Direct out-edges & Freq.  $O(T^{\frac{2}{3}} (\ln T )^{\frac{1}{3}})$ \\
		\bottomrule
	\end{tabular}
}
\end{table*}

We refer to the above problem as ``offline'' competitive influence maximization, since the influence probabilities on edges, i.e., the probabilities of an item’s propagation along edges, are known in advance. 
It can be solved by a greedy algorithm due to submodularity~\citep{chen2013information}. 
However, in many real-world applications, the influence probabilities on edges are unknown. 
We study the competitive influence maximization in this setting, and call it the Online Competitive Influence Maximization (OCIM) problem. 
In OCIM, the influence probabilities on edges need to be learned through repeated influence maximization trials: in each round, given the seed nodes of the competitor, we (i)  choose $k$ seed nodes; (ii) observe the resulting diffusion that follows the CIC model to update our knowledge of the edge probabilities; and (iii) obtain a reward, which is the total number of nodes activated by our item. Our goal is to choose the seed nodes in each round based on previous observations so as to maximize the cumulative reward. %\carlee{I think you should start a new paragraph here and then say in the next paragraph that many works have considered the non-competitive framework using CMAB.}\jinhang{Revised}

Most previous studies on the online non-competitive influence maximization problem use a combinatorial multi-armed bandit (CMAB) framework~\citep{chen2016combinatorial,wen2017online}, an extension of the classical multi-armed bandit problem that captures the tradeoff between exploration and exploitation in sequential decision making. 
In CMAB, a player chooses a combinatorial action to play in each round, observes a set of arms triggered by this action and receives a reward. 
The player aims to maximize her cumulative reward over multiple rounds, navigating a tradeoff between exploring unknown actions/arms and exploiting the best known action. %Besides the exploration and exploitation tradeoff, 
CMAB algorithms must also deal with an exponential number of possible combinatorial actions, which makes exploring all actions infeasible. 
% CMAB has been applied to non-competitive online influence maximization~\cite{chen2016combinatorial,wang2017improving,wen2017online}.
%Chen et al.~\cite{chen2016combinatorial} proposed a extension of the Upper Confidence Bound (UCB) algorithm~\cite{auer2002finite} for CMAB, while Wang \& Chen~\cite{wang2017improving} improved the regret bound of CMAB with probabilistically triggered arms. Both works focus on non-competitive settings. 

{\bf Our Contributions.} To the best of our knowledge, we are the first to study the online competitive influence maximization problem. We introduce a general contextual combinatorial multi-armed bandit framework with probabilistically triggered arms (C$^2$MAB-T) for OCIM.
Within this framework, OCIM presents a new challenge: the key monotonicity property (influence spread increases when influence probabilities on edges increase) no longer holds due to the competitive nature of propagation, and thus upper confidence bound (UCB) based algorithms~\citep{chen2016combinatorial,wen2017online} cannot be directly applied to OCIM. Such non-monotonicity also complicates the analysis of the important Triggering Probability Modulated (TPM) condition for CMAB \citep{wang2017improving}, and we provide a non-trivial new proof to show it still holds for OCIM. 
We are the first to identify the OCIM problem as a natural CMAB problem without monotonicity and tackle it from three directions, providing three solutions with different tradeoffs, as shown in Table~\ref{tab:algorithms}: OCIM-TS uses standard offline oracles to achieve good Bayesian regret, but requires prior knowledge of edge probabilities; OCIM-OFU has a stronger frequentist regret bound without prior knowledge, but requires harder offline computation; and OCIM-ETC uses standard offline oracles and fewer observations, but leads to a worse frequentist regret bound. None is a perfect solution for 
OCIM, but we believe their tradeoffs shed light on the challenges involved in solving OCIM and even general CMAB problems without monotonicity. Our regret analysis of OCIM-TS delicately combines the key property of Thompson Sampling (TS) with the TPM condition to tackle non-monotonicity and allows any benchmark (exact, approximate, or even heuristic) oracle; our analysis of OCIM-OFU and OCIM-ETC extends the analysis for CMAB to a new contextual setting (C$^2$MAB-T) where the contexts are defined as the feasible sets of super arms and are not bonded with base arms.
% {\color{blue}
We also discuss the extension of our framework to settings with more complex competitor actions.
% } 
% \carlee{This is a bit confusing as you haven't said yet how you model the competitor: maybe just say that your framework can be extended to more complex competitor actions?}
Experiments on two real-world datasets demonstrate the effectiveness of our proposed algorithms. Due to the space constraint, we discuss important insights of our proofs and move the complete proofs as well as the results for the general C$^2$MAB-T problem to the Appendix.

{\bf Related Work.} 
\cite{kempe2003maximizing} formally defined the influence maximization problem in their seminal work. 
Since then, the problem has been extensively studied~\citep{LiFWT18}. 
\cite{borgs2014maximizing} presented a breakthrough approximation algorithm that runs in near-linear time, which was improved by 
a series of algorithms~\citep{tang2015influence,NguyenTD16,TangTXY18}.
% that run in $O((k+l)(m+n)\log n / \epsilon^{2})$ expected time for a budget $k$ and a graph with $n$ nodes and $m$ edges, and return a $(1-1/e-\epsilon)$-approximate solution with probability $1-n^l$.
A number of studies~\citep{carnes2007maximizing,bharathi2007competitive,budak2011limiting,he2012influence,lin2015analyzing,ivanov2017content} addressed competitive influence maximization problems where multiple competing sources propagate in the same network. \cite{carnes2007maximizing} proposed the distance-based and wave propagation models, and considered the influence maximization problem from the follower’s perspective. \cite{bharathi2007competitive} considered the CIC model and gave an algorithm for computing the best response to an opponent’s strategy.

When the influence probabilities of edges are unknown, the non-competitive online influence maximization problem has been extensively studied~\citep{chen2016combinatorial,wang2017improving,wen2017online,IMFB2019,vaswani2017model,perrault2020budgeted}.
\cite{chen2016combinatorial} studied the problem under the IC model and proposed a general CMAB framework. We introduce a new contextual extension of CMAB, called C$^2$MAB-T, different from the contextual CMAB studied by~\citet{chen2018contextual} and \citet{qin2014contextual}: they consider the context features of all base arms and assume the action space of super arms is a subset of all base arms, while we consider the feasible set of super arms as the context, which is more flexible than a subset of all base arms.
\cite{wang2017improving} introduced a triggering probability modulated (TPM) bounded smoothness condition to remove an undesired factor in the regret bound of~\citet{chen2016combinatorial}. 
% Wen et al.~\cite{wen2017online} and Wu et al.~\cite{IMFB2019} further considered edge probabilities represented by latent feature vectors, which is useful for large-scale settings.
% Vaswani et al.~\cite{vaswani2017model} considered a surrogate function as an approximation to the original influence maximization objective, but this heuristic function does not provide a theoretical guarantee.
\cite{perrault2020budgeted} introduced a budgeted online influence maximization framework, where marketers optimize their seed sets under a budget rather than a cardinality constraint.
Our OCIM-TS algorithm is similar to the Combinatorial Thompson Sampling (CTS) algorithm of \citet{wang2018thompson}. However, CTS requires an exact oracle and has frequentist regret bound, while OCIM-TS allows any benchmark oracle and has Bayesian regret bound. \cite{huyuk2020thompson} studied the Bayesian regret of CTS for CMAB, but they also require an exact oracle and a monotonicity assumption that does not hold for OCIM.
Our Bayesian regret analysis is also different from that of \citet{russo2016information}: they only study a simple special CMAB problem, while we provide the regret bound for general C$^2$MAB-T instances, including the OCIM problem.

% Wang \& Chen~\cite{wang2018thompson} proposed a Combinatorial Thompson Sampling (CTS) algorithm for CMAB: it requires an exact oracle and has frequentist regret bound, while our OCIM-TS algorithm allows any benchmark oracle and achieves logarithmic Bayesian regret. Hüyük \& Tekin~\cite{huyuk2020thompson} studied the Bayesian regrets of CTS for CMAB, but they also require an exact oracle and a monotonicity assumption.
% Our Bayesian regret analysis is also different from that in \cite{russo2016information}: they only study a simple special CMAB problem, while we provide the regret bound for general C$^2$MAB-T instances, including the OCIM problem.
\section{OCIM Formulation}\label{sec:model}
In this section we present the formulation of OCIM. We first introduce the traditional competitive influence maximization problem, and then discuss its online extension where edge probabilities are unknown. %\carlee{``tasks'' seems awkward; maybe ``trials'' or ``observations''?}
%\wei{Influence maximization in general is an optimization task, so it is ok here in general, but using plural making it a bit awkward. so I put it as repeated runs of the task.}
\subsection{Competitive Independent Cascade Model} 
We consider a Competitive Independent Cascade (CIC) model, which is an extension of the classical IC model to multi-item influence diffusion. A network is modeled as a directed graph $G = (V, E)$ with $n = |V|$ nodes and $m = |E|$ edges. 
Every edge $(u,v) \in E$ is associated with a probability $p(u,v)$. There are two items, $A$ and $B$, trying to propagate in $G$ from their own seed sets $S_A$ and $S_B$. 
%\carlee{We consider discrete time steps $t = 0,1,2,\ldots$ I suggest replacing $t$ with $s$ since you use $t$ to index rounds below.} 
%\wei{Agree. Let's fix the terminology and the notation: for online, we use rounds and $t=1,2,\ldots$, for CIC propagation within a round, we use steps and $s=1,2,\ldots$.}
The influence propagation runs as follows: nodes in $S_A$ (resp. $S_B$)  are activated by $A$ (resp. $B$) at step $0$; at each step $s\ge 1$, a node $u$ activated by $A$ (resp. $B$) in step $s-1$ 
	tries to activate each of its inactive out-neighbors $v$ to be $A$ (resp. $B$) with an independent probability $p(u,v)$ that is the same for $A$ and $B$ (i.e., we consider a homogeneous CIC model). The homogeneity assumption is reasonable since typically $A$ and $B$ are two items of the same category (thus competing), so they are likely to have similar propagation characteristics.
%\carlee{rephrase this to make clear it applies to each out-neighbor of $u$.}

% {\color{blue}
If two in-neighbors of $v$ activated by $A$ and $B$ respectively both successfully activate $v$ at step $s$, then a tie-breaking rule is applied at $v$ to determine the final adoption.
%If both $A$ and $B$ are attempting to activate $v$ at step $t$ (i.e., there exists two in-neighbors of $v$ activated by $A$ and $B$ respectively in step $t-1$), it causes a competition and $v$ will adopt one of them depending on a tie-breaking rule. 
% In this paper, we consider two dominance tie-breaking rules: $A>B$, which means $v$ will always adopt $A$ in a competition, and $B>A$, which means $v$ will always adopt $B$ in a competition. 
In this paper, we consider two types of tie-breaking rules: dominance~\citep{budak2011limiting} and proportional~\citep{chen2011influence} tie-breaking rules. Dominance tie-breaking with $A>B$ (resp. $B>A$) means $v$ will always adopt $A$ (resp. $B$) in a competition. Proportional tie-breaking means that if there are $n_A$ in-neighbors activated by $A$ and $n_B$ in-neighbors activated by $B$ trying to activate $v$ at the same step, the probability that $v$ adopts $A$ (resp. $B$) is $\frac{n_A}{n_A+n_B}$ (resp. $\frac{n_B}{n_A+n_B}$).
The same tie-breaking rule also applies to the case when a node $u$ is selected both as an $A$-seed and a $B$-seed. 
% The dominance tie-breaking rule reflects scenarios such as a novel technology dominating the old technology, or negative information dominating positive information, which is reasonable in practice. 
The process stops when no nodes activated at a step $s$ have inactive out-neighbors.
% }

We consider the follower's perspective in the optimization task: let $A$ be the follower and $B$ be the competitor. 
Then given $S_B$, our goal is to choose at most $k$ seed nodes in $G$ as $S_A$ to maximize the influence spread of $A$, denoted as $\sigma_A(S_A, S_B)$, which is the expected number of nodes activated by $A$ after the propagation ends.
%Notice that we assume that the CIC model is homogeneous, where $p(u,v)$ is the same for $A$ and $B$.
According to \citet{budak2011limiting}'s result, the above optimization task under the homogeneous CIC model with the dominance tie-breaking rule has the monotone and submodular properties, and thus can be approximately solved by a greedy algorithm.

% \carlee{Some justification of the dominance and homogeneity (e.g., why these occur in practice) would be good to include if there's room.}
% \wei{add a justification.}

\subsection{OCIM Model}
In the online competitive influence maximization (OCIM) problem, the edge probabilities $p(u, v)$'s are unknown and need to be learned: in each round $t$, given $S_B^{(t)}$, we can choose up to $k$ seed nodes as $S_A^{(t)}$, observe the whole propagation of $A$ and $B$ that follows the CIC model, and obtain the reward, which is the number of nodes finally activated by $A$ in this round. 
The propagation feedback observed is then used to update the estimates on edge probabilities $p(u,v)$'s, so that we can achieve better influence maximization results in subsequent rounds.
Our goal is to accumulate as much reward as possible through this repeated process
	over multiple rounds. 
%\carlee{specify that this process is repeated from scratch (starting from seed nodes on an uninfluenced network). Also, the transition to the next sentence is pretty abrupt; maybe rephrase as ``To do so, we express OCIM in the CMAB-T framework...''}
%\wei{revised.}

We introduce a new contextual combinatorial multi-armed bandit framework with probabilistically triggered arms (C$^2$MAB-T) for the OCIM problem, which is a contextual extension of CMAB-T from \citet{wang2017improving}. In OCIM, the set of edges $E$ is the set of (base) arms $[m]=\{1, ..., m\}$, and their outcomes follow $m$ independent Bernoulli distributions with expectation $\mu_e = p(u,v)$ for all $e = (u,v) \in E$. We denote the independent samples of arms in round $t$ as $X^{(t)} = (X_1^{(t)},\dots,X_m^{(t)}) \in \{0, 1\}^m$, where $X_i^{(t)}=1$ means the $i$-th edge is on (or live) and $X_i^{(t)}=0$ means the $i$-th edge is off (or blocked) in round $t$, and thus $X^{(t)} $ corresponds to the {\em live-edge graph} 
~\citep{kempe2003maximizing} in round $t$. 
We consider the seed set of the competitor, $S_B^{(t)}$, as the {\em context} in round $t$ since it is determined by the competitor and can affect our choice of $S_A^{(t)}$. We define $\bm{\mathcal{S}}^{(t)} = \left\{S \mid S=(S_A^{(t)}, S_B^{(t)}), |S_A^{(t)}| \leq k \right\}$ as the action space in round $t$ and $S^{(t)}\in \bm{\mathcal{S}}^{(t)}$ as the real action.
We define the triggered arm set $\tau_t$ as the set of edges reached by the propagation from either $S_A^{(t)}$ or $S_B^{(t)}$. 
Thus, $\tau_t$ is the set of edges $(u,v)$ where $u$ can be reached from $S^{(t)}$ by passing through only edges $e \in E$ with $X_e^{(t)} = 1$. The outcomes of $X_i^{(t)}$ for all $i\in\tau_t$ are observed as the feedback. 
% Notice that although $A$ and $B$ may compete in the propagation, $\tau_t$ is not affected as long as $S_A^{(t)} \cup S_B^{(t)}$ remains the same.%\carlee{, due to the dominance tie-breaking?}\jinhang{It is true for any tie-breaking rules in homogeneous CIC model}
We denote the obtained reward in round $t$ as $R(S^{(t)},X^{(t)})$, which is the number of nodes finally activated by $A$. The expected reward $r_{S^{(t)}}(\bm{\mu})=\mathbb{E}[R(S^{(t)}, X^{(t)})]$ is a function of the action $S^{(t)}$ and the vector $\bm{\mu} = (\mu_1,\dots,\mu_m)$. 
% {\color{blue} 
Note that our framework can also handle dynamic tie-breaking rules over different rounds, by treating the tie-breaking rule as a part of the context. For ease of explanation, we assume a fixed tie-breaking rule in this paper.
% }

The performance of a learning algorithm $\mathcal{A}$ is measured by its expected regret, which is the difference in expected cumulative reward between always playing the best action and playing actions selected by algorithm $\mathcal{A}$. Let $\text{opt}^{(t)}(\bm{\mu}) = \sup_{S_A^{(t)}}r_{S^{(t)}}(\bm{\mu})$ denote the expected reward of the optimal action in round $t$. Since the offline influence maximization under the CIC model is NP-hard~\citep{budak2011limiting}, we assume that there exists an offline $(\alpha,\beta)$-approximation oracle $\mathcal{O}$, which takes $S_B^{(t)}$ and $\bm{\mu}$ as inputs and outputs an action $S^{{\mathcal{O}},(t)}$ such that $\text{Pr}\{r_{S^{{\mathcal{O}},(t)}}(\bm{\mu}) \geq \alpha \cdot \text{opt}^{(t)}(\bm{\mu})) \} \geq \beta$, where $\alpha$ is the approximation ratio and $\beta$ is the success probability. Instead of comparing 
with the exact optimal reward, we use the following $(\alpha,\beta)$-approximation {\em frequentist regret} for $T$ rounds:
\begin{equation}\textstyle
    Reg^{\mathcal{A}}_{\alpha,\beta}(T;\bm{\mu}) = \sum_{t=1}^T  \alpha \cdot \beta \cdot \text{opt}^{(t)}(\bm{\mu}) - \sum_{t=1}^T r_{S^{\mathcal{A},(t)}}(\bm{\mu}),
\end{equation}
where $S^{\mathcal{A},(t)} := (S_A^{\mathcal{A},(t)}, S_B^{(t)})$ is the action chosen by algorithm $\mathcal{A}$ in round $t$. Here $S_B^{(t)}$ is the context and $S_A^{\mathcal{A},(t)}$ is the seed set of item $A$ chosen by algorithm $\mathcal{A}$.

Another way to measure the performance of the algorithm $\mathcal{A}$ is using {\em Bayesian regret}~\citep{russo2014learning}. Denote the prior distribution of $\bm{\mu}$ as $\mathcal{Q}$ (we will discuss how to derive $\mathcal{Q}$ for OCIM in Section~\ref{sec:TS}). When the prior $\mathcal{Q}$ is given, the corresponding Bayesian regret is defined as:
\begin{equation}\textstyle\label{eq:BR}
    BayesReg^{\mathcal{A}}_{\alpha,\beta}(T) = \mathbb{E}_{\bm{\mu}\sim\mathcal{Q}}  Reg^{\mathcal{A}}_{\alpha,\beta}(T;\bm{\mu}).
\end{equation}
% \vspace{-2pt}
We will design algorithms to solve the OCIM problem and bound their achieved Bayesian and frequentist regrets in Section~\ref{sec:TS} and Section \ref{sec:FreqReg}, respectively. We also discuss the general C$^2$MAB-T problem and its solutions in the Appendix.

%\wei{About the offline oracle: one one hand, we can say that by~\cite{chen2013information}, there exists such a oracle with $\alpha=1-1/e-\varepsilon$ and $\beta = 1-1/n$.
%But on the other hand, we are not sure if we are going to use this oracle. At least for the CUCB variant, we need to use another more powerful oracle, and we do not know its approximation ratio.	
%Need to check what we want to say for the new oracle, and the new approximation regret.
%}\jinhang{We introduce a new offline oracle in section 3, so I think it is easy to see the approximation ratio depends on which oracle we are using.}

% Note that $S_B^{(t)}$ allows to change round by round, so the optimal action $S_A^{(t)}$ and the optimal $S_t$ will also change correspondingly. 
\section{Properties of OCIM}\label{sec:properties}
In this section, we first show that the key monotonicity property for CMAB does not hold in OCIM. We then prove that the important Triggering Probability Modulated (TPM) condition still holds, which is essential for the analysis of all proposed algorithms.

%% (Fig. ? shows an example) 
%To remove the requirement of monotonicity, we introduce a new offline oracle in Section \ref{sec:non-monotone} and discuss the design of such an oracle in Section \ref{sec:offline}. The second condition is the Triggering Probability Modulated (TPM) condition. One key contribution of our paper is to show that the TPM condition still holds in the competitive setting. 
%The proof is more challenging than that in non-competitive cases due to non-monotonicity and the competitive propagation process. With the new offline oracle and TPM condition, we propose the Competitive Influence Maximization Optimism In the Face of Uncertainty (CIM-OIFU) algorithm and shows it achieves logarithmic approximation regret. We also design a Competitive Influence Maximization Explore-Then-Commit (CIM-ETC) algorithm that has a worse regret bound than CIM-OIFU, but does not require a new offline oracle and needs fewer observations in each round.
% (Fig. ? shows an example) 
\subsection{Non-monotonicity}\label{sec:non-monotone}
The monotonicity condition given by~\citet{wang2017improving} could be stated as follows in the context of OCIM:
 for any action $S = (S_A, S_B)$, for any two expectation vectors $\bm{\mu} = (\mu_1,\dots,\mu_m)$ and $\bm{\mu}' = (\mu_1',\dots,\mu_m')$, we have $r_S(\bm{\mu}) \leq r_S(\bm{\mu}')$ if $\mu_i \leq \mu_i'$ for all $i \in [m]$.
%\begin{restatable}{condition}{condOverMono}(Monotonicity).
%We say that an OCIM problem instance satisfies monotonicity, if for any joint action $S = \{S_A, S_B\}$, for any two expectation vectors $\bm{\mu} = (\mu_1,\dots,\mu_m)$ and $\bm{\mu}' = (\mu_1',\dots,\mu_m')$, we have $r_S(\bm{\mu}) \leq r_S(\bm{\mu}')$ if $\mu_i \leq \mu_i'$ for all $i \in [m]$. 
%\end{restatable}
Figure~\ref{fig:non_monetone} shows a simple example of OCIM that does not satisfy the monotonicity condition. The left and right nodes are the seed nodes of $A$ and $B$; the numbers below edges are influence probabilities. It is easy to calculate that $r_S(\bm{\mu}) = \mu_1(1-\mu_2) + 2$, for both dominance and proportional tie-breaking rules. Thus, if we increase $\mu_2$, $r_S(\bm{\mu})$ will decrease, which is contrary to monotonicity.
In general, for every edge $(u,v)$, depending on the positions of the $A$- and $B$-seeds, increasing the influence probability of $(u,v)$ may benefit the propagation of $A$ or may benefit the propagation of $B$ and thus impair the propagation of $A$. Thus, the influence spread of $A$ has intricate connections with the influence probabilities on the edges.
\begin{figure}[H]
    \centering
    \includegraphics[width=0.5\columnwidth]{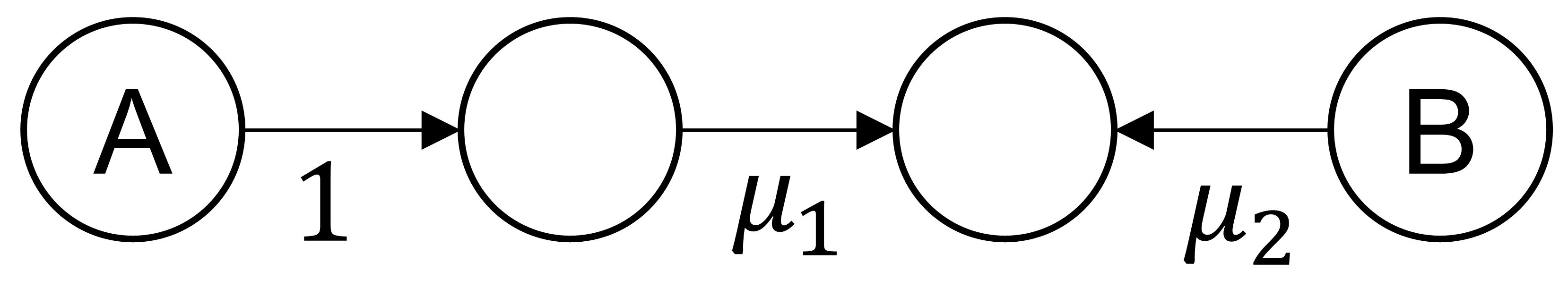}
    \caption{Example of non-monotonicity in OCIM}
    \label{fig:non_monetone}
\end{figure}
% \vspace{-5mm}
The lack of monotonicity poses a significant challenge to the OCIM problem. We cannot directly use UCB-type algorithms \citep{chen2016combinatorial}, as they will not provide optimistic solutions to bound the regret. 

\subsection{Triggering Probability Modulated (TPM) Bounded Smoothness}\label{sec:TPM}
The lack of monotonicity further complicates the analysis of the Triggering Probability Modulated (TPM) condition \citep{wang2017improving}, 
which is crucial in establishing regret bounds for CMAB algorithms.
We use $p_i^{S}(\bm{\mu})$ to denote the probability that the action $S$ triggers arm $i$ when the expectation vector is $\bm{\mu}$. 
The TPM condition in OCIM is given below.
% \begin{condition}
% We say that an OCIM problem instance satisfies 1-norm TPM bounded smoothness, if there exists $B\in R^+$ (referred as the bounded smoothness constant) such that, for any two environment distributions with expectation vectors $\bm{\mu}$ and $\bm{\mu}'$, and any joint action $S$, we have $|r_{S}(\bm{\mu}) - r_{S}(\bm{\mu}')| \leq B \sum_{i\in[m]}p_i^{\bm{\mu},S}|\mu_i-\mu_i'|$,
% \end{condition}
\begin{restatable}{condition}{condOverOneNorm}(1-Norm TPM bounded smoothness). \label{cond:TPM}
We say that an OCIM problem instance satisfies 1-norm TPM bounded smoothness, if there exists $C\in \mathbb{R}^+$ (referred to as the bounded smoothness coefficient) such that, for any two expectation vectors $\bm{\mu}$ and $\bm{\mu}'$, and any action $S = (S_A, S_B)$, we have $|r_{S}(\bm{\mu}) - r_{S}(\bm{\mu}')| \leq C \sum_{i\in[m]}p_i^{S}(\bm{\mu})|\mu_i-\mu_i'|$.
\end{restatable}
Fortunately, with a more intricate analysis, we are able to show the following TPM condition.
%\wei{What is the bounded smoothness constant here?}\jinhang{It is similar to that in \cite{wang2017improving}, which is the maximum number of nodes that can be triggered by action $S$, and it's no larger than $n = |V|$. Do we need to mention it here?}
%\wei{Need to add the bounded smoothness constant in the following result.}
\begin{restatable}{theorem}{thmOverTPM}\label{thm:TPM}
Under both dominance and proportional tie-breaking rules,
OCIM instances satisfy the 1-norm TPM bounded smoothness condition with coefficient $C=\tilde{C}$, where $\tilde{C}$ is the maximum number of nodes that any one node can reach in graph $G$.
\end{restatable}
The proof of the above theorem is one of the key technical contributions of the paper.
In the non-competitive setting, an edge coupling method could give a relatively simple proof for the TPM condition.\footnote{The original proof in ~\citep{wang2017improving} occupies several pages, but \cite{Li2020OIMLT} (in their Appendix E) provide a much shorter proof based on edge coupling.}
The idea of edge coupling is that for every edge $e\in E$, we sample a real number $X_e \in [0,1]$ uniformly at random, and determine $e$ to be live under $\bm{\mu}$ if $X_e\le \mu_e$ and blocked if $X_e>\mu_e$, and similarly for $\bm{\mu}'$.
This couples the live-edge graphs $L$ and $L'$ under $\bm{\mu}$ and $\bm{\mu}'$ respectively.
In the non-competitive setting, due to the monotonicity property, we only need to consider the TPM condition when $\bm{\mu} \ge \bm{\mu}'$ (coordinate-wise), and this implies that $L'$ is a subgraph of $L$, which significantly simplifies the analysis.
However, in the competitive setting, monotonicity does not hold, and we have to show the TPM condition for every pair of $\bm{\mu}$ and $\bm{\mu}'$.
	Thus, $L$ and $L'$ no longer have the subgraph relationship.
In this case, we have to show that for every coupling $L$ and $L'$, for every $v\in V$ that is activated by $A$ in $L$ but not activated by $A$ in $L'$,
	it is because either
	(a) some edge $e=(u,w)$ is live in $L$ but blocked in $L'$ while $u$ is $A$-activated (or equivalently $e$ is $A$-triggered); or
	(b) some edge $e$ is live in $L'$ but blocked in $L$ while $e$ is $B$-triggered.
The case (b) is due to the possibility of $B$ blocking $A$'s propagation, a unique scenario in OCIM.
The above claim needs nontrivial inductive proofs for dominance and proportional tie-breaking rules, and then its correctness ensures the TPM condition.

\section{Bayesian Regret Approach}\label{sec:TS}
% \subsection{OCIM-TS Algorithm} \label{sec:OCIM-TS}
In our OCIM model, since the samples of base arms follow Bernoulli distributions with mean vector $\bm{\mu}$, we can assume the prior distributions of $\bm{\mu}$, $\mathcal{Q}$, are Beta distributions, where $\mu_i \sim Beta(a_i, b_i)$ for all arm $i$. Given the prior distributions of all arms, we propose an Online Competitive Influence Maximization-Thompson Sampling (OCIM-TS) algorithm, which is described in Algorithm \ref{alg:OCIM-TS}. We initialize the prior distribution of each arm $i$ to $Beta(a_i, b_i)$. Then we take the context $S_B^{(t)}$ and the sampled $\bm{\mu}^{(t)}$  from prior distributions as inputs to the oracle $\mathcal{O}$, and get an output action $S^{(t)}$. After taking this action, we get feedback $X_i^{(t)}$'s from all triggered arms $i \in \tau$, then use them to update the prior distributions of all triggered base arms in $\tau$.
\begin{algorithm}[t]
 \caption{OCIM-TS with offline oracle $\mathcal{O}$}\label{alg:OCIM-TS}
 \begin{algorithmic}[1]
 \STATE \textbf{Input}: $m$, $\mathcal{O}$, Prior $\mathcal{Q} = \prod_{i\in[m]} Beta(a_i, b_i)$.
%  \STATE For all arm $i\in [m]$, $\mu_i \sim Beta(a_i, b_i)$.
 \FOR{$t = 1,2,3,\dots$}
    \STATE For each arm $i\in[m]$, draw a sample $\mu_i^{(t)}$ from $Beta(a_i, b_i)$; let $\bm{\mu}^{(t)} = (\mu_1^{(t)},\cdots, \mu_m^{(t)})$.
    \STATE Obtain context $S_B^{(t)}$.
    \STATE $S^{(t)} \leftarrow {\mathcal{O}}(S_B^{(t)}, \bm{\mu}^{(t)})$.
    \STATE Play action $S^{(t)}$, which triggers a set $\tau \subseteq [m]$ of base arms with feedback $X_i^{(t)}$'s, $i\in \tau$.
    \FORALL {$i\in \tau$}
    \STATE $a_i \leftarrow a_i+X_i^{(t)}; b_i \leftarrow b_i + 1 - X_i^{(t)}$.
    \ENDFOR
 \ENDFOR
 \end{algorithmic} 
\end{algorithm}
% Recall that $\bm{\mathcal{S}}^{(t)}$ is the action space in round $t$. 
% We define the reward gap $\Delta_{S}^{(t)}{=}\max(0, \alpha \cdot \text{opt}^{(t)}(\bm{\mu}) - r_S(\bm{\mu}))$ for all actions $S \in \bm{\mathcal{S}}^{(t)}$. 
% For each arm $i$, we define $\Delta^{i,T}_{\max} = \max_{t\in[T]}\sup_{S\in\bm{\mathcal{S}}^{(t)}:p_i^S(\bm{\mu}) > 0, \Delta_{S}^{(t)} > 0} \Delta_{S}^{(t)}$. 
% If there is no action $S$ such that $p_i^S(\bm{\mu}) > 0$ and $\Delta_{S}^{(t)} > 0$, we define $\Delta^{i,T}_{\max} = 0$. 
% We define $\Delta^{(T)}_{\max} = \max_{i \in [m]}\Delta^{i,T}_{\max}$ and $\delta^{(T)}_{\max} = \max_{\bm{\mu}} \Delta^{(T)}_{\max}$. 
Let $\widetilde{S} = \{i\in [m] \mid p_i^S(\bm{\mu}) > 0\}$ be the set of arms that can be triggered by $S$. We define $K = \max_{S\in\bm{\mathcal{S}}^{(t)}}|\widetilde{S}|$ as the largest number of arms that could be triggered by a feasible action. 
% We use $\lceil x \rceil_0$ to denote $\max\{\lceil x \rceil, 0\}$.
We provide the Bayesian regret bound of OCIM-TS.
\begin{restatable}{theorem}{thmOverTS}\label{thm:TS}
The OCIM-TS algorithm has the following Bayesian regret bound with $\tilde{C}$ as defined in Theorem~\ref{thm:TPM}:
\begin{align}\textstyle
    % \text{BayesReg}_{\alpha,\beta}(T) \leq 12\tilde{C}\sqrt{mKT\ln T}+ 2\tilde{C}m +\left(\left\lceil \log_2 \frac{T}{18\ln T}\right\rceil_0 + 4\right)\cdot \frac{\pi^2}{6}\cdot n \cdot m.
    \text{BayesReg}_{\alpha,\beta}(T) \leq O(\tilde{C}\sqrt{mKT\ln T}).
\end{align}
% \xutong{I am not sure whether this should be \begin{equation}
%     \text{BayesReg}_{\alpha,\beta}(T) \leq 12\tilde{C}\sqrt{mKT\ln T}+ 2\tilde{C}m +\left(\left\lceil \log_2 \frac{T}{18\ln T}\right\rceil_0 + 4\right)\cdot m \cdot \frac{\pi^2}{6}\cdot\delta^{(T)}_{\max},
% \end{equation}}
\end{restatable}
This regret bound essentially matches the distribution-independent frequentist regret bound of OCIM-OFU in the next section. 
The proof of the above theorem is inspired by the posterior sampling regret decomposition of \citet{russo2014learning}. However, we combine the key property of posterior sampling with the TPM condition in Theorem \ref{thm:TPM} to tackle non-monotonicity.
OCIM-TS can also be applied to general C$^2$MAB-T problems and allows any benchmark offline oracles (e.g., approximate or heuristic oracles). We provide the Bayesian regret bound of OCIM-TS on general C$^2$MAB-T problems in the Appendix.

\section{Frequentist Regret Approach}\label{sec:FreqReg}
Although OCIM-TS can solve the OCIM problem with a standard offline oracle (e.g., TCIM in \cite{lin2015analyzing}), it requires the prior distribution of the network parameter $\bm{\mu}$, which might not be available in practice.
In this section, we first propose the OCIM-OFU algorithm. It achieves good frequentist regret without the prior knowledge, but requires a new oracle to solve a harder offline problem. We then design the OCIM-ETC algorithm, which requires less feedback and easier offline computation, but yields a worse frequentist regret bound.

\subsection{OCIM-OFU Algorithm}\label{sec:OCIM-OFU}
As discussed in Section \ref{sec:non-monotone}, due to the lack of monotonicity, we cannot directly use UCB-type algorithms. However, it is still possible to design bandit algorithms following the principle of Optimism in the Face of Uncertainty (OFU).
We first introduce a new offline problem that jointly optimizes for both the seed set $S^*$ and the optimal influence probability vector $\bm{\mu}^*$, where each dimension of $\bm{\mu}^*$, ${\mu}^{*}_i$, is searched within a confidence interval $c_i$, for all $i\in E$.
% \begin{equation} \label{eq:newopt}
% \underset{S,\,\bm{\mu}}{\text{max}}\;r_S(\bm{\mu}),\;\text{s.t.}\;|S_A|\leq k,\;S = (S_A, S_B),\;\mu_i \in c_i\;(i = 1, \ldots, m).
% \end{equation}
\begin{equation} \label{eq:newopt}\textstyle
\begin{aligned}
& \underset{S,\,\bm{\mu}}{\text{maximize}}
& & r_S(\bm{\mu}) \\
& \text{subject to}
& & |S_A|\leq k, S = (S_A, S_B)\\
& & & \mu_i \in c_i, \; i = 1, \ldots, m.
\end{aligned}
\end{equation}
We then define a new offline $(\alpha,\beta)$-approximation oracle $\widetilde{\mathcal{O}}$ to solve this problem. 
Oracle $\widetilde{\mathcal{O}}$ takes $S_B$ and $c_i$'s as inputs and outputs $\bm{\mu}^{\widetilde{\mathcal{O}}}$ and action $S^{\widetilde{\mathcal{O}}} = (S_A^{\widetilde{\mathcal{O}}}, S_B)$, such that $\text{Pr}\{r_{S^{\widetilde{\mathcal{O}}}}(\bm{\mu}^{\widetilde{\mathcal{O}}}) \geq \alpha \cdot r_{S^*}(\bm{\mu}^*) \} \geq \beta$, 
where $(S^*,\bm{\mu}^*) $ is the optimal solution for Eq.\eqref{eq:newopt}. 
%	$\alpha$ is the approximation ratio and $\beta$ is the success probability. 
%Although $\mathcal{O}$ can use greedy or RR set algorithms \cite{lin2015analyzing} to solve its offline problem, it is not straightforward to see how to find the approximate solution for $\widetilde{\mathcal{O}}$, since it needs to optimize over $S$ and $\bm{\mu}$ at the same time. 
% We defer the design of such an oracle to Sec.~\ref{sec:offline}.

With the offline oracle $\widetilde{\mathcal{O}}$, we propose an algorithm following the principle of Optimism in the Face of Uncertainty (OFU), 
named OCIM-OFU. The algorithm maintains the empirical mean $\hat{\mu}_i$ and confidence radius $\rho_i$ for each edge probability. It uses the lower and upper confidence bounds to determine the range of $\mu_i$: $c_i = \left[(\hat{\mu}_i - \rho_i)^{0+}, (\hat{\mu}_i + \rho_i)^{1-}\right]$, where we use $(x)^{0+}$ and $(x)^{1-}$ to denote $\max\{x, 0\}$ and $\min\{x, 1\}$ for any real number $x$. 
%\wei{The notations $\lfloor x \rfloor_0$ and $\lceil x \rceil^1$ could be a bit confusing since they do not take floor or ceiling operations. But we have another notation below that does take the ceiling operation.
%	I have not thought about a better notation yet. In some cases, $x\vee 0$ and $x\wedge 1$ are used.
%}\jinhang{I change them to $(x)^{0+}$ and $(x)^{1-}$. Usually, $(x)^{+}$ represents a ReLU function $\max(x, 0)$}
%\xutong{How about change to $x^+=\max(x,0)$ and $x^-=\min(x,1)$.}
It feeds $S_B^{(t)}$ and all current $c_i$'s into the offline oracle $\widetilde{\mathcal{O}}$ to obtain the action $S^{(t)} = (S_A^{(t)}, S_B^{(t)})$ to play at round $t$. The confidence radius $\rho_i$ is large if arm $i$ is not triggered often, which leads to a wider search space $c_i$ to find the optimistic estimate of $\mu_i$. 
% With the same definitions in Section \ref{sec:TS}, for each arm $i$, we define $\Delta^{i,T}_{\min} = \min_{t\in[T]}\inf_{S\in\mathcal{S}^{(t)}:p_i^S(\bm{\mu}) > 0, \Delta_{S}^{(t)} > 0} \Delta_{S}^{(t)}$.
% If there is no action $S$ such that $p_i^S(\bm{\mu}) > 0$ and $\Delta_{S}^{(t)} > 0$, we define $\Delta^{i,T}_{\min} = +\infty$.
% We define $\Delta^{(T)}_{\min} = \min_{i \in [m]}\Delta^{i,T}_{\min}$.
We provide its frequentist regret bound.
\begin{algorithm}[t]
 \caption{OCIM-OFU with offline oracle $\widetilde{\mathcal{O}}$}\label{alg:CIM-OIFU}
 \begin{algorithmic}[1]
 \STATE \textbf{Input}: $m$, Oracle $\widetilde{\mathcal{O}}$.
 \STATE For each arm $i\in [m]$, $T_i\leftarrow 0$. \{maintain the total number of times arm $i$ is played so far.\}
 \STATE For each arm $i \in [m]$, $\hat{\mu}_i \leftarrow 1$. \{maintain the empirical mean of $X_i$.\}
 \FOR{$t = 1,2,3,\dots$}
    \STATE For each arm $i\in[m], \rho_i \leftarrow \sqrt{\frac{3\ln t}{2 T_i}}$. \{the confidence radius, $\rho_i = +\infty$ if $T_i = 0$.\}
    \STATE For each arm $i\in[m], c_i \leftarrow \left[(\hat{\mu}_i - \rho_i)^{0+}, (\hat{\mu}_i + \rho_i)^{1-}\right]$. \{the estimated range of $\mu_i$.\}
    \STATE Obtain context $S_B^{(t)}$.
    \STATE $S^{(t)} \leftarrow \widetilde{\mathcal{O}}(S_B^{(t)}, c_1, c_2, \dots, c_m)$.
    \STATE Play action $S^{(t)}$, which triggers a set $\tau \subseteq [m]$ of base arms with feedback $X_i^{(t)}$'s, $i\in \tau$.
    \STATE For every $i\in \tau$ update $T_i$ and $\hat{\mu}_i$: $T_i = T_i + 1, \hat{\mu}_i = \hat{\mu}_i + (X_i^{(t)}-\hat{\mu}_i) / T_i$.
 \ENDFOR
 \end{algorithmic} 
\end{algorithm}

\begin{restatable}{theorem}{thmOverOIFU}\label{thm:OIFU}
The OCIM-OFU algorithm has the following distribution-independent bound (see the Appendix for the distribution-dependent bound) with $\tilde{C}$ defined in Theorem~\ref{thm:TPM} , 
% (1) if $\Delta^{(T)}_{\min} > 0$, we have a distribution-dependent bound
% \begin{align}\textstyle
%     \text{Reg}_{\alpha,\beta}(T; \bm{\mu}) \leq \sum_{i \in [m]} \frac{576\tilde{C}^2K\ln T}{\Delta^{i,T}_{\min}}  + 4\tilde{C}m +\sum_{i \in [m]} \left(\left\lceil \log_2 \frac{2\tilde{C}K}{\Delta^{i,T}_{\min}}\right\rceil_0 + 2\right)\cdot \frac{\pi^2}{6}\cdot\Delta^{(T)}_{\max},
% \end{align}
% and (2) we have a distribution-independent bound
\begin{align}\textstyle
\nonumber
    % \text{Reg}_{\alpha,\beta}(T; \bm{\mu}) \leq 12\tilde{C}\sqrt{mKT\ln T}+ 2\tilde{C}m + \left(\left\lceil \log_2 \frac{T}{18\ln T}\right\rceil_0 + 2\right)\cdot \frac{\pi^2}{6}\cdot n \cdot m.
    \text{Reg}_{\alpha,\beta}(T; \bm{\mu}) \leq O(\tilde{C}\sqrt{mKT\ln T})
\end{align}\end{restatable}

The above regret bound has the typical form of $\sqrt{T\ln T}$, indicating that it is tight on the important time horizon $T$.
In fact, it has the same order as in \citet{wang2017improving}'s for the CMAB problem under monotonicity, despite the fact that the OCIM problem does not enjoy monotonicity, and matches the lower bound of CMAB with general reward functions in \citep{merlis2020tight}.
%This indicates that in the competitive setting, the online learning part is not affected by the non-monotonicity of the model.
This result is due to our non-trivial TPM condition analysis (Theorem~\ref{thm:TPM}) that shows the same condition as in \citet{wang2017improving}'s setting with monotonicity.
%without the monotonicity in the OCIM setting.

{\bf Computational Efficiency}.
We now discuss the computational complexity of implementing the OCIM-OFU algorithm. We show the complexity of the new offline optimization problem in Eq. \eqref{eq:newopt}.
\begin{restatable}{theorem}{thmOverOffline}\label{thm:sharpP}
The offline problem in Eq.\eqref{eq:newopt} is \#P-hard.
\end{restatable}

As mentioned before, the original offline problem, i.e., maximizing $r_S(\bm{\mu})$ over $S$ when fixing $\bm{\mu}$, can be solved by several algorithms~\citep{lin2015analyzing} based on submodularity of $r_S(\bm{\mu})$ over $S$. A straightforward attempt on the new offline problem in Eq.\eqref{eq:newopt} is to show the submodularity of $g(S) = \max_{\bm{\mu}} r_S(\bm{\mu})$ over $S$, and then to use a greedy algorithm on $g$ to select $S$. Unfortunately, we find that $g(S)$ is not submodular (see the Appendix for a counterexample). Implementing the oracle $\widetilde{\mathcal{O}}$ is then a challenge. However, it is possible to design efficient approximate oracles for bipartite graphs, which model the competitive probabilistic maximum coverage problem with applications in online advertising~\citep{chen2016combinatorial}. The main idea is that we can pre-determine that either the lower or the upper bound of $c_i$ is optimal and should be chosen as $\mu^*_i$ depending on the tie-breaking rule, then use existing efficient influence maximization algorithms to get approximate solutions. The competitive propagation in the general graph is much more complicated, but we have a key observation that the optimal solution for the optimization problem in Eq.\eqref{eq:newopt} must occur at the boundaries of the intervals $c_i$. Based on that, we discuss solutions for some specific graphs such as trees. See the Appendix for more details.

\subsection{OCIM-ETC Algorithm} \label{sec:ETC}
In this section, we propose an OCIM Explore-Then-Commit (OCIM-ETC) algorithm. It has two advantages: first, it does not need the new offline oracle discussed in Sec.~\ref{sec:OCIM-OFU}; and second, it requires fewer observations than our other algorithms: instead of the observations of all triggered edges, i.e., $\tau$, it only needs the observations of all direct out-edges of seed nodes.

% {\color{blue}
Like other ETC algorithms~\citep{garivier2016explore}, OCIM-ETC divides the $T$ rounds into two phases: an exploration phase and an exploitation phase. In the exploration phase, it chooses each node as the seed node of $A$ for $N$ times. The exploration phase thus takes $\lceil nN/k \rceil$ rounds. In the exploitation phase, it takes $S_B^{(t)}$ and the empirical means $\hat{\mu}_i$ as inputs to the oracle $\mathcal{O}$ mentioned in Sec.~\ref{sec:model}, then plays the output action $S^{\mathcal{O},(t)}$. We give its frequentist regret bound.
\begin{restatable}{theorem}{thmOverETC}\label{thm:ETC}
The OCIM-ETC algorithm has the following distribution-independent regret bound (see the Appendix for the distribution-dependent bound) with $\tilde{C}$ defined in Theorem~\ref{thm:TPM}, 
% (1) if $\Delta^{(T)}_{\min} > 0$, when $N = \max\left\{1, \frac{2 \tilde{C}^2 m^2}{(\Delta^{(T)}_{\min})^2}\ln(\frac{kT(\Delta^{(T)}_{\min})^2}{ \tilde{C}^3 m})\right\}$, we have a distribution-dependent bound
% \begin{align}\textstyle
%     \text{Reg}_{\alpha,\beta}(T; \bm{\mu}) \leq \frac{n}{k} \Delta^{(T)}_{\max} + \frac{2\tilde{C}^2 m^2 n \Delta^{(T)}_{\max}}{k (\Delta^{(T)}_{\min})^2}\left(\max\left\{ \ln\left(\frac{kT(\Delta^{(T)}_{\min})^2}{\tilde{C}^2 mn}\right), 0\right\} + 1\right)
% \end{align}
when $N = (\tilde{C}mk)^{\frac{2}{3}} n^{-\frac{4}{3}} T^{\frac{2}{3}} (\ln T)^{\frac{1}{3}}$,
% \begin{equation}
%     \text{Reg}_{\bm{\mu},\alpha,\beta}(T) \leq 3.9 B^{\frac{2}{3}} m^{\frac{4}{3}} n^{\frac{2}{3}} k^{-\frac{1}{3}} T^{\frac{2}{3}} + 1
% \end{equation}
\begin{equation}\textstyle
     Reg_{\alpha,\beta}(T;\bm{\mu}) \leq O((\tilde{C}mn)^{\frac{2}{3}} k^{-\frac{1}{3}} T^{\frac{2}{3}} (\ln T )^{\frac{1}{3}}).
\end{equation}
\end{restatable}
Although this regret bound is worse than that of the OCIM-OFU algorithm in Theorem \ref{thm:OIFU}, OCIM-ETC requires easier offline computation and less feedback since it only needs to observe the results of direct out-edges of seed nodes, which shows the tradeoff between regret bound and feedback/computation in OCIM. 
% \carlee{If there's room, can we add a line explaining why this is the case? It's not at all obvious from the algorithm explanation.} 
% and easier offline computation, so it shows the tradeoff between regret bound and feedback/computation in OCIM. 

%
%	
%	
%idea is to check whether $g(S) = \max_{\bm{\mu}} r_S(\bm{\mu})$ also satisfies submodularity over $S$: if it is true and we know to solve another offline problem of maximizing $r_S(\bm{\mu})$ over $\bm{\mu}$ for any given $S$, we can use a greedy-based algorithm to find an approximation solution. Unfortunately, we show that $g(S)$ is not submodular over $S$ by giving a counterexample, and the problem of maximizing $r_S(\bm{\mu})$ over $\bm{\mu}$ for any given $S$ is \#P-hard via a reduction from an influence computation problem \cite{chen2013information}; complete proofs are in the supplementary material. Both of these findings show the challenges of solving this new offline optimization problem.
%\jinhang{I kept the idea we tried but fail here, feel free to modify or remove it} \carlee{I think it's good to mention this, or you could just say that you show how straightforward methods won't work in the supplementary material without giving details.} 
% {\color{blue}
\section{Extension to Probabilistic Seed Distribution for the Competitor}\label{sec:extension}
\cite{lin2015analyzing} extend the offline CIM problem to a probabilistic setting where the competitor's seed distribution is known (i.e., the probability of each node being selected as a seed by the competitor). In this section, we extend our algorithms to handle two new settings where the competitor has a probabilistic seed distribution.
Note that we need to slightly modify the TPM condition for these settings.
We denote the expected reward of follower $A$ as $r(S_A, D_B, \bm{\mu})$, where $S_A$ is the seed set of $A$, $D_B$ is the seed distribution of $B$. We use $p_i(S_A, D_B, \bm{\mu})$ to denote the probability that either $S_A$ or $S_B$ will trigger arm $i$ when the seed set of $A$ is $S_A$, the seed set of $B$, $S_B$, is sampled from $D_B$, and the expectation vector is $\bm{\mu}$. The modified TPM condition is given below. 
% \carlee{must this hold for each $D_B^{(t)}$ in the dynamic setting?}
% \jinhang{Yes, but actually $2\tilde{C}$ works for any $D_B^{(t)}$.}
\begin{restatable}{condition}{condOverOneNormProb}(Modified TPM bounded smoothness). \label{cond:TPM_prob}
We say that an OCIM problem instance satisfies modified TPM bounded smoothness, if there exists $C\in \mathbb{R}^+$ such that, for any two expectation vectors $\bm{\mu}$ and $\bm{\mu}'$, and any seed set $S_A$ and seed distribution $D_B$, we have $|r(S_A, D_B, \bm{\mu}) - r(S_A, D_B, \bm{\mu}')| \leq C \sum_{i\in[m]}p_i(S_A, D_B, \bm{\mu})|\mu_i-\mu_i'|$.
\end{restatable}
With the similar analysis of Theorem \ref{thm:TPM}, we can show the following TPM condition when the competitor has probabilistic seed distribution.
\begin{restatable}{theorem}{thmOverTPMProb}\label{thm:TPM_prob}
Under both dominance and proportional tie-breaking rules,
OCIM instances satisfy the modified TPM bounded smoothness condition with coefficient $C=2\tilde{C}$, where $\tilde{C}$ is the maximum number of nodes that any one node can reach in graph $G$.
\end{restatable}

{\bf Known dynamic seed distribution}. In round $t$, the competitor's seed set $S_B^{(t)}$ follows a distribution $D_B^{(t)}$, i.e., $S_B^{(t)} \sim D_B^{(t)}$. However, the follower only knows $D_B^{(t)}$ but not $S_B^{(t)}$ before choosing $S_A^{(t)}$. 
Since our proposed framework has a nice separation between online learning and offline computation, in this setting, only the offline computation part will be affected. 
Specifically, we can replace the oracle $\mathcal{O}(S_B^{(t)}, \bm{\mu}^{(t)})$ in OCIM-TS and OCIM-ETC with a new oracle $\mathcal{O}_{\text{new}}(D_B^{(t)}, \bm{\mu}^{(t)})$. For OCIM-OFU, similar to oracle $\widetilde{\mathcal{O}}$, we need a new oracle $\widetilde{\mathcal{O}}_{\text{new}}$ that takes $D_B^{(t)}$ and the confidence intervals $\left\{c_i\right\}$ as inputs and outputs $S_A^{(t)}$. We can use the TCIM algorithm of~\citep{lin2015analyzing} to design $\mathcal{O}_{\text{new}}$ and $\widetilde{\mathcal{O}}_{\text{new}}$. Our proposed algorithms will have the same regret bounds as in Theorems~\ref{thm:TS} and \ref{thm:OIFU}.

{\bf Unknown fixed seed distribution}.
In this setting, the seed distribution of the competitor, $D_B$, is unknown to the follower but fixed for all rounds. To solve this problem, we introduce a virtual $B$ seed node $u_B$, which connects to each existing node $u$ with an unknown edge probability $p(u_B, u)$ equal to the probability of $u$ being selected as a $B$ seed. This reduces the case of probabilistic seed selection to the standard CIC model with a known seed node $u_B$. The unknown edge probabilities $p(u_B, u)$'s can be learned together with the edge probabilities in the original graph. Therefore, we do not need to know the competitor's seed selection in advance and can learn it over time through the online learning process. Our algorithms will have the same regret guarantees as in Theorems~\ref{thm:TS} and \ref{thm:OIFU}.
% as long as the seed selection is independent on the nodes and we can observe the realized $B$ seeds
% In each round, we first select $S_A^{(t)}$ based on our learning algorithm, then observe the cascade including the $B$ seed set realization $S_B \sim D_B$, and finally update our knowledge of the edge probabilities (including the added virtual edges). The TPM condition holds in this setting (with bounded  smoothness  coefficient $\tilde{C}_{\text{new}} = 2\tilde{C}$) since it is still an instance of the CIC model, so
% }
\section{Experiments}\label{sec:experiment}
% \carlee{there are several grammatical errors in this section. Please proofread it carefully before submission.}
% \xutong{revised.}
% \begin{table}[H]
% \centering
% \caption{Dataset Statistics}
% \label{tab:dataset}
% % \begin{minipage}{1.0\columnwidth}
% \begin{tabular}{cccc}
%     \toprule
%       Network & $n$ & $m$ & Average Degree \\
%      %Layer& \# of vertices & \# of edges  \\
%      \midrule
%      BI-PLW & $1,300$ & $2,990$ & $2.24$ \\
%      Yahoo-Ad & $11,475$ & $52,567$ & $4.58$\\
%      DM & $679$ & $3,374$ & $4.96$ \\
%      \bottomrule
% \end{tabular}
% % \end{minipage}
% \end{table}
\textbf{Datasets and settings.}
To validate our theoretical findings, we conduct experiments on two real-world datasets widely used in the influence maximization literature, with detailed statistics summarized in Table~\ref{tab:dataset}.
% For the bipartite graph, we first generate a synthetic graph with $300$ source nodes $S$ and $1,000$ target nodes $T$, with around $3,000$ edges between them.
% This graph, denoted as BI-PLW, is generated so that the degree distribution of the source nodes follows the \textit{power-law} distribution with $\gamma=2.5$.
% After assigning the degrees to each source node $s \in S$, it is connected to $deg_{+}(s)$ nodes chosen uniformly at random from $T$, where $deg_{+}(s)$ is the out-degree of node $s$.
First, we use the Yahoo! Search Marketing Advertiser Bidding Data\footnote{https://webscope.sandbox.yahoo.com} (denoted as Yahoo-Ad), which contains a bipartite graph between $1,000$ keywords and $10,475$ advertisers.
% Each edge represents a bid on a keyword from an advertiser, and our goal is to select a set of keywords that attracts the most advertisers.
Every entry in the original Yahoo-Ad dataset is a 4-tuple, which represents a “keyword-id” bid by “advertiser-id” at “time-stamp” with “price”. We extract advertiser-ids and keyword-ids as nodes, and add an edge if the advertiser bids the keyword at least once.
Each edge shows the "who is interested in what" relationship. This dataset will contain $11,475$ nodes and $52,567$ edges.
The motivation of this experiment is to select a set of keywords that is maximally associated to advertisers, which is useful for the publisher to promote keywords to advertisers.
% Note that this experiment is motivated by answering the question on how to select a set of "what"s that maximally influence "who"s in the online competitive environment, which is useful for the publisher to promote keywords to advertisers.
% a network of data mining researchers extracted from the ArnetMiner archive \cite{tang2009social},
We then consider the DM network~\citep{tang2009social} with 679 nodes representing researchers and $3,374$ edges representing collaborations between them. We simulate a researcher asking others (i.e., $S_A$) to spread her ideas while her competitor (i.e., $S_B$) promotes a competing proposal.
We set the parameters of our experiments as the following.
For the edge weights, Yahoo-Ad uses the weighted cascade method~\citep{kempe2003maximizing}, i.e. $p(s,t)=1/deg_{-}(s)$, where $deg_{-}(s)$ is the in-degree of node $s$, and weights for DM are obtained by the learned edge parameters from ~\citep{tang2009social}.
For Bayesian regrets, we set a prior distribution of $\mu_e \sim Beta(5w_e, 5(1-w_e))$, where $w_e$ is the true edge weight as specified above.

We model non-strategic and strategic competitors by selecting the seed set $S_B$ uniformly at random (denoted as RD) or by running the non-competitive influence maximization algorithm (denoted as IM).
% Here we focus on the $B>A$ tie-breaking rule.
% We repeat each experiment 50 times and show the average regret with $95\%$ confidence interval.
% We provide details of the datasets and other experiment parameters, and results with $A > B$ tie-breaking, in the supplementary material.
 In our experiments, we set $|S_A|=|S_B|=5$ for Yahoo-Ad and $|S_A|=|S_B|=10$ for the DM dataset, and $B>A$.
% Note that we only show results for $B > A$, because when $A>B$ we can simply ignore $S_B$ for bipartite graphs as mentioned in Sec.~\ref{sec:offline_bipartite} and the reason for general graphs will be explained later.
Since the optimal solution given the true edge probabilities cannot be derived in polynomial time, for Yahoo-Ad, we use the greedy solution as the optimal baseline, which is a $(1-1/e, 1)$-approximate solution. For the DM dataset, we use the IMM solution as the optimal baseline, which is a $(1-1/e-\epsilon, 1-n^{-l})$-approximate solution.
% The running times of different algorithms can be found in the Appendix.
% As mentioned in Sec.~\ref{secsec:offline}, we use greedy and IMM algorithms as the offline oracle to . 
For frequentist regrets, we repeat each experiment 50 times and show the average regret with $95\%$ confidence interval.
For Bayesian regrets, we draw 5 problem instances according to the prior distributions, conduct 10 experiments in each instance and report the average Bayesian regret over the 50 experiments.
Due to the space constraint, results of other settings are provided in the Appendix.

\begin{table}[t]
\centering
\caption{Dataset Statistics}
\label{tab:dataset}
% \begin{minipage}{1.0\columnwidth}
\begin{tabular}{cccc}
    \toprule
      Network & $n$ & $m$ & Average Degree \\
     %Layer& \# of vertices & \# of edges  \\
     \midrule
    %  BI-PLW & $1,300$ & $2,990$ & $2.24$ \\
    DM & $679$ & $3,374$ & $4.96$ \\
     Yahoo-Ad & $11,475$ & $52,567$ & $4.58$\\
     \bottomrule
\end{tabular}
% \end{minipage}
\end{table}

\begin{table}[t]
\centering
\caption{Average Running Time (second/round)}
\label{tab:runningtime}
\resizebox{1\columnwidth}{!}{
\begin{tabular}{cccccc}
    \toprule
      Dataset & OCIM-OFU & OCIM-TS &OCIM-ETC & $\epsilon$-greedy & EMP\\
     %Layer& \# of vertices & \# of edges  \\
     \midrule
    Yahoo-Ad & $1.221$ & $1.641$ & $0.729$& $1.244$ & $1.226$\\
    DM & $1.142$ & $1.195$ & $0.621$ &$1.173$ & $1.125$\\
     \bottomrule
\end{tabular}
}
\end{table}

\textbf{Algorithms for comparison.} % We present the results for Alg.~\ref{alg:CIM-OIFU} with some level of modification in the implementation.
% We use the approximation algorithm from Sec.~\ref{sec:offline_bipartite} and the heuristic from Sec.~\ref{sec:generalgraph} as OCIM-OFU's offline oracle for Yahoo-Ad and DM respectively.
%We use the offline algorithm from Sec.~\ref{sec:offline_bipartite} as OCIM-OFU's offline oracle for Yahoo-Ad, and use the heuristic algorithm from Sec.~\ref{sec:generalgraph} for DM.
For OCIM-TS, since the true prior distribution is unknown for the frequentist setting, we use the uninformative prior $Beta(1,1)$ for each $\mu_e$.
For OCIM-OFU, we shrink its confidence interval by $\alpha_{\rho}$, i.e., $\rho_i \leftarrow \alpha_{\rho} \sqrt{{3\ln t}/{2 T_i}}$, to speed up the learning. The role of $\alpha_\rho$ represents a tradeoff between theoretical guarantees and real-world performance. $\alpha_\rho \ge 1$ provides theoretical regret bounds for the worst-case (i.e., our algorithms have sublinear regret for any problem instance) and most of the bandit literature gives regret analysis under this condition. However, in practice, we often do not face the worst problem instance. Taking a more aggressive $\alpha_\rho$ helps speed up the learning empirically \citep{liu2021multi}, though the algorithms may incur linear regrets for bad problem instances (which are likely rare in practice), preventing us from achieving worst-case theoretical regret bounds.
% though our theoretical regret bound requires $\alpha_{\rho} = 1$.
% Second, instead of enumerating for $O(2^m)$ for the general graph (i.e., DM network), we replace the enumeration in the offline oracle with the following heuristic, which further accelerates our algorithm.
% Concretely, we assign lower bounds for the direct out-edges of $S_B$ and all other edges to be the upper bounds, since $S_A$ can be selected from any node.
% Intuitively, this heuristic is to decrease the $S_B$'s influence spread and increase $S_A$'s to achieve a \textit{optimistic} estimation.
% For the selection of the seed sets, we use the IMM algorithm~\cite{tang2014influence}.
We compare OCIM-OFU/OCIM-TS to the $\epsilon$-Greedy algorithm with parameter $\epsilon = 0$ (denoted as the EMP algorithm) and $\epsilon = 0.01$, which inputs the empirical mean into the offline oracle with $1-\epsilon$ probability and otherwise selects $S_A$ uniformly at random.
%EMP algorithm, which always inputs the empirical mean into the offline oracle, and the $\epsilon$-Greedy algorithm with parameter $\epsilon=0.01, 0.05$, which inputs the empirical mean into the offline oracle with $1-\epsilon$ probability and otherwise selects $S_A$ uniformly at random. 
%Since OCIM-ETC requires many more rounds than OCIM-OFU, we put its results into the supplementary material.
The results of OCIM-ETC are moved to the Appendix as it requires more rounds to learn than others.

\textbf{Running time.}
We show the average running times for different algorithms in Table~\ref{tab:runningtime}. For the Yahoo-Ad dataset, OCIM-ETC is the fastest one as it only needs to call the oracle for one time before the exploitation phase.
% Note that we record the running time of OCIM-OFU with exact oracles that are discussed in Section~\ref{sec:compute_bipart} for Yahoo-Ad graph and heuristic oracles in Section~\ref{sec:compute_general} for DM graph (rather than \#P hard exact oracle). 
The running time of OCIM-TS is slower than that of OCIM-OFU because it requires an extra sampling procedure to generate Thompson samples. For the DM dataset, all algorithms consume less time since the graph is smaller, but the relative order for different algorithms are preserved.

\textbf{Experimental result for frequentist regrets} 
Figures~\ref{fig:Yahoo-Ad-RD} and~\ref{fig:Yahoo-Ad-IM} show the results for Yahoo-Ad. 
First, the regret of OCIM-OFU grows sub-linearly with respect to round $T$ for all $\alpha_{\rho}$, consistent with Theorem~\ref{thm:OIFU}'s regret bound.
Second, we can observe that OCIM-OFU is superior to EMP and $\epsilon$-Greedy when $\alpha_{\rho}=0.05$.
When $\alpha_{\rho}=0.2$, OCIM-OFU may have larger regret due to too much exploration.
The OCIM-TS algorithm has larger slope in regrets compared to other algorithms.
We speculate that such large slope comes from the uninformative prior, which requires more rounds to compensate for the mismatch of the uninformative and the true priors.

The results on the DM dataset are shown in Figs.~\ref{fig:DM-RD} and~\ref{fig:DM-IM}.
Generally, they are consistent with those on the Yahoo-Ad dataset: OCIM-OFU also grows sub-linearly w.r.t round $T$.
When $\alpha_{\rho}=0.05$, OCIM-OFU has smaller regret than all baselines.
Moreover, the difference between OCIM-OFU and the baselines for the non-strategic competitor (RD) is more significant than that of the strategic competitor's (IM), because the non-strategic competitor is less ``dominant'' and OCIM-OFU can carefully trade off exploration and exploitation to maximize $A$'s influence.
OCIM-TS learns faster and achieves better performance in this dataset compared to that in the Yahoo-Ad dataset.
% Note that we use a heuristic to replace enumerating all the boundary value combinations of edges, which trades off the efficiency and the theoretical guarantee. 
% The results show that our heuristic is very effective and does not degrade OCIM-OFU's performance.
\begin{figure}[t]
	\centering
	\begin{subfigure}[b]{0.23\textwidth}
		\centering
		\includegraphics[width=\textwidth]{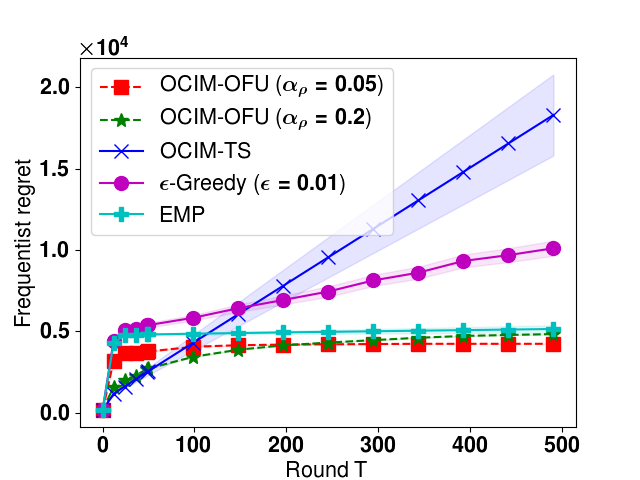}
		\caption{Yahoo-Ad, RD}
		\label{fig:Yahoo-Ad-RD}
	\end{subfigure}
% 	\hfill
	\begin{subfigure}[b]{0.23\textwidth}
		\centering
		\includegraphics[width=\textwidth]{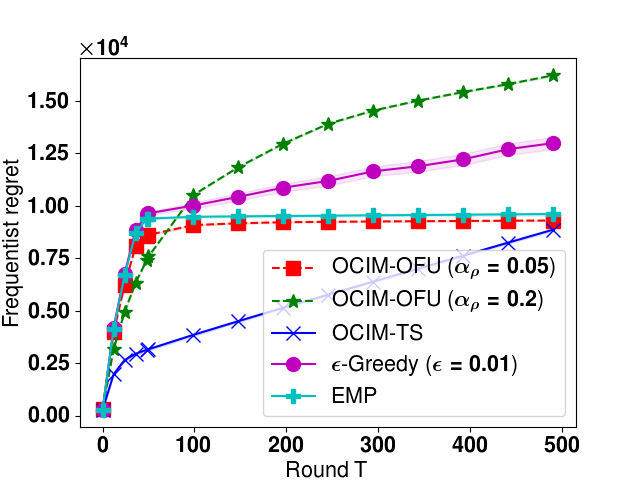}
		\caption{Yahoo-Ad, IM}
		\label{fig:Yahoo-Ad-IM}
	\end{subfigure}
	\begin{subfigure}[b]{0.23\textwidth}
		\centering
		\includegraphics[width=\textwidth]{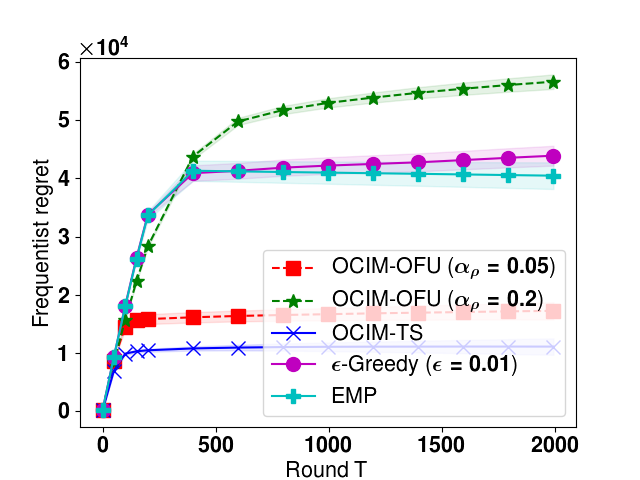}
		\caption{DM, RD}
		\label{fig:DM-RD}
	\end{subfigure}
	\begin{subfigure}[b]{0.23\textwidth}
		\centering
		\includegraphics[width=\textwidth]{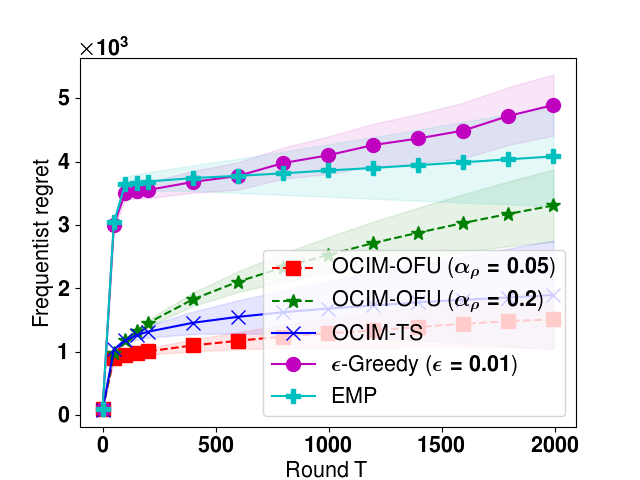}
		\caption{DM, IM}
		\label{fig:DM-IM}
	\end{subfigure}
	\caption{Frequentist regrets of algorithms for bipartite graph Yahoo-Ad and general graph DM.
	}\label{fig:GG}
\end{figure}
\begin{figure}[t]
	\centering
	\begin{subfigure}[b]{0.23\textwidth}
		\centering
		\includegraphics[width=\textwidth]{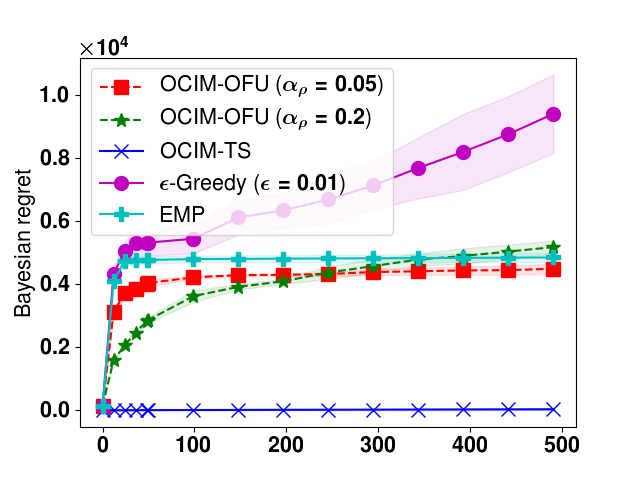}
		\caption{Yahoo-Ad, RD}
		\label{fig:Bayes_Yahoo-Ad-RD}
	\end{subfigure}
% 	\hfill
	\begin{subfigure}[b]{0.23\textwidth}
		\centering
		\includegraphics[width=\textwidth]{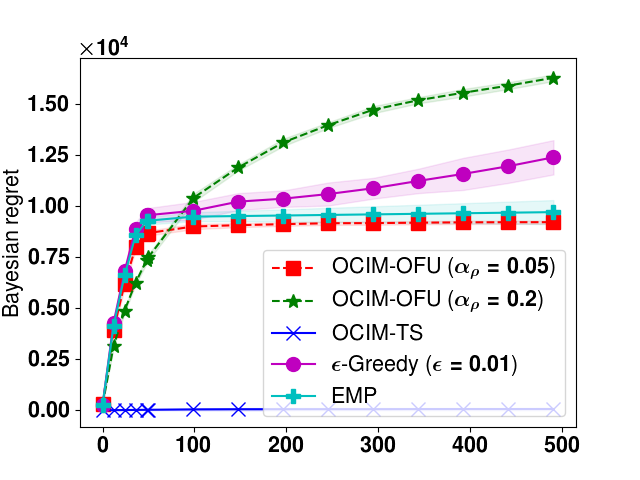}
		\caption{Yahoo-Ad, IM}
		\label{fig:Bayes_Yahoo-Ad-IM}
	\end{subfigure}
	\begin{subfigure}[b]{0.23\textwidth}
		\centering
		\includegraphics[width=\textwidth]{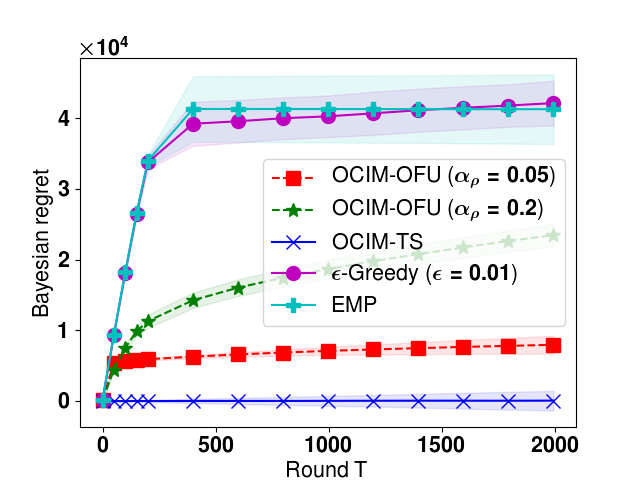}
		\caption{DM, RD}
		\label{fig:Bayes_DM-RD}
	\end{subfigure}
	\begin{subfigure}[b]{0.23\textwidth}
		\centering
		\includegraphics[width=\textwidth]{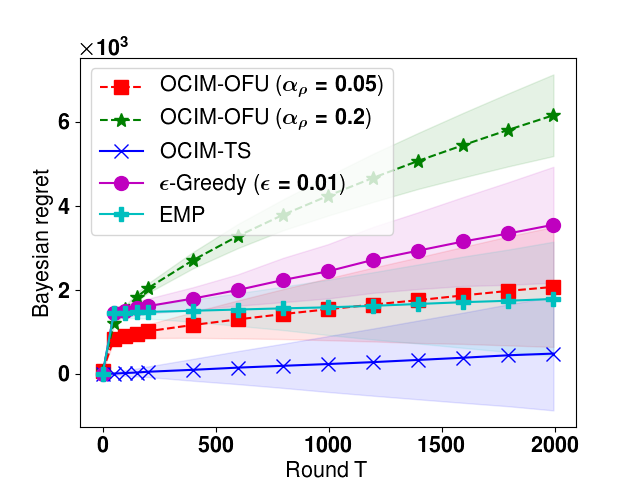}
		\caption{DM, IM}
		\label{fig:Bayes_DM-IM}
	\end{subfigure}
	\caption{Bayesian regrets of algorithms for bipartite graph Yahoo-Ad and general graph DM.
	}\label{fig:Bayesian}
\end{figure}

\textbf{Experimental result for Bayesian regrets}
We show Bayesian regrets of all algorithms in Figure~\ref{fig:Bayesian}. 
All algorithms except for OCIM-TS have similar curves.
OCIM-TS, however, achieves at least two orders of magnitudes lower regret ($BayesReg(T)\approx100$) compared with other algorithms.
The reason is that OCIM-TS leverages its prior knowledge to quickly converge to the optimal solution, but other algorithms cannot use this knowledge effectively.

\section{Conclusion and Future Work}\label{sec:conclusion}
In this paper, we formulate the OCIM problem and introduce a general C$^2$MAB-T framework for it. We prove that one important condition required by prior CMAB algorithms, the TPM condition, still holds, while the other one, monotonicity, is not satisfied. We propose three algorithms that balance between prior knowledge, offline computation, feedback and regret bound: OCIM-TS relies on prior knowledge and achieves logarithmic Bayesian regret; OCIM-OFU needs to solve a harder offline problem and achieves logarithmic frequentist regret; and OCIM-ETC requires less feedback at the expense of a worse frequentist regret bound. We extend our framework to settings with more complex competitor actions.

% In this paper, we formulate the OCIM problem and adapt a CMAB framework to solve it. We prove that one important condition required by prior CMAB algorithms, the TPM condition, still holds in the competitive setting, while the other one, monotonicity, is not satisfied. To remove the requirement of monotonicity, we introduce a new offline oracle and discuss its algorithm design for different graph structures. We propose a OCIM-OIFU algorithm with such an oracle that achieves logarithmic regret. We also design a CIM-ETC algorithm that requires less feedback and easier offline computation. It leads to worse regret bound than that of CIM-OIFU, showing the trade-off between achieving a better regret bound and requiring more feedback/computation in the OCIM problem.

This paper initiates the first study on OCIM, and it opens up a number of future directions. One is to design efficient offline approximation algorithms in the competitive setting when edge probabilities take a range of values. 
Another interesting direction is to study other partial feedback models, e.g. we only observe feedback from edges triggered by $A$ but not $B$.
A further direction is to look into distributed online learning, when competitors $A$ and $B$ both learn from the propagation and deploy their seeds accordingly.
\section*{Acknowledgements}
John C.S. Lui is supported in part by the GRF 14200321.

\bibliography{references}
\onecolumn
\appendix
\section*{Appendix}\label{sec:supplementary}

% \section{List}
% \begin{itemize}
%   \item Proof of Theorem \ref{thm:TPM} in Section \ref{sec:online}
%   \item Proof of Theorem \ref{thm:OIFU} in Section \ref{sec:online}
%   \item Proof of Theorem \ref{thm:ETC} in Section \ref{sec:online}
%   \item Non-submodular example in Section \ref{sec:offline}
%   \item Proof of \#P-hardness in Section \ref{sec:offline}
%   \item Proof of Lemma \ref{lem:UpperLower} in Section \ref{sec:offline} 
%   \item Example of a (binary) tree in in Section \ref{sec:offline} 
%   \item Additional experiments in Section \ref{sec:experiment} 
% \end{itemize}

\section{Proof of Theorem \ref{thm:TPM}}
\begin{proof}
Let $r_S^v(\bm{\mu})$ be the probability that node $v$ is activated by $A$. From the proof of Lemma 2 in~\citep{wang2017improving}, we know that  if for every node $v$ and every $\bm{\mu}$ and $\bm{\mu}'$ vectors we have
\begin{equation}
    \left|r_S^v(\bm{\mu}) - r_S^v(\bm{\mu'}) \right| \leq \sum_{e \in E} p_e^{S}(\bm{\mu}) \left|\mu_e - \mu'_e\right|,
\end{equation}
then Theorem~\ref{thm:TPM} is true. Notice that
\begin{align}
    r_S^{v}(\bm{\mu}) &= \mathbb{E}_{L \sim \bm{\mu}} \left[\mathds{1}\{v \text{ is activated by $A$ under } L\}\right]\\
    r_S^{v}(\bm{\mu'}) &= \mathbb{E}_{L' \sim \bm{\mu'}}\left[\mathds{1}\{v \text{ is activated by $A$ under } L'\}\right]
\end{align}
where $L$ and $L'$ are two live-edge graphs sampled under $\bm{\mu}$ and $\bm{\mu}'$, respectively. As mentioned in Sec.~\ref{sec:TPM}, we use an edge coupling method to compute the difference between $r_S^{v}(\bm{\mu})$ and $r_S^{v}(\bm{\mu'})$. Specifically, for each edge $e$, suppose we independently draw a uniform random variable $X_e$ over $[0,1]$, let
\begin{align*}
    &L(e) = L'(e) = 1, &&\text{if } X_e \leq \min(\mu_e, \mu'_e)\\
    &L(e) = 1, L'(e) = 0, &&\text{if } \mu'_e < X_e < \mu_e\\
    &L(e) = 0, L'(e) = 1, &&\text{if } \mu_e < X_e < \mu'_e\\
    &L(e) = L'(e) = 0, &&\text{if } X_e \geq \max(\mu_e, \mu'_e)
\end{align*}
where $L(e)$ represents the live/blocked state of edge $e$ in live-edge graph $L$. Notice that $L$ and $L'$ does not have the subgraph relationship. Let $\bm{X} := (X_1,\dots,X_e)$, the difference can be written as:
\begin{equation}
    r_S^v(\bm{\mu}) - r_S^v(\bm{\mu'}) = \mathbb{E}_{\bm{X}}[f(S,L,v) - f(S,L',v)],
\end{equation}
where $f(S,L,v) := \mathds{1}\{v \text{ is activated by $A$ under } L\}$.
Since $f(S,L,v) - f(S,L',v)$ could be 0, 1 or -1, we will discuss these cases separately.

1) $f(S,L,v) - f(S,L',v) = 0$.\\
This will not contribute to the expectation.

2) $f(S,L,v) - f(S,L',v) = 1$.\\
This will occur only if there exists a path such that: under $L$, $v$ can be activated by $A$ via this path, while under $L'$, $v$ cannot be activated by $A$ via this path. We denote this event as $\mathcal{E}_1$. We will show that $\mathcal{E}_1$ occurs only if at least one of $\mathcal{E}_1^A$ and $\mathcal{E}_1^B$ occurs.

$\mathcal{E}_1^A$: There exists a path $u \to v_1 \to \dots \to v_d = v$ such that:\\
1. $u$ is activated by $A$ under both $L$ and $L'$\\
2. edge $(u, v_1)$ is live under $L$ but not $L'$

$\mathcal{E}_1^B$: There exists a path $u' \to v_1' \to \dots \to v'_{d'} = v$ such that:\\
1. $u'$ is activated by $B$ under both $L$ and $L'$\\
2. edge $(u', v_1')$ is live under $L'$ but not $L$

\begin{lemma}\label{lemma:E_1}
$\mathcal{E}_1$ occurs only if at least one of $\mathcal{E}_1^A$ and $\mathcal{E}_1^B$ occurs.
\end{lemma}
\begin{proof}
Let us first discuss the relationship between $\mathcal{E}_1$, $\mathcal{E}_1^A$ and $\mathcal{E}_1^B$. For $\mathcal{E}_1$, if $v$ can be activated by $A$ under $L$ but not $L'$, it is because either: (a) some edge $e=(u,w)$ is live in $L$ but blocked in $L'$ while $u$ is $A$-activated (or equivalently $e$ is $A$-triggered); or (b) some edge $e$ is live in $L'$ but blocked in $L$ while $e$ is $B$-triggered. The former could be relaxed to $\mathcal{E}_1^A$, and the latter could be relaxed to $\mathcal{E}_1^B$. Notice that $\mathcal{E}_1^A$ and $\mathcal{E}_1^B$ are not mutually exclusive and we are interested in the upper bound of $\mathbb{P}\{\mathcal{E}_1\}$. 
\begin{figure}[H]
    \centering
    \includegraphics[width=0.7\columnwidth]{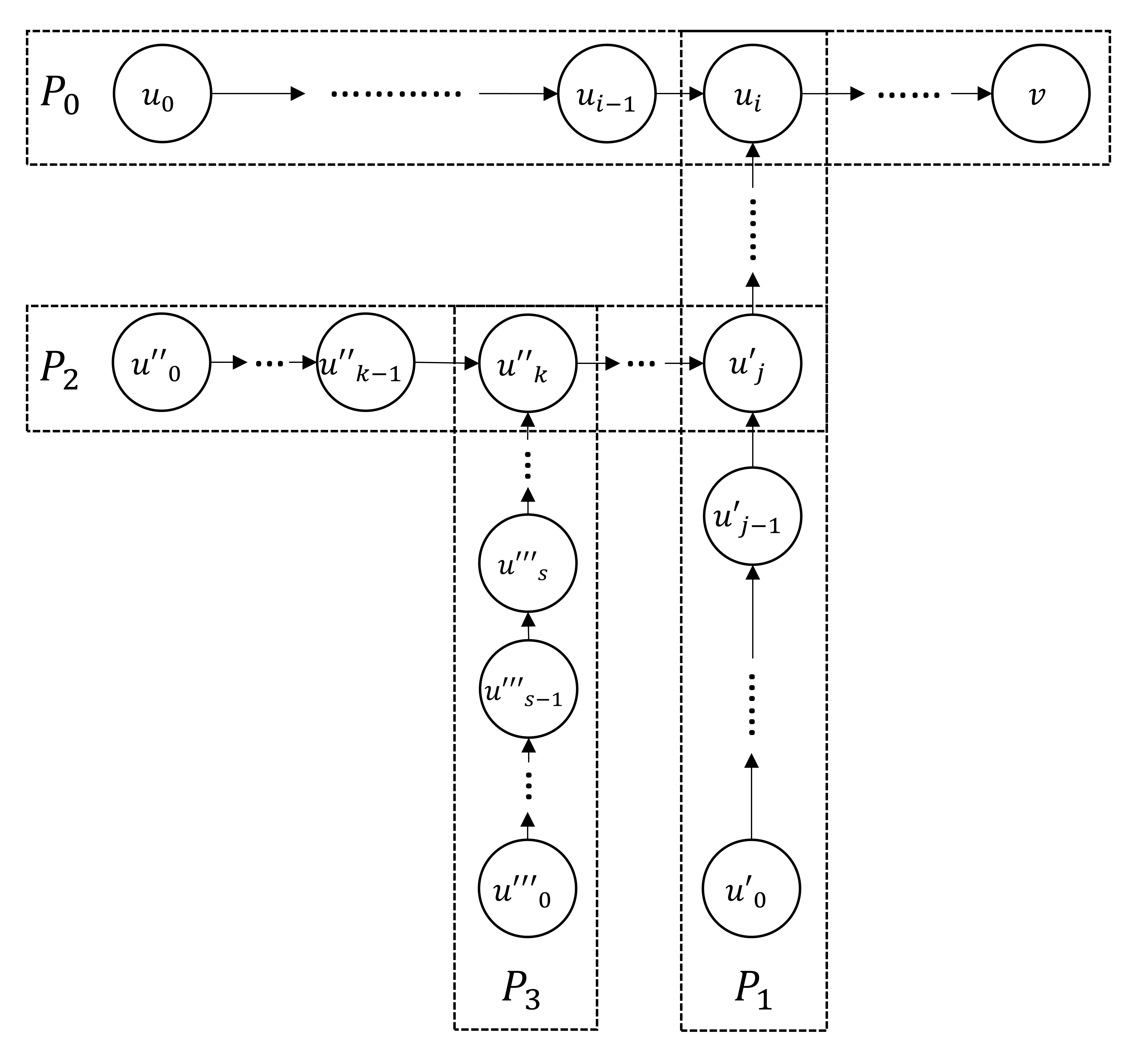}
    \caption{Path $P_0, P_1, P_2$ and $P_3$}
    \label{fig:TPM}
\end{figure}
Assuming $\mathcal{E}_1$ is true, consider the shortest path $P_0 := \{u_0 \to u_1 \to \dots \to u_{l_0} = v\}$ from one seed node of $A$, $u_0$, to node $v$, such that under $L$ node $v$ is activated by $A$ but under $L'$ it is not. When $\mathcal{E}_1$ is true, there must exist a node that is not activated by $A$ in $P_0$ under $L'$. We denote the first node from $u_0$ to $v$ (i.e., closest to $u_0$) in $P_0$ that is not activated by $A$ under $L'$ as $u_i$.

Next, let us consider the live/blocked state of edge $(u_{i-1}, u_i)$. We already know edge $(u_{i-1}, u_i)$ is live under $L$. If edge $(u_{i-1}, u_i)$ is blocked under $L'$, since $u_{i-1}$ is activated by $A$ under both $L$ and $L'$, it directly becomes $\mathcal{E}_1^A$. 
Otherwise, if edge $(u_{i-1}, u_i)$ is live under $L'$, the reason that node $u_i$ is not activated by $A$ could only be that it is activated by $B$. In this case, there must exist a path $P_1 := \{u'_0 \to u'_1 \to \dots \to u'_{l_1} = u_i\}$ from one seed node of $B$, $u'_0$, to node $u_i$, such that $u_i$ is activated by $B$ under $L'$ but not $L$. This can only occur when there exists a node that is not activated by $B$ in $P_1$ under $L$. We denote the first node from $u'_0$ to $u'_{l_1}$ (i.e., closest to $u'_0$) in $P_1$ that is not activated by $B$ under $L$ as $u'_j$. Notice that when the tie-breaking rule is $A>B$, we have $l_1 < i \leq l_0$ as $B$ should arrive at $u_i$ earlier than $A$; when the tie-breaking rule is $B>A$, we have $l_1 \leq i \leq l_0$ as $B$ should arrive at $u_i$ no later than $A$. We will discuss the case of the proportional tie-breaking rule separately after the discussion of the dominance tie-breaking rules.

Then, let us consider the live/blocked state of edge $(u'_{j-1}, u'_j)$. We already know edge $(u'_{j-1}, u'_j)$ is live under $L'$. If edge $(u'_{j-1}, u'_j)$ is blocked under $L$, since $u'_{j-1}$ is activated by $B$ under both $L$ and $L'$, it directly becomes $\mathcal{E}_1^B$. 
Otherwise, if edge $(u'_{j-1}, u'_j)$ is live under $L$, the reason that node $u'_j$ is not activated by $B$ could only be that it is activated by $A$. It also means neither $\mathcal{E}_1^A$ nor $\mathcal{E}_1^B$ occurs so far. In this case, there must exist a path $P_2 := \{u''_0 \to u''_1 \to \dots \to u''_{l_2} = u'_j\}$ from one seed node of $A$, $u''_0$, to node $u'_j$, such that $u'_j$ is activated by $A$ under $L$ but not $L'$. 
This can only occur when there exists a node that is not activated by $A$ in $P_2$ under $L'$. We denote the first node from $u''_0$ to $u''_{l_2}$ (i.e., closest to $u''_0$) in $P_2$ that is not activated by $A$ under $L'$ as $u''_k$.
Notice that when $A>B$, we have $l_2 \leq j \leq l_1 < l_0$ as $A$ should arrive at $u'_j$ no later than $B$; when $B>A$, we have $l_2 < j \leq l_1 \leq l_0$ as $A$ should arrive at $u'_j$ earlier than $B$. 

Now let us consider the live/blocked state of edge $(u''_{k-1}, u''_k)$. We already know edge $(u''_{k-1}, u''_k)$ is live under $L$. If edge $(u''_{k-1}, u''_k)$ is blocked under $L'$, since $u''_{k-1}$ is activated by $A$ under both $L$ and $L'$, it directly becomes $\mathcal{E}_1^A$. 
Otherwise, if edge $(u''_{k-1}, u''_k)$ is live under $L'$, the reason that node $u''_k$ is not activated by $A$ could only be that it is activated by $B$. In this case, there must exist a path $P_3 := \{u'''_0 \to u'''_1 \to \dots \to u'''_{l_3} = u''_k\}$ from one seed node of $B$, $u'''_0$, to node $u''_k$, such that $u''_k$ is activated by $B$ under $L'$ but not $L$. This can only occur when there exists a node that is not activated by $B$ in $P_3$ under $L$. We denote the first node from $u'''_0$ to $u'''_{l_3}$ (i.e., closest to $u'''_0$) in $P_3$ that is not activated by $B$ under $L$ as $u'''_s$.
Notice that when $A>B$, we have $l_3 < k \leq l_2 \leq l_1$ as $B$ should arrive at $u''_k$ earlier than $A$; when $B>A$, we have $l_3 \leq k \leq l_2 < l_1$ as $B$ should arrive at $u''_k$ no later than $A$. 

Again, let us consider the live/blocked state of edge $(u'''_{s-1}, u'''_s)$. We already know edge $(u'''_{s-1}, u'''_s)$ is live under $L'$. If edge $(u'''_{s-1}, u'''_s)$ is blocked under $L$, since $u'''_{s-1}$ is activated by $B$ under both $L$ and $L'$, it directly becomes $\mathcal{E}_1^B$. Otherwise, if edge $(u'''_{s-1}, u'''_s)$ is live under $L$, similar to the discussion above, we need to consider a new path $P_4$ with length $l_4$ and $l_4 < l_2$.

For the case of the proportional tie-breaking rule, in addition to the edge coupling, we also need to couple the permutation order~\citep{chen2011influence} for each node in $L$ and $L'$. More specific, for each node $j$, we randomly permute all of its in-neighbors, then when we need to break a tie on $j$, we find its activated neighbor $i$ that is ordered first in the permutation order, and assign the state of $i$ as $j$'s state. Assuming the same permutation order in $L$ and $L'$, let us consider path $P_0$ and $P_1$ again. If $l_0 = l_1$, then $u_i$ must be $v$. If $\mathcal{E}_1^A$ does not occur in $P_0$, then the only neighbor of $v$ in $P_1$ must be ordered before the only neighbor of $v$ in $P_0$ in the permutation order on $v$. However, if $\mathcal{E}_1^B$ does not occur in $P_1$, with such permutation order, it is impossible that $v$ is activated by $A$ under $L$ but not $L'$. As a result, if neither $\mathcal{E}_1^A$ nor $\mathcal{E}_1^B$ occurs in path $P_0$ and $P_1$, we have $l_2\le l_1 < l_0$ in the case of the proportional tie-breaking rule.

To sum up, if neither $\mathcal{E}_1^A$ nor $\mathcal{E}_1^B$ occurs in path $P_0$ and $P_1$, we need to check whether they could occur in a new path $P_2$ shorter than $P_0$, and $P_3$ shorter than $P_1$. As a result, we only need to check whether $\mathcal{E}_1^A$ or $\mathcal{E}_1^B$ occurs in the path with only one edge. In that case, $\mathcal{E}_1^A$ or $\mathcal{E}_1^B$ occurs for sure. Thus, by induction, we conclude that at least one of $\mathcal{E}_1^A$ and $\mathcal{E}_1^B$ occurs when considering  any path with more than one edge, so $\mathcal{E}_1$ will occur only if at least one of $\mathcal{E}_1^A$ and $\mathcal{E}_1^B$ occurs.
\end{proof}

Now, let us consider the two events in $\mathcal{E}_1^A$ for a specific edge $e = (u, v_1)$. We find that the first event \{$u$ is activated by $A$ under both $L$ and $L'$\}, is independent of the second event \{edge $e$ is live under $L$ but not $L'$\}, since the live/blocked state of edge $e$ does not affect the activation of its tail node $u$. Also, for edge $e = (u, v_1)$, the probability of these two events can be written as
\begin{align}
    &\mathbb{P}\{\text{$u$ is activated by $A$ under $L$ and $L'$}\} = \mathbb{P}\{\text{$e$ is triggered by $A$ under $L$ and $L'$}\},\\
    &\mathbb{P}\{\text{$e$ is live under $L$ but not $L'$}\} = 
    \begin{cases}
        \mu_e - \mu'_e &\quad\text{if $\mu_e > \mu'_e$}\\
        0 &\quad\text{otherwise.} \\ 
     \end{cases}
\end{align}
As a result, we have:
\begin{equation}
    \mathbb{P}\{\mathcal{E}_1^A\} \leq \sum_{e:\,\mu_e > \mu'_e} \mathbb{P}\{\text{$e$ is triggered by $A$ under $L$ and $L'$}\} (\mu_e - \mu'_e)
\end{equation}
Since $\mathcal{E}_1^A$ and $\mathcal{E}_1^B$ are symmetric, we also have:
\begin{equation}
    \mathbb{P}\{\mathcal{E}_1^B\} \leq \sum_{e:\,\mu'_e > \mu_e} \mathbb{P}\{\text{$e$ is triggered by $B$ under $L$ and $L'$}\} (\mu'_e - \mu_e)
\end{equation}
Combining with Lemma.~\ref{lemma:E_1}, we have
\begin{equation}
    \mathbb{P}\{\mathcal{E}_1\} \leq \mathbb{P}\{\mathcal{E}_1^A\} + \mathbb{P}\{\mathcal{E}_1^B\}
\end{equation}

3) $f(S,\bm{w_1},v) - f(S,\bm{w_2},v) = -1$.\\
Similar to the previous case, this will occur only if there exists a path such that: under $L'$, $v$ can be activated by $A$ via this path, while under $L$, $v$ cannot be activated by $A$ via this path. We denote this event as $\mathcal{E}_{-1}$. We show that $\mathcal{E}_{-1}$ occurs only if at least one of $\mathcal{E}_{-1}^A$ and $\mathcal{E}_{-1}^B$ occurs.

$\mathcal{E}_{-1}^A$: There exists a path $u \to v_1 \to \dots \to v_d = v$ such that:\\
1. $u$ is activated by $A$ under both $L$ and $L'$\\
2. edge $(u, v_1)$ is live under $L'$ but not $L$

$\mathcal{E}_{-1}^B$: There exists a path $u' \to v_1' \to \dots \to v'_{d'} = v$ such that:\\
1. $u'$ is activated by $B$ under both $L$ and $L'$\\
2. edge $(u', v_1')$ is live under $L$ but not $L'$

Since they are symmetric with $\mathcal{E}_{1}^A$ and $\mathcal{E}_{1}^B$, following the same analysis, we can get
\begin{align}
    &\mathbb{P}\{\mathcal{E}_{-1}^A\} \leq \sum_{e:\,\mu'_e > \mu_e} \mathbb{P}\{\text{$e$ is triggered by $A$ under $L$ and $L'$}\} (\mu'_e - \mu_e)\\
    &\mathbb{P}\{\mathcal{E}_{-1}^B\} \leq \sum_{e:\,\mu_e > \mu'_e} \mathbb{P}\{\text{$e$ is triggered by $B$ under $L$ and $L'$}\} (\mu_e - \mu'_e)\\
    &\mathbb{P}\{\mathcal{E}_{-1}\} \leq \mathbb{P}\{\mathcal{E}_{-1}^A\} + \mathbb{P}\{\mathcal{E}_{-1}^B\}
\end{align}

Combining all cases together, we have:
\begin{align}
    \left|r_S^v(\bm{\mu}) - r_S^v(\bm{\mu'})\right| &= \left|\mathbb{E}_{\bm{X}}[f(S,L,v) - f(S,L',v)]\right|\nonumber\\
    &\leq \left|1\cdot\mathbb{P}\{\mathcal{E}_{1}\} + (-1)\cdot \mathbb{P}\{\mathcal{E}_{-1}\}\right|\nonumber\\
    &\leq \left|1\cdot\left(\mathbb{P}\{\mathcal{E}_{1}^A\}+\mathbb{P}\{\mathcal{E}_{1}^B\}\right) + (-1)\cdot \left(\mathbb{P}\{\mathcal{E}_{-1}^A\}+\mathbb{P}\{\mathcal{E}_{-1}^B\}\right)\right|\nonumber\\
    &\leq \sum_{e \in E} \mathbb{P}\{\text{$e$ is triggered by $A$ or $B$ under $L$ and $L'$}\} \left|\mu_e - \mu'_e\right|.\label{eq:TPM_r}
\end{align}
The last inequality above is due to: 
\begin{align*}
    |\mathbb{P}\{\mathcal{E}_{1}^A\}-\mathbb{P}\{\mathcal{E}_{-1}^B\}| \leq \sum_{e:\,\mu_e > \mu'_e} \mathbb{P}\{\text{$e$ is triggered by $A$ or $B$ under $L$ and $L'$}\} |\mu_e - \mu'_e|\\
    |\mathbb{P}\{\mathcal{E}_{1}^B\}-\mathbb{P}\{\mathcal{E}_{-1}^A\}| \leq \sum_{e:\,\mu'_e > \mu_e} \mathbb{P}\{\text{$e$ is triggered by $A$ or $B$ under $L$ and $L'$}\} |\mu_e - \mu'_e|
\end{align*}
Notice that Eq.\eqref{eq:TPM_r} could be relaxed to:
\begin{align}
    \left|r_S^v(\bm{\mu}) - r_S^v(\bm{\mu'})\right| &\leq \sum_{e\in E} \mathbb{P}\{\text{$e$ is triggered by $A$ or $B$ under $L$}\} \left|\mu_e - \mu'_e\right|\nonumber\\
    &\leq \sum_{e \in E} p_e^{S}(\bm{\mu}) \left|\mu_e - \mu'_e\right|.
\end{align}
\end{proof}

\section{Proof of Theorem \ref{thm:TS}} \label{appendix:TS}
\begin{proof}
We define $G^{(t)}$ as the feedback of OCIM in round $t$, which includes the outcomes of $X_i^{(t)}$ for all $i\in \tau_t$. We denote by $\mathcal{F}_{t-1}$ the history $(S^{(1)}, G^{(1)}, \cdots, S^{(t-1)}, G^{(t-1)})$ of observations available to the player when choosing an action $S^{(t)}$. For the Bayesian analysis, we assume the mean vector $\bm{\mu}$ follows a prior distribution $\mathcal{Q}$. In round $t$, given $\mathcal{F}_{t-1}$, we define the posterior distribution of $\bm{\mu}$  as $\mathcal{Q}^{(t)}$ (i.e., $\bm{\mu}^{(t)}\sim \mathcal{Q}^{(t)}$ where $\bm{\mu}^{(t)}$ is given in Alg.~\ref{alg:OCIM-TS}). As mentioned in Section~\ref{sec:TS}, OCIM-TS allows any benchmark offline oracles, including approximation oracles. We consider a general benchmark oracle $\mathcal{O}(S_B, \bm{\mu})$. As oracle $\mathcal{O}$ might be a randomized policy (e.g., an $(\alpha,\beta)$-approximation oracle with success probability $\beta$), we use a random variable $\omega \sim \Omega$ to represent all its randomness. In order to discuss the performance of OCIM-TS with oracle $\mathcal{O}$, we rewrite the Bayesian regret in Eq.\eqref{eq:BR} as
\begin{equation}\label{eq:generalBR}
        BayesReg(T) = \mathbb{E}_{\omega \sim \Omega, \bm{\mu}\sim \mathcal{Q}}  \left[\sum_{t=1}^T  \left(r_{\mathcal{O}(S_B^{(t)}, \bm{\mu})}(\bm{\mu}) - r_{\mathcal{O}(S_B^{(t)}, \bm{\mu_t)})}(\bm{\mu})\right) \right].
        % &=\mathbb{E}_{\omega\sim \Omega, \mathcal{F}_{t-1}}\mathbb{E}_{\mathcal{F}_{t-1}}  \left[\sum_{t=1}^T  \left(\alpha \cdot \beta \cdot \text{opt}^{(t)}(\bm{\mu}) - r_{S^{\mathcal{A},(t)}}(\bm{\mu})\right)\mid \mathcal{F}_{t-1} \right].
\end{equation}
Notice that $\mathcal{O}(S_B^{(t)}, \bm{\mu})$ is the action taken by the player if the true $\bm{\mu}$ is known, while $\mathcal{O}(S_B^{(t)}, \bm{\mu}_t)$ is the real action chosen by OCIM-TS. The original regret definition in Eq.\eqref{eq:BR} is a special case of Eq.\eqref{eq:generalBR} for an $(\alpha,\beta)$-approximation oracle, and will focus on this general form in this proof.

The key step to derive the Bayesian regret bound of OCIM-TS is to show that the conditional distributions of $\bm{\mu}$ and $\bm{\mu}_{t}$ given $\mathcal{F}_{t-1}$ are the same:
\begin{equation}\label{eq:CondDist}
    \mathbb{P}(\bm{\mu} = \cdot \mid \mathcal{F}_{t-1}) = \mathbb{P}(\bm{\mu}_t = \cdot \mid \mathcal{F}_{t-1}),
\end{equation}
which is true since we use Thompson sampling to update the posterior distribution of $\bm{\mu}$. With this finding, we consider the Bayesian regret in Eq.\eqref{eq:BR}: 
\begin{align}
\nonumber
    & BayesReg(T)\\
    =&\mathbb{E}_{\omega \sim \Omega}\left[ \sum_{t=1}^{T}\mathbb{E}_{\bm{\mu}\sim\mathcal{Q},\bm{\mu}_t\sim\mathcal{Q}_t}\left[r_{\mathcal{O}(S_B^{(t)}, \bm{\mu})}(\bm{\mu}) - r_{\mathcal{O}(S_B^{(t)}, \bm{\mu}_t)}(\bm{\mu})\right] \right] \\\label{eq:BR-mu}
    =&\mathbb{E}_{\omega \sim \Omega}\left[ \sum_{t=1}^{T}\mathbb{E}_{\mathcal{F}_{t-1}}\left[\mathbb{E}_{\bm{\mu}\sim\mathcal{Q},\bm{\mu}_t\sim\mathcal{Q}_t}\left[r_{\mathcal{O}(S_B^{(t)}, \bm{\mu})}(\bm{\mu}) - r_{\mathcal{O}(S_B^{(t)}, \bm{\mu}_t)}(\bm{\mu})\right] \mid \mathcal{F}_{t-1}\right]\right] \\\label{eq:BR-mu_t}
    =& \mathbb{E}_{\omega \sim \Omega}\left[ \sum_{t=1}^{T}\mathbb{E}_{\mathcal{F}_{t-1}}\left[\mathbb{E}_{\bm{\mu}\sim\mathcal{Q},\bm{\mu}_t\sim\mathcal{Q}_t}\left[r_{\mathcal{O}(S_B^{(t)}, \bm{\mu}_t)}(\bm{\mu}_t) - r_{\mathcal{O}(S_B^{(t)}, \bm{\mu}_{t})}(\bm{\mu})\right] \mid \mathcal{F}_{t-1}\right]\right] \\\label{eq:BR_last}
    =& \mathbb{E}\left[ \sum_{t=1}^{T} \left[r_{\mathcal{O}(S_B^{(t)}, \bm{\mu}_t)}(\bm{\mu}_t) - r_{\mathcal{O}(S_B^{(t)}, \bm{\mu}_{t})}(\bm{\mu})\right]\right],
\end{align}
where Eq.\eqref{eq:BR-mu_t} comes from applying Eq.\eqref{eq:CondDist} to Eq.\eqref{eq:BR-mu}. Let $S_t = \mathcal{O}(S_B^{(t)}, \bm{\mu}_t)$ and $\mathcal{C}_t=\{\bm{\mu}': |\mu'_i - \hat{\mu}_{i,t}| \le \rho_{i,t}, \forall i\}$, where $\rho_{i,t}=\sqrt{3\ln t/2T_{i,t-1}}$ and $T_{i,t-1}$ is the total number of times arm $i$ is played until round $t$. We define ${\Delta}_{S_t}=r_{S_t}(\bm{\mu}_t) - r_{S_t}(\bm{\mu})$ and $M=\sqrt{576\tilde{C}^2mK\ln T/T}$. By Eq.\eqref{eq:BR_last}, we have 
\begin{align}
\nonumber
    & BayesReg(T)\\
=&\mathbb{E}[\sum_{t=1}^T \Delta_{S_t}]\\\nonumber
\le & \underbrace{\mathbb{E}\left[\sum_{t=1}^T \Delta_{S_t} \mathbb{I}\{\Delta_{S_t} \ge M, \bm{\mu}_t \in \mathcal{C}_t, \bm{\mu} \in \mathcal{C}_t, \mathcal{N}_t^{\text{t}}\}\right]}_{(a)} + \underbrace{\mathbb{E}[\sum_{t=1}^T \Delta_{S_t}\mathbb{I}\{\bm{\mu}_t \notin \mathcal{C}_t\}] + \mathbb{E}[\sum_{t=1}^T\Delta_{S_t} \mathbb{I}\{\bm{\mu} \notin \mathcal{C}_t\}]}_{(b)} \\ \label{eq:BR3terms}
&+ \underbrace{\mathbb{E}[\sum_{t=1}^T\Delta_{S_t} \mathbb{I}\{\Delta_{S_t} \le M\}]}_{(c)} + \underbrace{\mathbb{E}[\sum_{t=1}^T\Delta_{S_t} \mathbb{I}\{\neg \mathcal{N}_t^{\text{t}}\}]}_{(d)}
\end{align}
We can bound these three terms separately. For term (a), 
when $\bm{\mu}_t \in \mathcal{C}_t, \bm{\mu} \in \mathcal{C}_t$, we could bound $|\mu_{i,t} - \mu_{i}| \le |\mu_{i,t} - \hat{\mu}_{i,t}| + |\mu_i - \hat{\mu}_{i,t}| \le 2\rho_{i,t}, \forall i$.
When $\Delta_{S_t}\geq M$ and $\mathcal{N}_t^{\text{t}}$ (Definition 7 in \citep{wang2017improving}) holds, by the proof of Lemma 5 in~\citep{wang2017improving}, we have $\Delta_{S_t}\le\sum_{i \in \tilde{S}_t}\kappa_{j_i,T}(M_i, N_{i,j_i, t-1})$ where $\tilde{S}_t$ is the set of arms triggered by $S_t$ and $\kappa_{j_i,T}(M_i, N_{i,j_i, t-1})$ is defined in~\citep{wang2017improving}. We have

\begin{align*}
(a) &= \mathbb{E}\left[\sum_{t=1}^T \Delta_{S_t} \mathbb{I}\{ \Delta_{S_t} \ge M, \bm{\mu}_t \in \mathcal{C}_t, \bm{\mu} \in \mathcal{C}_t, \mathcal{N}_t^{\text{t}}\}\right]\\
&\le \mathbb{E}\left[ \sum_{t=1}^T \sum_{i \in \tilde{S}_t}\kappa_{j_i,T}(M_i, N_{i,j_i, t-1}) \right]\\
&\le \mathbb{E}\left[ \sum_{i \in [m]} \sum_{j=1}^{+\infty}\sum_{s=0}^{N_{i,j,T}-1} \kappa_{j,T}(M,s)\right]\\
\\
&\le 4\tilde{C}m + \sum_{i \in [m]} \frac{576 \tilde{C}^2 K \ln T}{M}
\end{align*}

For term (b), we can observe that $\mathbb{E}[\mathbb{I}\{\bm{\mu}\in \mathcal{C}_t\}|\mathcal{F}_{t-1}]=\mathbb{E}[\mathbb{I}\{\bm{\mu}_t\in \mathcal{C}_t\}|\mathcal{F}_{t-1}]$, since  $\mathcal{C}_t$ is determined given $\mathcal{F}_{t-1}$, and given $\mathcal{F}_{t-1}$, $\bm{\mu}$ and $\bm{\mu}_t$ follow the same distribution. Since $\max_{S_t} \Delta_{S_t} \le n$, we have

\begin{align*}
(b) &= \mathbb{E}[\sum_{t=1}^T \Delta_{S_t}\mathbb{I}\{\bm{\mu}_t \notin \mathcal{C}_t\}] + \mathbb{E}[\sum_{t=1}^T\Delta_{S_t} \mathbb{I}\{\bm{\mu} \notin \mathcal{C}_t\}]\\
&\le n \left( \mathbb{E}[\sum_{t=1}^T \mathbb{I}\{\bm{\mu}_t \notin \mathcal{C}_t\}] + \mathbb{E}[\sum_{t=1}^T \mathbb{I}\{\bm{\mu} \notin \mathcal{C}_t\}]\right)\\
&= n \left( \mathbb{E}\left[\sum_{t=1}^T \mathbb{E}\left[\mathbb{I}\{\bm{\mu}_t \notin \mathcal{C}_t\}|\mathcal{F}_{t-1}\right]\right]\right) + n \left( \mathbb{E}\left[\sum_{t=1}^T \mathbb{E}\left[\mathbb{I}\{\bm{\mu} \notin \mathcal{C}_t\}|\mathcal{F}_{t-1}\right]\right]\right)\\
&= 2n \left( \mathbb{E}\left[\sum_{t=1}^T \mathbb{E}\left[\mathbb{I}\{\bm{\mu} \notin \mathcal{C}_t\}|\mathcal{F}_{t-1}\right]\right]\right) \\
&=2n \left( \mathbb{E}\left[\sum_{t=1}^T \mathbb{I}\{\bm{\mu} \notin \mathcal{C}_t\}\right]\right)\\
&= 2n \left( \sum_{t=1}^T \mathbb{P}\left(\bm{\mu} \notin \mathcal{C}_t\right)\right) \\
&\le \frac{2\pi^2 mn}{3}
\end{align*}

For term (c), we can bound it by
\begin{align*}
(c) =\mathbb{E}[\sum_{t=1}^T\Delta_{S_t} \mathbb{I}\{\Delta_{S_t} \le M\}] \le TM
\end{align*}

For term (d), similar to Eq.(20) in~\citep{wang2017improving}, we have
\begin{align*}
(d) =\mathbb{E}[\sum_{t=1}^T\Delta_{S_t} \mathbb{I}\{\neg \mathcal{N}_t^{\text{t}}\}] \le \frac{\pi^2}{6} \cdot \sum_{i\in[m]}j_{\max}^i \cdot n
\end{align*}

Combine them together, we have
\begin{align}\nonumber
BayesReg(T) \le& 4\tilde{C}m + \sum_{i \in [m]} \frac{576 \tilde{C}^2 K \ln T}{M} + \frac{2\pi^2 mn}{3} + TM + \frac{\pi^2}{6} \cdot \sum_{i\in[m]}j_{\max}(M) \cdot n
\end{align}
where $j_{\max}(M) = \left\lceil \log_2 \frac{2\tilde{C}K}{M}\right\rceil_0$. Take $M=\sqrt{576\tilde{C}^2mK\ln T/T}$, we finally get finally get the Bayesian regret bound of TS-OCIM:
\begin{align}\textstyle
\nonumber
    &BayesReg(T) \leq 12\tilde{C}\sqrt{mKT\ln T}+ 2\tilde{C}m +\left(\left\lceil \log_2 \frac{T}{18\ln T}\right\rceil_0 + 4\right) \cdot \frac{\pi^2}{6} \cdot n \cdot m.
\end{align}
\end{proof}

\section{Proof of Theorem \ref{thm:OIFU}}\label{appendix:TS:OFU}
\begin{proof}
We first introduce the following definitions to assist our analysis.
Recall that $\bm{\mathcal{S}}^{(t)}$ is the action space in round $t$. We define the reward gap $\Delta_{S}^{(t)}{=}\max(0, \alpha \cdot \text{opt}^{(t)}(\bm{\mu}) - r_S(\bm{\mu}))$ for all actions $S \in \bm{\mathcal{S}}^{(t)}$. 
For each base arm $i$, we define $\Delta^{i,T}_{\max} = \max_{t\in[T]}\sup_{S\in\bm{\mathcal{S}}^{(t)}:p_i^S(\bm{\mu}) > 0, \Delta_{S}^{(t)} > 0} \Delta_{S}^{(t)}$ and  $\Delta^{i,T}_{\min} = \min_{t\in[T]}\inf_{S\in\mathcal{S}^{(t)}:p_i^S(\bm{\mu}) > 0, \Delta_{S}^{(t)} > 0} \Delta_{S}^{(t)}$. 
If there is no action $S$ such that $p_i^S(\bm{\mu}) > 0$ and $\Delta_{S}^{(t)} > 0$, we define $\Delta^{i,T}_{\max} = 0$ and $\Delta^{i,T}_{\min} = +\infty$. 
We define $\Delta^{(T)}_{\max} = \max_{i \in [m]}\Delta^{i,T}_{\max}$ and $\Delta^{(T)}_{\min} = \min_{i \in [m]}\Delta^{i,T}_{\min}$. Let $\widetilde{S} = \{i\in [m] \mid p_i^S(\bm{\mu}) > 0\}$ be the set of arms that can be triggered by $S$. We define $K = \max_{S\in\bm{\mathcal{S}}^{(t)}}|\widetilde{S}|$ as the largest number of arms could be triggered by a feasible action. We use $\lceil x \rceil_0$ to denote $\max\{\lceil x \rceil, 0\}$.
If $\Delta^{(T)}_{\min} > 0$, we provide the distribution-dependent bound of the OCIM-OFU algorithm.
\begin{align}\textstyle
    \text{Reg}_{\alpha,\beta}(T; \bm{\mu}) \leq \sum_{i \in [m]} \frac{576\tilde{C}^2K\ln T}{\Delta^{i,T}_{\min}}  + 4\tilde{C}m +\sum_{i \in [m]} \left(\left\lceil \log_2 \frac{2\tilde{C}K}{\Delta^{i,T}_{\min}}\right\rceil_0 + 2\right)\cdot \frac{\pi^2}{6}\cdot\Delta^{(T)}_{\max}.\nonumber
\end{align}

To prove the distribution-dependent and the distribution-independent regret bounds, we generally follow the proof of Theorem 1 in~\cite{wang2017improving}. However, since we extend the original CMAB problem to a new contextual setting where the action space $\bm{\mathcal{S}}^{(t)}$ is the context, and monotonicity does not hold in the OCIM setting, we need to modify their analysis to tackle these changes. We introduce a positive real number $M_i$ for each arm $i$ and define $M_{S^{(t)}} = \max_{i\in\tilde{S}^{(t)}} M_i$.
Define
$$\kappa_{j, T}(M, s) = \begin{cases}
4\cdot 2^{-j}\tilde{C}, &\mbox{if } s=0,\\
2\tilde{C}\sqrt{\frac{72\cdot 2^{-j} \ln T}{s}}, &\mbox{if } 1\le s\le \ell_{j, T}(M),\\
0, &\mbox{if } s \ge \ell_{j, T}(M)+1,
\end{cases}$$
where
$$\ell_{j, T}(M)=\left\lfloor\frac{288\cdot 2^{-j} \tilde{C}^2 K^2 \ln T}{M^2}\right\rfloor.$$
Let $\mathcal{N}_{t}^{\text{s}}$ be the event that at the beginning of round $t$, for every arm $i \in [m]$, $|\hat{\mu}_{i,t} - \mu_{i}| \leq 2\rho_{i,t}$. Let $\mathcal{H}_t$ be the event that at round $t$ oracle $\widetilde{\mathcal{O}}$ outputs a solution, $S^{(t)} = \{S_A^{(t)}, S_B^{(t)}\}$ and $\bm{\mu}^{(t)} = (\mu_1^{(t)},\dots,\mu_m^{(t)})$, such that $r_{S^{(t)}}(\bm{\mu}^{(t)}) < \alpha \cdot r_{S^{*}}(\bm{\mu}^*)$, i.e., oracle $\widetilde{\mathcal{O}}$ fails to output an $\alpha$-approximate solution. Let $\mathcal{N}_{t}^{\text{t}}$ be the event that the triggering is nice at the beginning of round $t$ (Definition 7 in~\citep{wang2017improving}). The following lemma explains how $\kappa$ contributes to the regret.
\begin{lemma}\label{lemma:kappa}
	For any vector $\{M_i\}_{i\in[m]}$ of positive real numbers and $1\le t \le T$, if $\{\Delta^{(t)}_{S^{(t)}} \ge M_{S^{(t)}}\}, \lnot \mathcal{H}_t, \mathcal{N}_{t}^{\text{s}}$ and $\mathcal{N}_{t}^{\text{t}}$ hold,
	we have
	$$\Delta^{(t)}_{S^{(t)}} \le \sum_{i\in \tilde{S}^{(t)}} \kappa_{j_i, T}(M_i, N_{i, j_i, t-1}),$$
	where $j_i$ is the index of the TP group with $S^{(t)}\in \mathcal{S}_{i, j_i}$ (see Definition 5 in~\citep{wang2017improving}).
\end{lemma}
\begin{proof}
By $\mathcal{N}_{t}^{\text{s}}$ and $0 \leq \mu_i \leq 1$ for all $i\in [m]$, we have
\begin{equation}
    \forall i\in [m], \mu_{i} \in c_{i,t} = \left[(\hat{\mu}_{i,t} - \rho_{i,t})^{0+}, (\hat{\mu}_{i,t} + \rho_{i,t} )^{1-}\right].
\end{equation}
It means that we have the correct estimated range of $\mu_i$ for all $i \in [m]$ at round $t$. Combining with $\neg \mathcal{H}_t$ for the offline oracle $\widetilde{\mathcal{O}}$, we have
\begin{equation}
    r_{S^{(t)}}(\bm{\mu}^{(t)}) \geq \alpha \cdot r_{S^{*}}(\bm{\mu}^*) \geq  \alpha \cdot \text{opt}^{(t)}(\bm{\mu}) = r_{S^{(t)}}(\bm{\mu}) + \Delta^{(t)}_{S^{(t)}}.
\end{equation}
By the TPM condition in Theorem.~\ref{thm:TPM}, we have
\begin{equation}
\Delta^{(t)}_{S^{(t)}} \leq  r_{S^{(t)}}(\bm{\mu}^{(t)}) - r_{S^{(t)}}(\bm{\mu})  \leq \tilde{C} \sum_{i\in [m]} p_i^{S^{(t)}}(\bm{\mu}) |\mu_{i}^{(t)} - \mu_{i}|.
\end{equation}
We want to bound $\Delta^{(t)}_{S^{(t)}}$ by bounding $p_i^{S^{(t)}}(\bm{\mu})|\mu_{i}^{(t)} - \mu_{i}|$. We first perform a transformation. Since $\Delta^{(t)}_{S^{(t)}} \ge M_{S^{(t)}}$, we have $\tilde{C} \sum_{i\in [m]} p_i^{S^{(t)}}(\bm{\mu}) |\mu_{i}^{(t)} - \mu_{i}| \ge \Delta^{(t)}_{S^{(t)}} \ge M_{S^{(t)}}$. Then we have
\begin{align}
\Delta^{(t)}_{S^{(t)}}
&\le \tilde{C} \sum_{i\in [m]} p_i^{S^{(t)}}(\bm{\mu}) |\mu_{i}^{(t)} - \mu_{i}|\nonumber\\
&\le -M_{S^{(t)}} + 2\tilde{C} \sum_{i\in [m]} p_i^{S^{(t)}}(\bm{\mu}) |\mu_{i}^{(t)} - \mu_{i}|\nonumber\\
&\le 2\tilde{C}\sum_{i\in [m]} \left[p_i^{S^{(t)}}(\bm{\mu}) |\mu_{i}^{(t)} - \mu_{i}| - \frac{M_i}{2\tilde{C}K}\right]. \label{eq:TPMkappa.transform}
\end{align}
In fact, if $\mathcal{N}_{t}^{\text{s}}$ holds and $\mu_{i}^{(t)} \in c_{i,t}$ for all $i \in [m]$,
\begin{equation}
    \forall i\in [m], |\mu_{i}^{(t)} - \mu_{i}| \leq 2\rho_{i,t} = 2\sqrt{\frac{3\ln t}{2T_{i, t-1}}}.
\end{equation}
So far, all requirements on bounding $\Delta_{S_t}$ in Lemma 5 from~\citep{wang2017improving} are also satisfied by $\Delta^{(t)}_{S^{(t)}}$ of OCIM-OFU algorithm in the OCIM setting without monotonicity. We can then follow the same steps to bound $p_i^{S^{(t)}}(\bm{\mu})|\mu_{i}^{(t)} - \mu_{i}|$ in the two cases they considered (combining their Eq.(11)-(13)) and get
\begin{align}
\Delta^{(t)}_{S^{(t)}}
&\le 2\tilde{C}\sum_{i\in [m]} \left[p_i^{S^{(t)}}(\bm{\mu}) |\mu_{i}^{(t)} - \mu_{i}| - \frac{M_i}{2\tilde{C}K}\right]\nonumber\\
&\le \sum_{i\in \tilde{S}^{(t)}} \kappa_{j_i, T}(M_i, N_{i, j_i, t-1}). \nonumber
\end{align}
\end{proof}
With Lemma~\ref{lemma:kappa}, we can follow the proof of Lemma 6 in~\citep{wang2017improving} to bound the regret when $\{\Delta^{(t)}_{S^{(t)}} \ge M_{S^{(t)}}\}, \lnot \mathcal{H}_t, \mathcal{N}_{t}^{\text{s}}$ and $\mathcal{N}_{t}^{\text{t}}$ hold.
\begin{equation}
    Reg(\{\Delta^{(t)}_{S^{(t)}}\geq M_{S^{(t)}}\} \land \lnot \mathcal{H}_t \land \mathcal{N}_{t}^{\text{s}} \land \mathcal{N}_{t}^{\text{t}})
\le \sum_{i\in [m]} \frac{576\tilde{C}^2K\ln T}{M_i} + 4\tilde{C}m.
\end{equation}

Finally, we take $M_i = \Delta^{i,T}_{\min}$. If $\Delta^{(t)}_{S^{(t)}}< M_{S^{(t)}}$, then $\Delta^{(t)}_{S^{(t)}} = 0$, since we have either $\tilde{S}^{(t)} =  \emptyset$ or $\Delta^{(t)}_{S^{(t)}} < M_{S^{(t)}} \leq M_i$ for some $i \in \tilde{S}^{(t)}$. Thus, no regret is accumulated when $\Delta^{(t)}_{S^{(t)}}< M_{S^{(t)}}$. Following Eq.(17)-(21) in~\citep{wang2017improving}, we can derive the distribution-dependent regret bound
\begin{align}\textstyle
    \text{Reg}_{\alpha,\beta}(T; \bm{\mu}) \leq \sum_{i \in [m]} \frac{576\tilde{C}^2K\ln T}{\Delta^{i,T}_{\min}}  + 4\tilde{C}m +\sum_{i \in [m]} \left(\left\lceil \log_2 \frac{2\tilde{C}K}{\Delta^{i,T}_{\min}}\right\rceil_0 + 2\right)\cdot \frac{\pi^2}{6}\cdot\Delta^{(T)}_{\max}.
\end{align}
To derive the distribution-independent bound, we take $M_i=M=\sqrt{(576\tilde{C}^2mK\ln T)/T}$, follow Eq.(23) in~\citep{wang2017improving} and get
\begin{align}\textstyle
    \text{Reg}_{\alpha,\beta}(T; \bm{\mu}) \leq 12\tilde{C}\sqrt{mKT\ln T}+ 2\tilde{C}m + \left(\left\lceil \log_2 \frac{T}{18\ln T}\right\rceil_0 + 2\right)\cdot \frac{\pi^2}{6}\cdot n \cdot m.
\end{align}
\end{proof}

\section{Computational Efficiency of OCIM-OFU}
\subsection{Proof of Theorem \ref{thm:sharpP}}
\begin{proof}
In order to prove Theorem \ref{thm:sharpP}, we first introduce a new optimization problem denoted as $P_{1}$: given $S$, the new problem aims to find the optimal $\mu_i$ for one edge $i$ to maximize $r_S(\bm{\mu})$, while fixing the values of all others. The following lemma shows it is \#P-hard.
\begin{restatable}{lemma}{lemSharpP}\label{lem:SharpP}
Given $S$ and fixing $\mu_e$ for all $e \neq i$, finding the optimal $\mu_i \in c_i$ for one edge $i$ that maximizes $r_S(\bm{\mu})$ is \#P-hard.
\end{restatable}
\begin{proof}
We prove the hardness of this optimization problem via a reduction from the influence computation problem. We first consider a general graph $G_0$ with $n$ nodes and $m$ edges, where all influence probabilities on edges are set to $1/2$. Given $S_A$, computing the influence spread of $A$ in such a graph is \#P-hard. Notice that there is no seed set of $B$ in $G_0$. Now let us take one node $v$ in $G_0$ and denote its activation probability by $A$ as $h_A(G_0, S_A, v)$. Actually, computing $h_A(G_0, S_A, v)$ is also \#P-hard and we want to show that it can be reduced to our optimization problem in polynomial time. 
\begin{figure}[H]
    \centering
    \includegraphics[width=0.6\columnwidth]{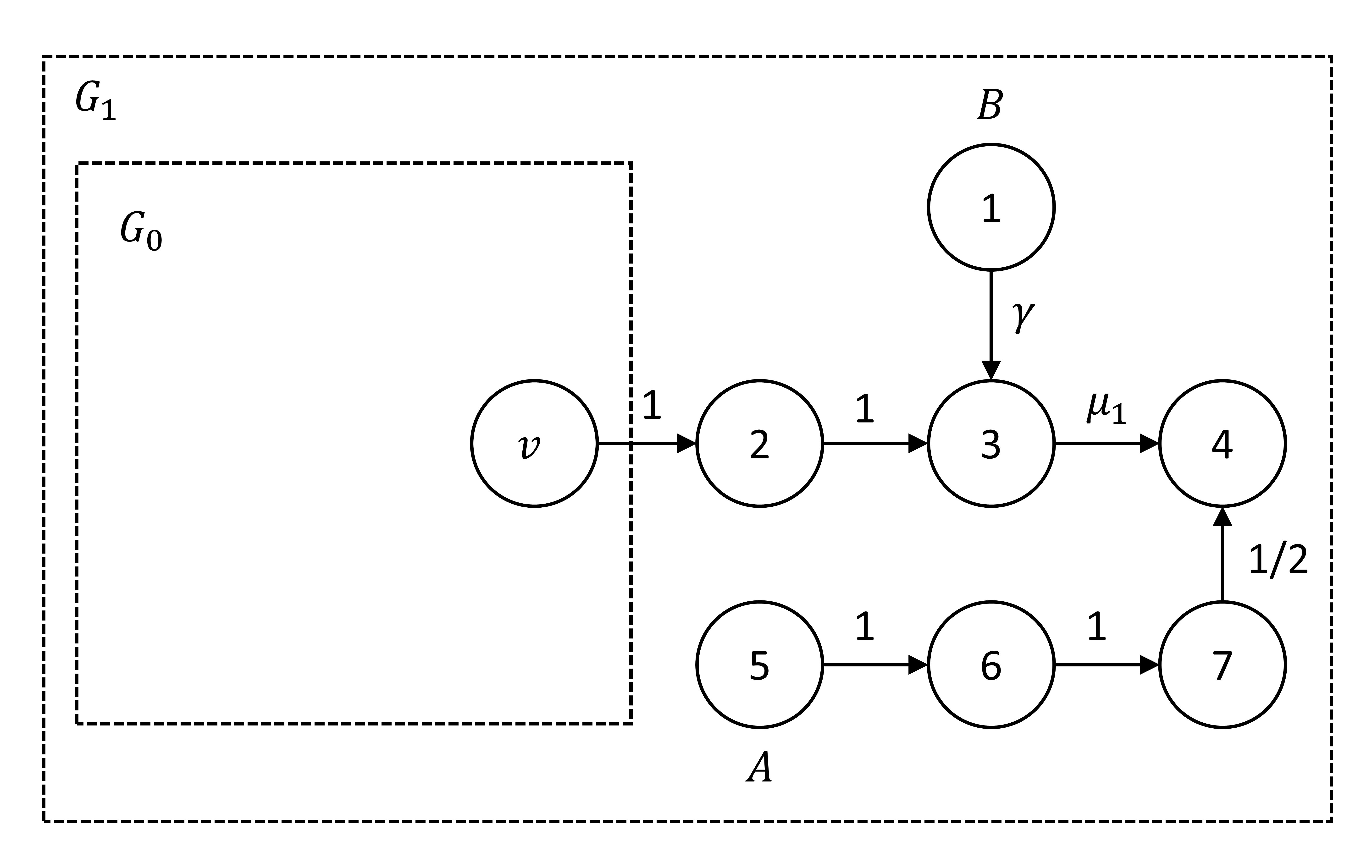}
    \caption{Construction of $G_1$ based on $G_0$}
    \label{fig:sharpP}
\end{figure}
We first construct a new graph $G_1$ based on $G_0$. For $G_1$, we keep $G_0$ and $S_A$ unchanged, then add several nodes and edges as shown in Fig.~\ref{fig:sharpP}. We add node $1$ to the seed set of $B$ and node $5$ to the seed set of $A$, so the joint action $S = \{S_A \cup\{5\}, S_B=\{1\}\}$. In this new graph $G_1$, we consider the optimization problem of finding the optimal $\mu_1$ (influence probability on edge $(3,4)$) within its range $c_1$ that maximizes $r_S(\bm{\mu})$. Notice that the influence probability $\gamma$ on edge $(1,3)$ is a constant and $\mu_1$ would only affect the activation probability of node $4$. 
We denote the activation probability by $A$ of node $4$ as $h_A(G_1, S, 4)$. In order to maximize $r_S(\bm{\mu})$, we only need to maximize $h_A(G_1, S, 4)$. It can be written as:
\begin{equation}\label{eq:inf_4}
    h_A(G_1, S, 4) = \frac{1}{2}\Big[(1-\gamma) \cdot h_A(G_1, S, v) - \gamma\Big]\cdot \mu_1 + \frac{1}{2}.
\end{equation}
It is easy to see $h_A(G_1, S, 4)$ has a linear relationship with $\mu_1$, so the optimal $\mu_1$ could only be either the lower or upper bound of its range $c_1$. Assuming we can solve the optimization problem of finding the optimal $\mu_1$, then we can determine the sign of $\mu_1$'s coefficient in Eq.\eqref{eq:inf_4}: if the optimal $\mu_1$ is the upper bound value in $c_1$, we have $(1-\gamma) \cdot h_A(G_1, S, v) - \gamma \geq 0$; otherwise, $(1-\gamma) \cdot h_A(G_1, S, v) - \gamma < 0$. It means we can answer the question that whether $h_A(G_1, S, v)$ is larger (or smaller) than $\frac{\gamma}{1-\gamma}$. Notice that $h_A(G_0, S_A, v) = h_A(G_1, S, v)$, so we can manually change the value of $\gamma$ to check whether $h_A(G_0, S_A, v)$ is larger (or smaller) than $x = \frac{\gamma}{1-\gamma}$ for any $x \in [0,1]$, Recall that all edge probabilities in $G_0$ are set to $1/2$, so the highest precision of $h_A(G_0, S_A, v)$ should be $2^{-m}$. Hence, we can use a binary search algorithm to find the exact value of $h_A(G_0, S_A, v)$ in at most $m$ times. It means computing the activation probability of $v$ in $G_0$ can be reduced to the optimization problem of finding the optimal $\mu_1$ in $G_1$, which completes the proof.
\end{proof}
We then show that $P_{1}$ is a special case of Eq.\eqref{eq:newopt}. The main idea is to relax the constraints $|S_A| \leq k, S = \{S_A, S_B\}$ in Eq.\eqref{eq:newopt} and show that it can find the optimal $\bm{\mu}$ for any given $S$. Consider a graph $G$ with $n$ nodes and a given seed set $S=\{S_A, S_B\}$. We construct a new graph $G'$ by manually add additional $n+1$ nodes pointing from each seed node in $S_A$. If we can solve the optimization problem Eq.\eqref{eq:newopt} in the new graph $G'$, since $S_A$ must be the optimal seed set of $A$ and the added nodes will not affect the prorogation in $G$, we will also find the optimal $\mu_i$'s in the original graph $G$ for the given $S$. Then, it is easy to see $P_{1}$ is a special case of Eq.\eqref{eq:newopt} since $P_{1}$ only find the optimal $\mu_i$ for one edge $i$. With Lemma~\ref{lem:SharpP}, we know Eq.\eqref{eq:newopt} is also \#P-hard.

\end{proof}

\subsection{Non-submodularity of $g(S)$}
In Section \ref{sec:OCIM-OFU}, we introduce $g(S) = \max_{\bm{\mu}} r_S(\bm{\mu})$, which is an upper bound function of $r_S(\bm{\mu})$ for each $S$. If $g(S)$ is submodular over $S$, we can use a greedy algorithm on $g(S)$ to find an approximate solution. However, the following example in Fig.~\ref{fig:non_submodular} shows that $g(S)$ is not submodular.
\begin{figure}[H]
    \centering
    \includegraphics[width=0.5\columnwidth]{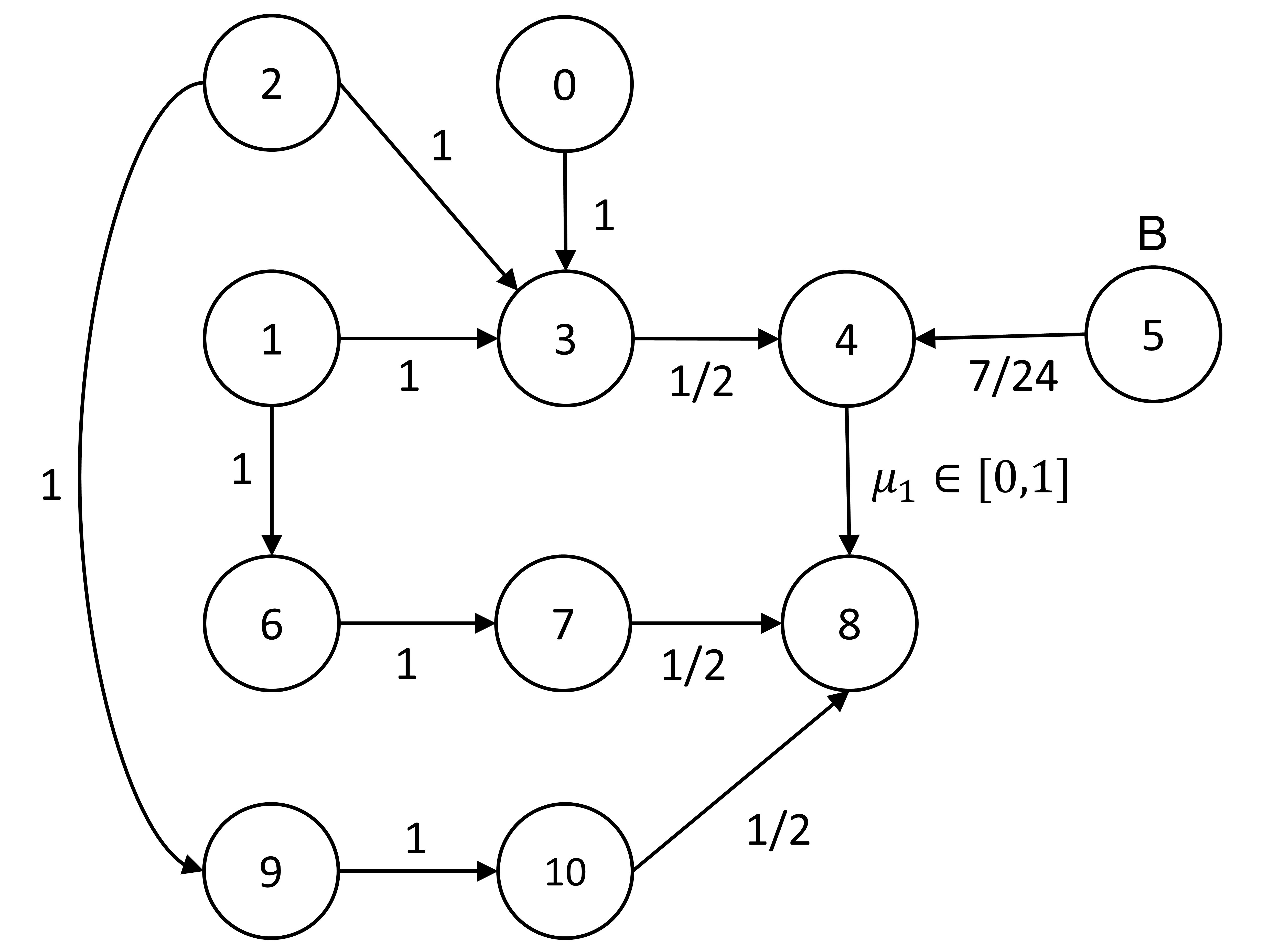}
    \caption{Example showing that $g(S)$ is not submodular}
    \label{fig:non_submodular}
\end{figure}
In Fig.~\ref{fig:non_submodular}, the numbers attached to edges are influence probabilities. Only the influence probability of edge $(4,8)$ is a variable and we denote it as $\mu_1$. We assume $\mu_1 \in [0,1]$ and $S_B = \{5\}$. Let us consider some choices of $S_A$. When $S_A$ is chosen as $\{0\}, \{0,1\}$ or $\{0,2\}$, the optimal $\mu_1$ that maximizes $r_S(\bm{\mu})$ is 1; when $S_A$ is chosen as $\{0,1,2\}$, the optimal $\mu_1$ that maximizes $r_S(\bm{\mu})$ is 0. Based on this observation, we can calculate $g(S)$ (assuming $S_B = \{5\}$):
\begin{align*}
    &g(S_A = \{0\}) = 2 + \frac{17}{24},\nonumber\\
    &g(S_A = \{0,1\}) = 5 + \frac{17}{24} \times \frac{4}{5},\nonumber\\
    &g(S_A = \{0,2\}) = 5 + \frac{17}{24} \times \frac{4}{5},\nonumber\\
    &g(S_A = \{0,1,2\}) = 8 +  \frac{17}{24} \times \frac{1}{2} + \frac{3}{4}.\nonumber\\
\end{align*}
Thus we have
\begin{equation}
    g(S_A = \{0,1\}) + g(S_A = \{0,2\}) < g(S_A = \{0\}) + g(S_A = \{0,1,2\}),
\end{equation}
which is contrary to submodularity.

\subsection{Bipartite Graph}\label{sec:compute_bipart}
We consider a weighted bipartite graph $G = (L,R,E)$ where each edge $(u, v)$ is associated with a probability $p(u, v)$. Given the competitor's seed set $S_B \subseteq L$, we need to choose $k$ nodes from $L$ as $S_A$ that maximizes the expected number of nodes activated by $A$ in $R$, where a node $v \in R$ can be activated by a node $u \in L$ with an independent probability of $p(u, v)$. As mentioned before, if $A$ and $B$ are attempting to activate a node in $L$ at the same time, the result will depend on the tie-breaking rule. If all edge probabilities are fixed, i.e., $\bm{\mu}$ is fixed, $r_{S}(\bm{\mu})$ is still submodular over $S_A$, so we can use a greedy algorithm as a $(1-1/e, 1)$-approximation oracle $\mathcal{O}_{\text{greedy}}$. 
Based on it, let us discuss the new offline optimization problem in Eq.\eqref{eq:newopt} under our two tie-breaking rules: (1) $A>B$: since $B$ will never influence nodes in $R$ earlier than $A$ in bipartite graphs, and $A$ will always win the competition, from $A$'s perspective, we can ignore $S_B$ to choose $S_A$. In this case, all edge probabilities should take the maximum values: for all $i\in E$, $\mu_i$ equals to the upper bound of $c_i$, and  we then use the oracle $\mathcal{O}_{\text{greedy}}$ to find $S_A$. 
(2) $B>A$: since $A$ will never influence nodes in $R$ earlier than $B$ in bipartite graphs, and $B$ will always win the competition, all out-edges of $S_B$, denoted as $E_{S_B}$, should take the minimum probabilities to maximize the influence spread of $A$. 
All the other edges in $E \backslash E_{S_B}$ should take the maximum probabilities. 
Formally, for all $i \in E_{S_B}$, $\mu_i$ equals to the lower bound of $c_i$; for all $i \in E \backslash E_{S_B}$, $\mu_i$ equals to the upper bound of $c_i$. We then use the oracle $\mathcal{O}_{\text{greedy}}$ to find $S_A$. To sum up, in bipartite graphs, $r_{S}(\bm{\mu})$ is optimized by pre-determining $\bm{\mu}$ based on the tie-breaking rule, and then using the greedy algorithm to get a $(1-1/e, 1)$-approximation solution. Since the time complexity of influence computation in the bipartite graph is $O(m)$, the time complexity of the offline algorithm is equal to that of the greedy algorithm, $O(kmn)$.

\subsection{General Graph}\label{sec:compute_general}GraphWe
The competitive propagation in the general graph is much more complicated, so it is hard to pre-determine all edge probabilities as in the bipartite graph case. 
However, we have a key observation:
\begin{restatable}{lemma}{lemUpperLower}\label{lem:UpperLower}
When fixing the seed set $S=\{S_A,S_B\}$, reward $r_{S}(\bm{\mu}) $ has a linear relationship with each $\mu_i$ (when other $\mu_j$'s with $j\ne i$
	are fixed).
This implies that the optimal solution for the optimization problem in Eq.\eqref{eq:newopt} must occur at the boundaries of the intervals $c_i$'s.
\end{restatable}
\begin{proof}
We can expand $r_{S}(\bm{\mu})$ based on the live-edge graph model (Chen et al., 2013a):
\begin{equation}\label{eq:live-edge}
r_{S}(\bm{\mu}) =  \sum_{L} |\Gamma_A(L,S)| \cdot \text{Pr}(L) =  \sum_{L} |\Gamma_A(L,S)| \prod_{e \in E(L)} \mu_e \prod_{e \notin E(L)} (1-\mu_e),
\end{equation}
where $L$ is one possible live-edge graph (each edge $e\in E$ is in $L$ with probability $\mu_e$ and not in $L$ with probability $1-\mu_e$, and this is independent from other edges), 
$\Gamma_A(L,S)$ is the set of nodes activated by $A$ from seed sets $S = \{S_A,S_B\}$ under live-edge graph $L$ and $E(L)$ is the set of edges that appear in live-edge graph $L$. 
Eq.\eqref{eq:live-edge} shows that $r_{S}(\bm{\mu})$ is linear with each $\mu_i$, so the optimal $\mu_i$ must take either the minimum or the maximum value in its range $c_i$. 
\end{proof}
Lemma~\ref{lem:UpperLower} implies that for any edge $e$ not reachable from $B$ seeds, it is safe to always take its upper bound value since it can only
	helps the propagation of $A$.
This further suggests that if we only have a small number (e.g. $\log m$) of edges reachable from $B$, then we can afford enumerating all the boundary value combinations of these edges.
For each such boundary setting $\bm{\mu}$,  we can use the IMM algorithm (Tang et al., 2014) to design a $(1-1/e-\epsilon, 1-n^{-l})$-approximation oracle $\mathcal{O}_{\text{IMM}}$ with time complexity $T_{\text{IMM}} = O((k+l)(m+n)\log n / \epsilon^{2})$. We discuss such graphs that satisfy the above condition in directed trees. Specifically, we consider the in-arborescence, where all edges point towards the root. For any node $u$ in the in-arborescence, there only exists one path from $u$ to the root; if $u$ is selected as the seed node of $B$, it could only propagate via this path. Hence, if the depth of the in-arborescence is in the order of $O(\log m)$, the number of edges reachable from $S_B$ would be $O(|S_B| \cdot \log m)$. In this case, we can use the IMM algorithm for $O(m^{|S_B|})$ combinations to obtain an approximate solution with time complexity $O(m^{|S_B|}\cdot T_{\text{IMM}})$. Examples of such in-arborescences with depth $O(\log m)$ could be the complete or full binary trees.

For general graphs, designing efficient approximation algorithms for the offline problem in Eq. (\ref{eq:newopt}) remains a challenging open problem, due to the joint optimization over $S$ and $\bm{\mu}$ and the complicated function form of $r_{S}(\bm{\mu})$.
Nevertheless, heuristic algorithms are still possible. In the experiment section, we employee the following heuristic with the $B>A$ tie-breaking rule: for all outgoing edges from $B$ seeds, we set their influence probabilities to their lower bound values, while for the rest, we set them to their upper bound values.
This setting guarantees that the first-level edges from the seeds are always set correctly, no matter how we select $A$ seeds.
They do not guarantee the correctness of second or higher level edge settings in the cascade, but the impact of those edges to influence spread decays significantly, so the above choice is reasonable as a heuristic.

\section{Proof of Theorem~\ref{thm:ETC}} \label{appendix:ETC}
\begin{proof}
\begin{algorithm}
 \caption{OCIM-ETC with offline oracle $\mathcal{O}$}
 \begin{algorithmic}[1]\label{alg:ETC}
 \STATE \textbf{Input}: $m$, $N$, $T$, Oracle $\mathcal{O}$.
 \STATE For each arm $i$, $T_i\leftarrow 0$. \{maintain the total number of times arm $i$ is played so far.\}
 \STATE For each arm $i$, $\hat{\mu}_i \leftarrow 0$. \{maintain the empirical mean of $X_i$.\}
 \STATE \textbf{Exploration phase}:
 \FOR{$t = 1,2,3,\dots, \lceil nN/k \rceil$}
    \STATE Take $k$ nodes that have not been chosen for $N$ times as $S_A$.
    \STATE Observe the feedback $X_i^{(t)}$ for each direct out-edge of $S_A$, $i \in \tau_{\text{direct}}$.
    \STATE For each arm $i\in \tau_{\text{direct}}$ update $T_i$ and $\hat{\mu}_i$: $T_i = T_i + 1, \hat{\mu}_i = \hat{\mu}_i + (X_i^{(t)}-\hat{\mu}_i) / T_i$.
 \ENDFOR 
    \STATE \textbf{Exploitation phase}:
 \FOR{$t =  \lceil nN/k \rceil+1,\dots, T$}
    \STATE Obtain context $S_B^{(t)}$.
    \STATE $S^{(t)} \leftarrow \mathcal{O}(S_B^{(t)}, \hat{\mu}_1, \hat{\mu}_2, \dots, \hat{\mu}_m)$.
    \STATE Play action $S^{(t)}$.
 \ENDFOR
 \end{algorithmic} 
\end{algorithm}
The OCIM-ETC algorithm is described in Alg.~\ref{alg:ETC}. We utilize the following well-known tail bound in our proof.
\begin{restatable}{lemma}{lemHoeffding}\label{lem:Hoeffding}
(Hoeffding’s Inequality) Let $X_1,\dots,X_n$ be independent and identically distributed random variables with common support $[0,1]$ and mean $\mu$. Let $Y = X_1 + \dots, + X_n$. Then for all $\delta \geq 0$,
\begin{equation*}
    \mathbb{P}\left\{|Y-n\mu|\geq \delta\right\} \leq 2e^{-2\delta^2/n}.
\end{equation*}
\end{restatable}
Let $\hat{\bm{\mu}} = (\hat{\mu}_1,\dots,\hat{\mu}_m)$ be the empirical mean of $\bm{\mu}$. Recall that oracle $\mathcal{O}$ takes $S_B^{(t)}$ and $\hat{\bm{\mu}}$ as inputs and outputs a solution $S^{(t)}$. Let us define event $\mathcal{F} = \left\{r_{S^{(t)}}(\hat{\bm{\mu}}) < \alpha \cdot \text{opt}^{(t)}(\hat{\bm{\mu}})\right\}$, which represents that oracle $\mathcal{O}$ fails to output an $\alpha$-approximate solution, and we know $\mathbb{P}(\mathcal{F}) < 1-\beta$.

With the same definitions in Appendix~\ref{appendix:TS:OFU}, we can decompose the regret as:
\begin{align}
    Reg_{\alpha,\beta}(T; \bm{\mu}) &\leq \lceil nN/k \rceil \cdot \Delta^{(T)}_{\max} + \sum_{t = T- \lceil nN/k\rceil+1}^{T} \Big[\alpha \beta \cdot \text{opt}^{(t)}(\bm{\mu}) - \mathbb{E}\big[r_{S^{(t)}}(\hat{\bm{\mu}})\big]\Big] \nonumber\\
    &\leq \lceil nN/k \rceil \cdot \Delta^{(T)}_{\max} + \sum_{t = T- \lceil nN/k\rceil+1}^{T} \Big[\alpha \beta \cdot \text{opt}^{(t)}(\bm{\mu}) - \beta\cdot\mathbb{E}\big[r_{S^{(t)}}(\hat{\bm{\mu}})\mid \neg \mathcal{F}\big]\Big] \nonumber\\
    &\leq \lceil nN/k \rceil \cdot \Delta^{(T)}_{\max} + \sum_{t = T- \lceil nN/k\rceil+1}^{T} \Big[\alpha\cdot \text{opt}^{(t)}(\bm{\mu}) - \mathbb{E}\big[r_{S^{(t)}}(\hat{\bm{\mu}})\mid \neg \mathcal{F}\big]\Big].\label{eq:reg_etc}
\end{align}
Next, let us rewrite the TPM condition in Theorem~\ref{thm:TPM}. For any $S$, $\bm{\mu}$ and $\bm{\mu}'$, we have
\begin{align}
|r_{S}(\bm{\mu}) - r_{S}(\bm{\mu}')| &\leq C \sum_{i\in[m]}p_i^{S}(\bm{\mu})|\mu_i-\mu_i'| \nonumber\\
&\leq C \sum_{i\in[m]}|\mu_i-\mu_i'| \nonumber\\
&\leq C m \cdot \max_{i\in [m]}|\mu_i-\mu_i'|, \label{eq:TPM_etc}
\end{align}
where $C$ is the maximum number of nodes that any one node can reach in graph $G$. Let $S^{*,t}_{\bm{\mu}}$ denote the optimal action for $\bm{\mu}$ in round $t$. Under $\neg \mathcal{F}$, we have
\begin{align}
    r_{S^{(t)}}(\hat{\bm{\mu}}) & \geq \alpha \cdot r_{S_{\hat{\bm{\mu}}}^{*,t}}(\hat{\bm{\mu}}) \nonumber\\
    & \geq \alpha \cdot r_{S_{\bm{\mu}}^{*,t}}(\hat{\bm{\mu}}) \nonumber\\
    & \geq \alpha \cdot r_{S_{\bm{\mu}}^{*,t}}(\bm{\mu}) - \alpha \cdot Cm \cdot \max_{i\in [m]}|\mu_i-\hat{\mu}_i| \nonumber\\
    % & \geq \alpha \cdot r_{S^{\mathcal{O}}}(\bm{\mu}) - \alpha \cdot Cm \cdot \max_{i\in [m]}|\mu_i-\hat{\mu}_i| \nonumber\\
    & \geq r_{S^{(t)}}(\bm{\mu}) + \Delta^{(t)}_{S^{(t)}} - \alpha \cdot Cm \cdot \max_{i\in [m]}|\mu_i-\hat{\mu}_i|, \label{eq:gap_etc}
\end{align}
where the third inequality is due to Eq.\eqref{eq:TPM_etc}. Combining Eq.\eqref{eq:TPM_etc} and Eq.\eqref{eq:gap_etc} together, we have
\begin{align}
    \Delta^{(t)}_{S^{(t)}} & \leq r_{S^{(t)}}(\hat{\bm{\mu}}) - r_{S^{(t)}}(\bm{\mu}) + \alpha \cdot Cm \cdot \max_{i\in [m]}|\mu_i-\hat{\mu}_i| \nonumber\\
    & \leq (1+\alpha) \cdot Cm \cdot \max_{i\in [m]}|\mu_i-\hat{\mu}_i|. \label{eq:delta_etc}
\end{align}
Let us define $\delta_0 := \frac{\Delta^{(T)}_{\min}}{2Cm}$. If $\max_{i\in [m]}|\mu_i-\hat{\mu}_i| <\delta_0$, then we know $S^{(t)}$ is at least an $\alpha$-approximate solution, such that $\Delta^{(t)}_{S^{(t)}} = 0$. Then the regret in Eq.\eqref{eq:reg_etc} can be written as
\begin{align}
    Reg_{\alpha,\beta}(T; \bm{\mu}) &\leq \lceil nN/k \rceil \cdot \Delta^{(T)}_{\max} + \Big(T -  \lceil nN/k \rceil\Big) \cdot 2m\exp(-2N\delta_0^2)\cdot \Delta^{(T)}_{\max}\nonumber\\
    &\leq \Big(\lceil nN/k \rceil + T \cdot 2m\exp(-2N\delta_0^2)\Big) \cdot \Delta^{(T)}_{\max}. \label{eq:reg2_etc}
\end{align}
The first inequality is obtained by applying the Hoeffding’s Inequality (Lemma~\ref{lem:Hoeffding}) and union bound to the event $\max_{i\in [m]}|\mu_i-\hat{\mu}_i| \geq \delta_0$. Now we need to choose an optimal $N$ that minimizes Eq.\eqref{eq:reg2_etc}. By taking $N = \max\left\{1, \frac{1}{2\delta_0^2} \ln \frac{4kmT\delta_0^2}{C}\right\} = \max\left\{1, \frac{2C^2 m^2}{(\Delta^{(T)}_{\min})^2}\ln(\frac{kT(\Delta^{(T)}_{\min})^2}{C^3 m})\right\}$, when $\Delta^{(T)}_{\min} > 0$, we can get the distribution-dependent bound
\begin{equation}
    \text{Reg}_{\alpha,\beta}(T; \bm{\mu}) \leq \frac{2C^2 m^2 n \Delta^{(T)}_{\max}}{k (\Delta^{(T)}_{\min})^2}\left(\max\left\{ \ln\left(\frac{kT(\Delta^{(T)}_{\min})^2}{C^2 mn}\right), 0\right\} + 1\right) + \frac{n}{k} \Delta^{(T)}_{\max},
\end{equation}

Next, let us prove the distribution-independent bound. Let $\mathcal{N}$ denote the event that $|\hat{\mu}_i - \mu_i| \leq \sqrt{\frac{2 \ln T}{N}}$ for all $i \in [m]$. By the Hoeffding’s Inequality and union bound, we have 
\begin{equation}\label{eq:prob_etc}
    \mathbb{P}\{\neg \mathcal{N}\} \leq m \cdot \frac{2}{T^4} \leq \frac{2}{T^3}.
\end{equation}
When $\mathcal{N}$ holds, with Eq.\eqref{eq:delta_etc}, we have
\begin{equation}
    \Delta^{(t)}_{S^{(t)}} \leq 2Cm \cdot \sqrt{\frac{2 \ln T}{N}},
\end{equation}
and the regret in Eq.\eqref{eq:reg_etc} can be written as
\begin{align}
    Reg_{\alpha,\beta}(T; \bm{\mu}) &\leq \lceil nN/k \rceil \cdot n + \sum_{t = T- \lceil nN/k\rceil+1}^{T} \Delta^{(t)}_{S^{(t)}}\nonumber\\
    &\leq \lceil nN/k \rceil \cdot n + O\left(T\cdot Cm \cdot \sqrt{\frac{\ln T}{N}}\right).
\end{align}
We can choose $N$ so as to (approximately) minimize the regret. For $N = (Cmk)^{\frac{2}{3}} n^{-\frac{4}{3}} T^{\frac{2}{3}} (\ln T)^{\frac{1}{3}}$, we obtain:
\begin{equation}
     Reg_{\alpha,\beta}(T; \bm{\mu}) \leq O((Cmn)^{\frac{2}{3}} k^{-\frac{1}{3}} T^{\frac{2}{3}} (\ln T )^{\frac{1}{3}}).
\end{equation}
To complete the proof, we need to consider both $\mathcal{N}$ and $\neg \mathcal{N}$. As shown in Eq.\eqref{eq:prob_etc}, the probability that $\neg \mathcal{N}$ occurs is very small, and we have:
\begin{align}
    Reg_{\alpha,\beta}(T; \bm{\mu}) &= \mathbb{E}\left[Reg_{\alpha,\beta}(T; \bm{\mu})\mid \mathcal{N} \right] \cdot \mathbb{P}\{\mathcal{N}\} + \mathbb{E}\left[Reg_{\alpha,\beta}(T; \bm{\mu})\mid \neg \mathcal{N} \right] \cdot \mathbb{P}\{\neg \mathcal{N}\}\nonumber\\
    &\leq \mathbb{E}\left[Reg_{\alpha,\beta}(T; \bm{\mu})\mid \mathcal{N} \right] + T\cdot n \cdot O(T^{-3})\nonumber\\
    &\leq O((Cmn)^{\frac{2}{3}} k^{-\frac{1}{3}} T^{\frac{2}{3}} (\ln T )^{\frac{1}{3}}).
\end{align}
\end{proof}

\section{Proof of Theorem~\ref{thm:TPM_prob}}
% In the setting where competitor $B$ has probabilistic seed distribution, we denote the expected reward of follower $A$ as $r(S_A, D_B, \bm{\mu})$, where $S_A$ is the seed set of $A$, $D_B$ is the seed distribution of $B$, i.e., $S_B \sim D_B$. We use $p_i(S_A, D_B, \bm{\mu})$ to denote the probability that $S_A$ or $S_B$ will trigger arm $i$ when the seed set of $A$ is $S_A$, the seed set of $B$, $S_B$, is sampled from $D_B$, and the expectation vector is $\bm{\mu}$. The modified TPM condition is given below.

% \begin{restatable}{condition}{condOverOneNormProb}(Modified TPM bounded smoothness). \label{cond:TPM_prob}
% We say that an OCIM problem instance satisfies modified TPM bounded smoothness, if there exists $C\in \mathbb{R}^+$ such that, for any two expectation vectors $\bm{\mu}$ and $\bm{\mu}'$, and any seed set $S_A$ and seed distribution $D_B$, we have $|r(S_A, D_B, \bm{\mu}) - r(S_A, D_B, \bm{\mu}')| \leq C \sum_{i\in[m]}p_i(S_A, D_B, \bm{\mu})|\mu_i-\mu_i'|$.
% \end{restatable}

% With the similar analysis of Theorem \ref{thm:TPM}, we can show the following TPM condition when the competitor has probabilistic seed distribution.
% \begin{restatable}{theorem}{thmOverTPMProb}\label{thm:TPM_prob}
% Under both dominance and proportional tie-breaking rules,
% OCIM instances satisfy the modified TPM bounded smoothness condition with coefficient $C=2\tilde{C}$, where $\tilde{C}$ is the maximum number of nodes that any one node can reach in graph $G$.
% \end{restatable}
\begin{proof}
As mentioned in Section~\ref{sec:extension}, we need to introduce a virtual $B$ seed node $u_B$, which connects to each existing node $u$ with an unknown edge probability $p(u_B, u)$ equal to the probability of $u$ being selected as a $B$ seed. By adding these virtual nodes and edges, we get a new graph $G'$ with $2n$ nodes and $m+n$ edges. Since $S_B$ is fixed under $G'$, we can follow the same steps in the proof of Theorem~\ref{thm:TPM} to show the TPM condition holds under $G'$. Note that the maximum number of nodes that any one node can reach in $G'$ is twice as that in the original graph $G$, so the new bounded smoothness coefficient $C=2\tilde{C}$.
\end{proof}

\section{Additional Experiments}

\subsection{Experiments for $A>B$ Tie-breaking Rule}
When we consider $A>B$ in bipartite graphs, we can trivially ignore $S_B$ to choose $S_A$ since the influence spread ends in one diffusion round, and OCIM becomes the online influence maximization problem without competition. We show such results in Figure~\ref{fig:Yahoo, A>B}. Note that the distribution of $B$ no longer affects the performance of $A$ when $A>B$ and we only use one figure for the IM and RD distribution.
For general graphs, we use the same DM dataset and parameter settings described in Sec.~\ref{sec:experiment}, and the only difference is that $A$ now dominates $B$.
% However, we use a different heuristic, where we set influence probabilities to their upper bound values to optimistically maximize A's influence.
We show the results in Figure~\ref{fig:GG, A>B}.
Overall, the results and the analysis for $A>B$ are consistent with $B>A$.

% \begin{figure}[ht]
% 	\begin{subfigure}[b]{0.49\textwidth}
% 		\centering
% 		\includegraphics[width=\textwidth]{figure/2020-06-10_679_3374_1500_real_indegree_uniform_IMM_False.png}
% 		\caption{DM, RD}
% 	\end{subfigure}
% 	\begin{subfigure}[b]{0.49\textwidth}
% 		\centering
% 		\includegraphics[width=\textwidth]{figure/2020-06-11_679_3374_1000_real_indegree_max_IMM_False_1.png} 
% 		\caption{DM, IM}
% 	\end{subfigure}
% 	\caption{Regrets of different algorithms for general graphs, when $A > B$.
% 	}\label{fig:GG_A>B}
% \end{figure}
\subsection{Experiments for OCIM-ETC}
We show the frequentist/Bayesian regret results for the OCIM-ETC algorithm in Figure~\ref{fig:ETC_Yahoo}, Figure~\ref{fig:ETC_DM} and Figure~\ref{fig:ETC_more_rounds}. In Figure~\ref{fig:ETC_Yahoo}, we set exploration phase to be $250$ rounds and the experiments show that we suffer linear regrets in both the exploration and the exploitation phase, meaning that the unknown parameters are under-explored. Thus we reset exploration to be $1500$ and Figure~\ref{fig:ETC_more_rounds} shows that OCIM-ETC now has constant regret in the exploitation phase.
% The dataset and parameter settings are the same, and we set the exploration phase to be $10,000$ and $20,000$ for Yahoo-Ad and DM, respectively.
For DM dataset, since the node number and the edge number are less than Yahoo-Ad, we can see constant regrets after $1000$ rounds of exploration in Figure~\ref{fig:ETC_DM}.
Compared with OCIM-OFU/OCIM-TS, OCIM-ETC requires more rounds to learn the unknown influence probabilities and has larger regrets than OCIM-OFU/OCIM-TS, but with sufficient exploration (which is much less than the theoretical requirements $N = (\tilde{C}m)^{\frac{2}{3}} (nk)^{-\frac{1}{3}} T^{\frac{2}{3}} (\ln T)^{\frac{1}{3}}$ in Theorem~\ref{thm:ETC}) OCIM-ETC can yield constant regrets during the exploitation phase in our experiments.

\begin{figure}[ht]
	\centering
	\begin{subfigure}[b]{0.25\textwidth}
		\centering
		\includegraphics[width=\textwidth]{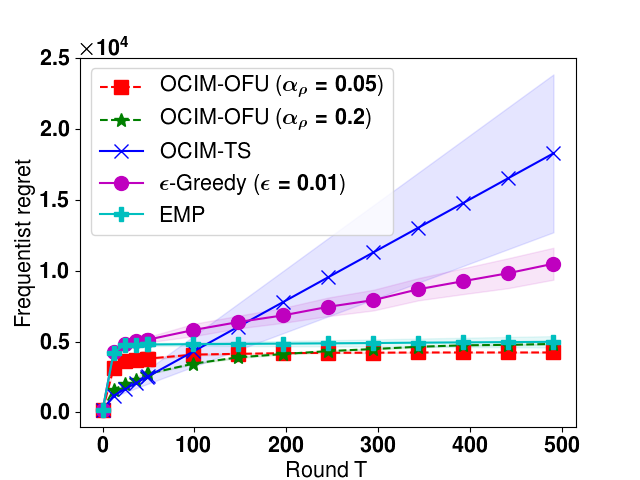}
		\caption{Yahoo-Ad, Frequentist}
		\label{fig:Yahoo, A>B, Freq}
	\end{subfigure}
	\begin{subfigure}[b]{0.25\textwidth}
		\centering
		\includegraphics[width=\textwidth]{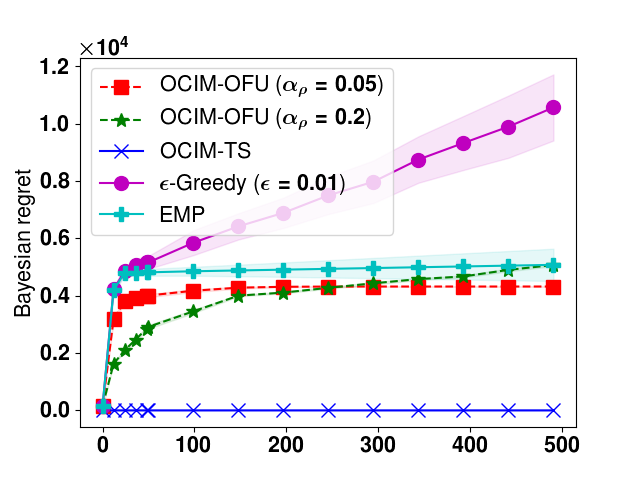}
		\caption{Yahoo-Ad, Bayesian}
		\label{fig:Yahoo, A>B, Bayes}
	\end{subfigure}
	\caption{Frequentist/Bayesian regrets of different algorithms for the Yahoo-Ad graph when $A>B$.
	}\label{fig:Yahoo, A>B}
\end{figure}

\begin{figure}[ht]
	\centering
	\begin{subfigure}[b]{0.23\textwidth}
		\centering
		\includegraphics[width=\textwidth]{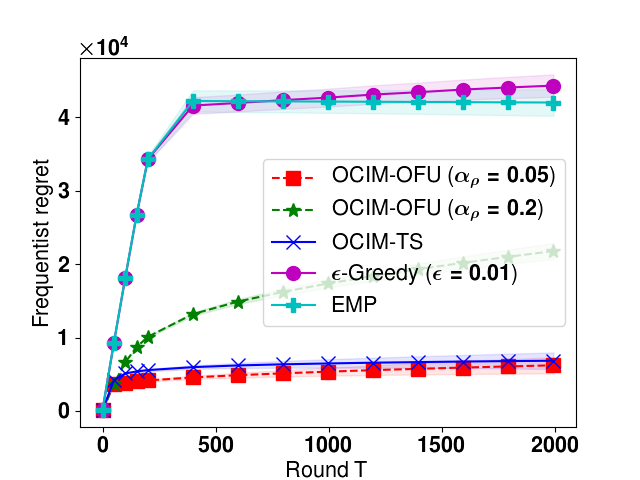}
		\caption{DM, RD, Frequentist}
		\label{fig:DM, RD, A>B}
	\end{subfigure}
	\begin{subfigure}[b]{0.23\textwidth}
		\centering
		\includegraphics[width=\textwidth]{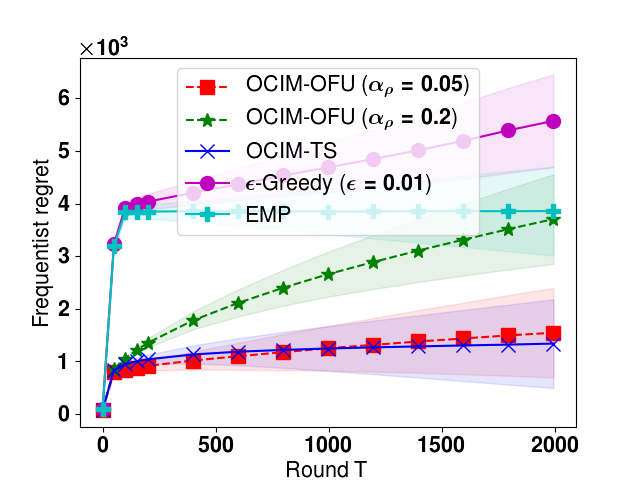}
		\caption{DM, IM, Frequentist}
		\label{fig:DM, IM, A>B}
	\end{subfigure}
	\begin{subfigure}[b]{0.23\textwidth}
		\centering
		\includegraphics[width=\textwidth]{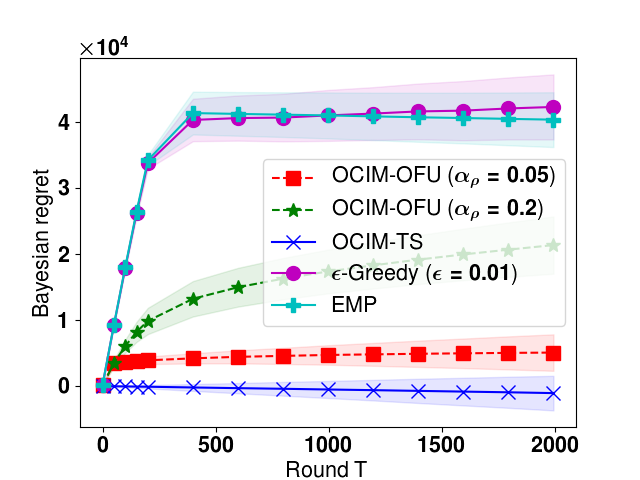}
		\caption{DM, RD, Bayesian}
		\label{fig:DM, RD, A>B, Bayes}
	\end{subfigure}
	\begin{subfigure}[b]{0.23\textwidth}
		\centering
		\includegraphics[width=\textwidth]{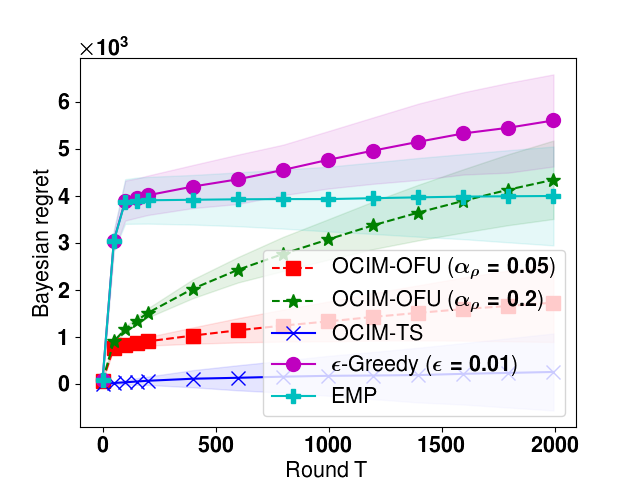}
		\caption{DM, IM, Bayesian}
		\label{fig:DM, IM, A>B, Bayes}
	\end{subfigure}
	\caption{Frequentist/Bayesian regrets of different algorithms for the general graph DM when $A>B$.
	}\label{fig:GG, A>B}
\end{figure}

\begin{figure}[ht]
	\centering
	\begin{subfigure}[b]{0.235\textwidth}
		\centering
		\includegraphics[width=\textwidth]{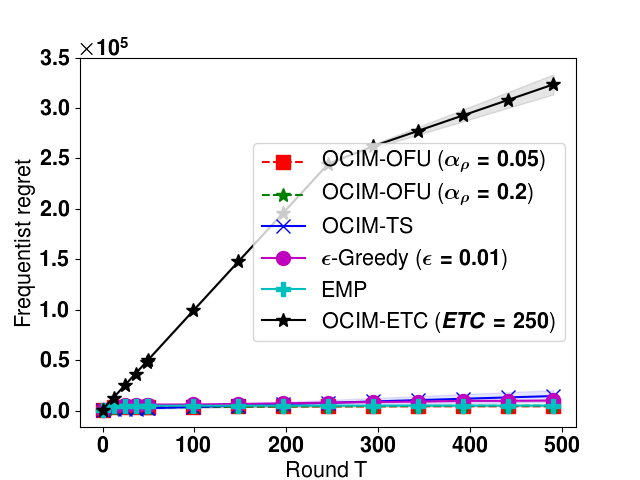}
		\caption{RD, Frequentist}
	\end{subfigure}
	\begin{subfigure}[b]{0.235\textwidth}
		\centering
		\includegraphics[width=\textwidth]{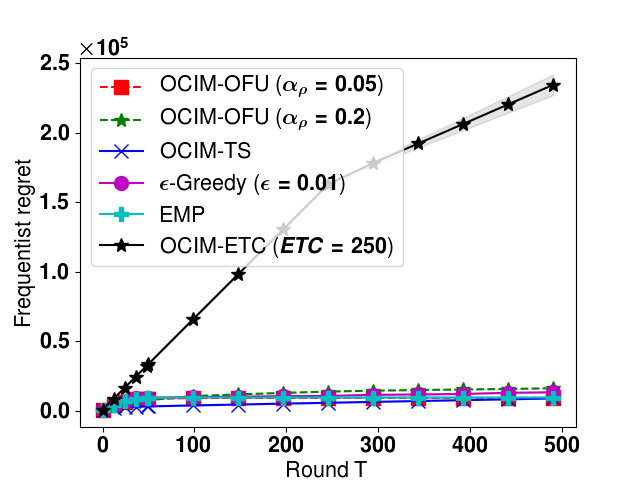}
		\caption{IM, Frequentist}
	\end{subfigure}
	\begin{subfigure}[b]{0.235\textwidth}
		\centering
		\includegraphics[width=\textwidth]{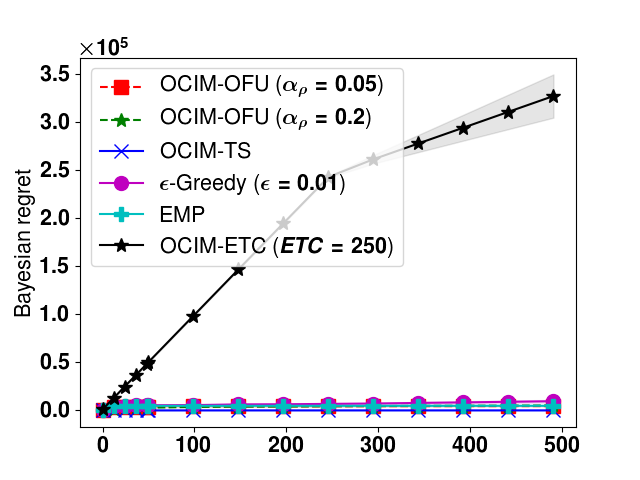}
		\caption{RD, Bayesian}
	\end{subfigure}
	\begin{subfigure}[b]{0.235\textwidth}
		\centering
		\includegraphics[width=\textwidth]{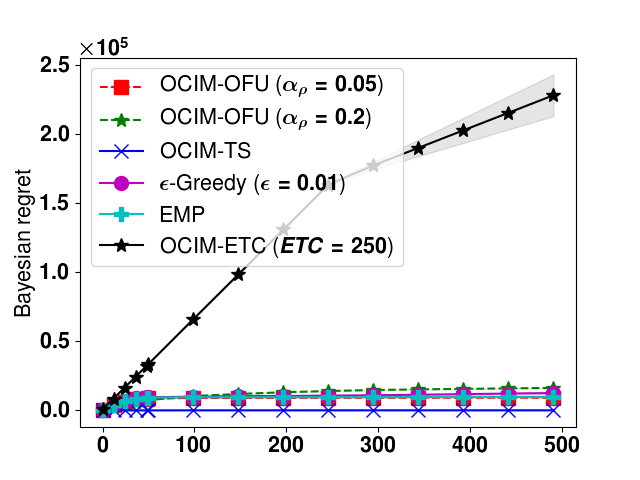}
		\caption{IM, Bayesian}
	\end{subfigure}
	\caption{Frequentist/Bayesian regrets of OCIM-ETC for the Yahoo-Ad graph.
	}\label{fig:ETC_Yahoo}
\end{figure}

\begin{figure}[ht]
	\centering
	\begin{subfigure}[b]{0.235\textwidth}
		\centering
		\includegraphics[width=\textwidth]{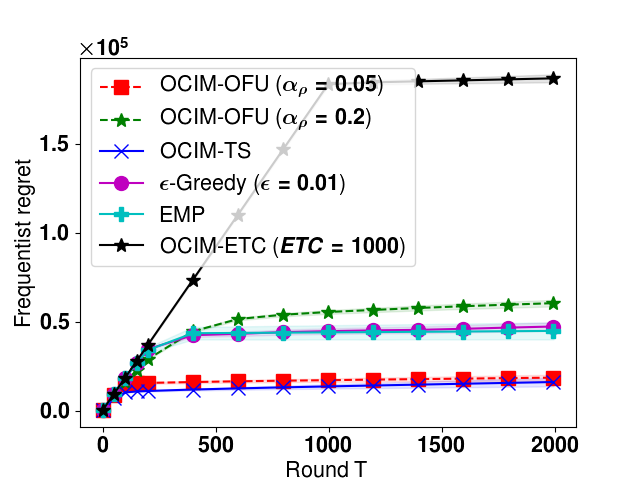}
		\caption{RD, Frequentist}
	\end{subfigure}
	\begin{subfigure}[b]{0.235\textwidth}
		\centering
		\includegraphics[width=\textwidth]{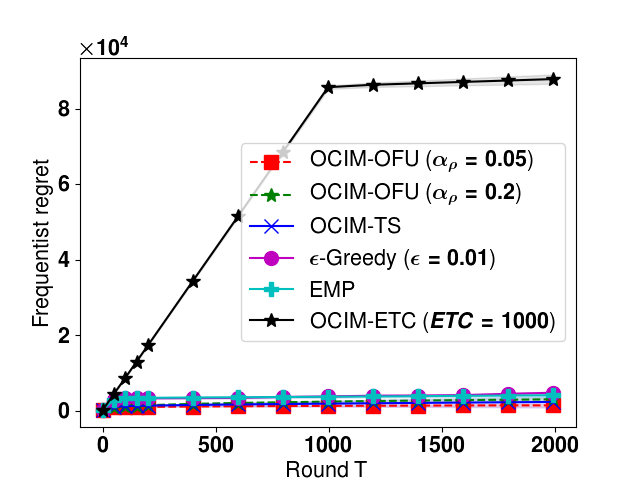}
		\caption{IM, Frequentist}
	\end{subfigure}
	\begin{subfigure}[b]{0.235\textwidth}
		\centering
		\includegraphics[width=\textwidth]{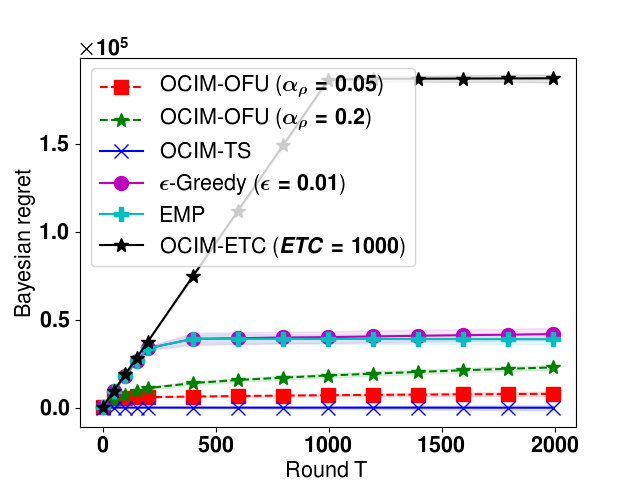}
		\caption{RD, Bayesian}
	\end{subfigure}
	\begin{subfigure}[b]{0.235\textwidth}
		\centering
		\includegraphics[width=\textwidth]{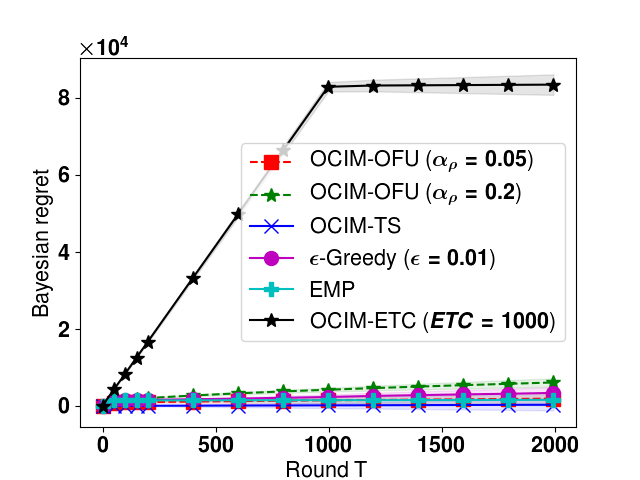}
		\caption{IM, Bayesian}
	\end{subfigure}
	\caption{Frequentist/Bayesian regrets of OCIM-ETC for the DM graph.
	}\label{fig:ETC_DM}
\end{figure}

\begin{figure}[ht]
	\centering
	\begin{subfigure}[b]{0.235\textwidth}
		\centering
		\includegraphics[width=\textwidth]{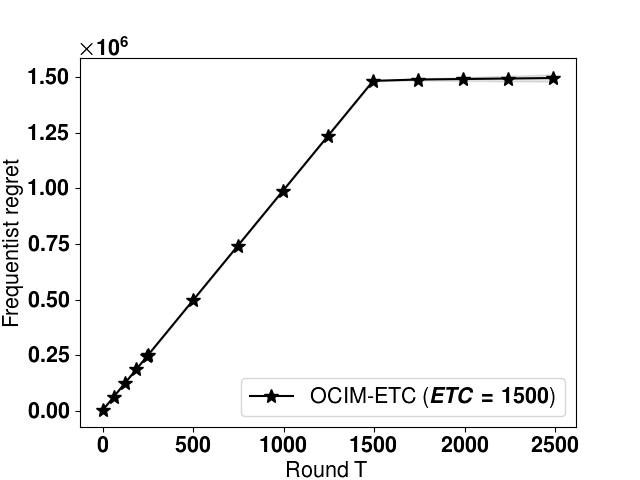}
		\caption{RD, Frequentist}
	\end{subfigure}
	\begin{subfigure}[b]{0.235\textwidth}
		\centering
		\includegraphics[width=\textwidth]{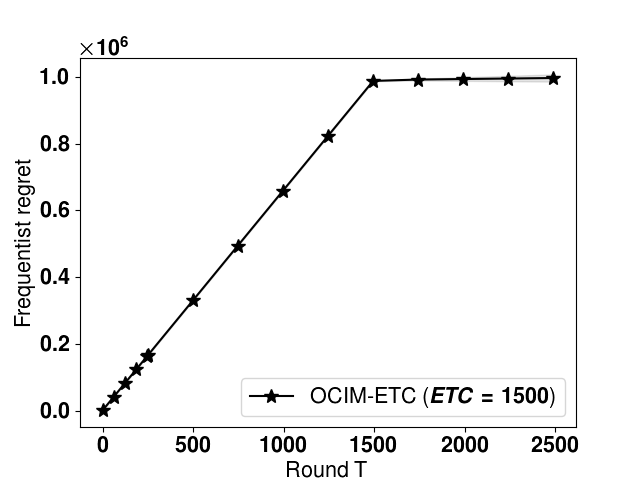}
		\caption{IM, Frequentist}
	\end{subfigure}
	\begin{subfigure}[b]{0.235\textwidth}
		\centering
		\includegraphics[width=\textwidth]{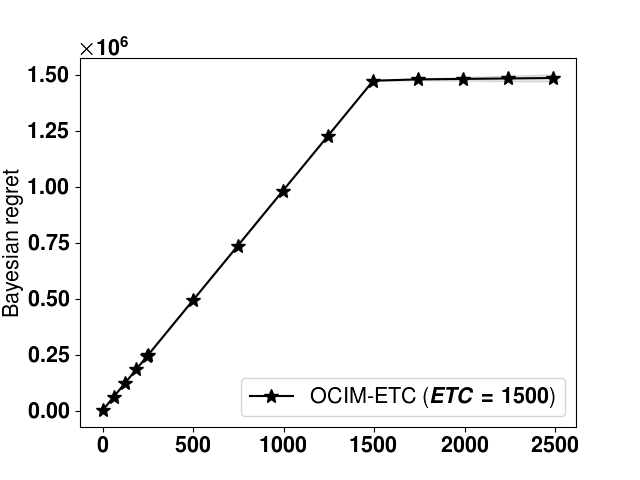}
		\caption{RD, Bayesian}
	\end{subfigure}
	\begin{subfigure}[b]{0.235\textwidth}
		\centering
		\includegraphics[width=\textwidth]{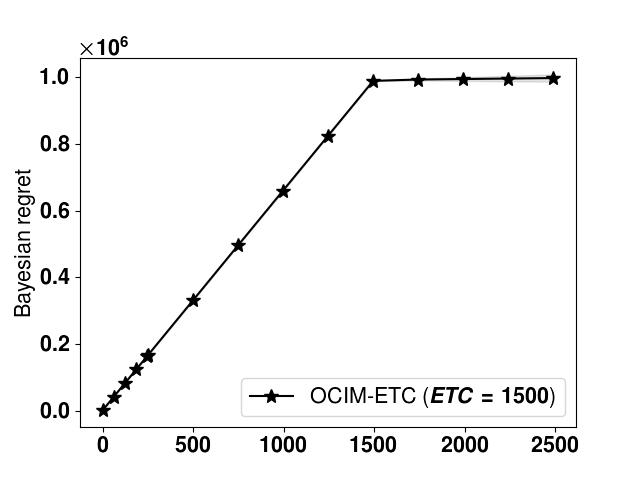}
		\caption{IM, Bayesian}
	\end{subfigure}
	\caption{Frequentist/Bayesian regrets of OCIM-ETC for the Yahoo-Ad graph with 1500 rounds of exploration.
	}\label{fig:ETC_more_rounds}
\end{figure}

\subsection{Experiments for Probabilistic Seed Distribution}
For the settings where the competitor has unknown fixed seed distribution, we first run the non-competitive influence maximization algorithm for $S_B$. and get the best $5$ seeds on Yahoo-Ad and the best $10$ seeds on DM, respectively, We then consider the seed distribution of $S_B$ as choosing each node from the best seeds with probability $0.5$, i.e., the probability that choosing all best $5$ seeds on Yahoo-Ad is $0.5^{5}$ and the probability that choosing all best $10$ seeds on DM is $0.5^{10}$. This seed distribution of $S_B$ is unknown to our algorithms. In our experiments, we set $|S_A|=5$ for Yahoo-Ad and $|S_A|=10$ for DM, and assume $B>A$. Figure~\ref{fig:prob_seed} shows that OCIM-OFU is still superior to EMP and $\epsilon$-Greedy for this setting with more complex competitor actions. We omit the results of OCIM-TS here as it requires the prior knowledge of the competitor's seed distribution. However, as long as the given prior does not differ much from the true prior, OCIM-TS will also achieve good regret results.
\begin{figure}[ht]
	\centering
	\begin{subfigure}[b]{0.35\textwidth}
		\centering
		\includegraphics[width=\textwidth]{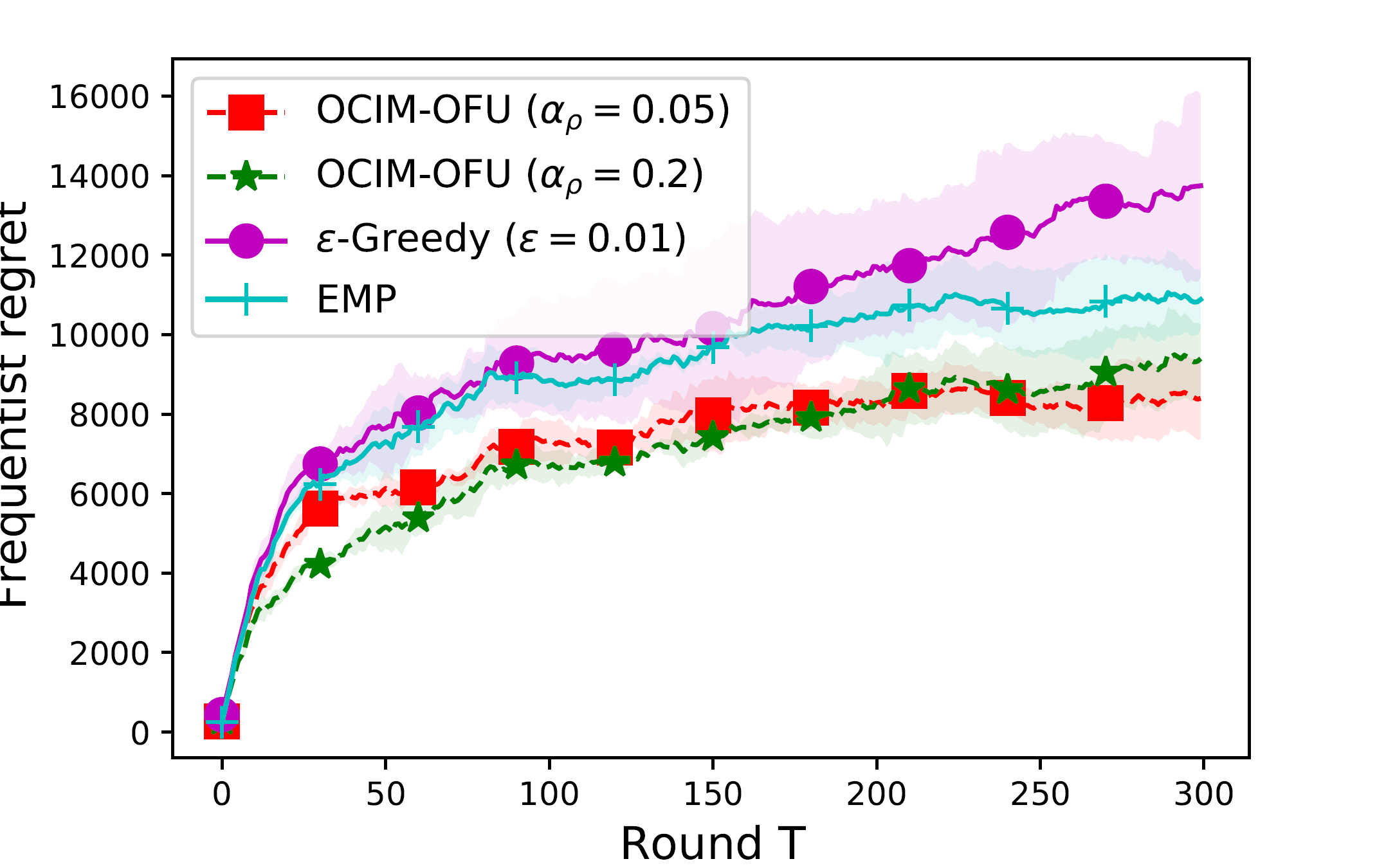}
		\caption{Yahoo-Ad}
	\end{subfigure}
	\begin{subfigure}[b]{0.35\textwidth}
		\centering
		\includegraphics[width=\textwidth]{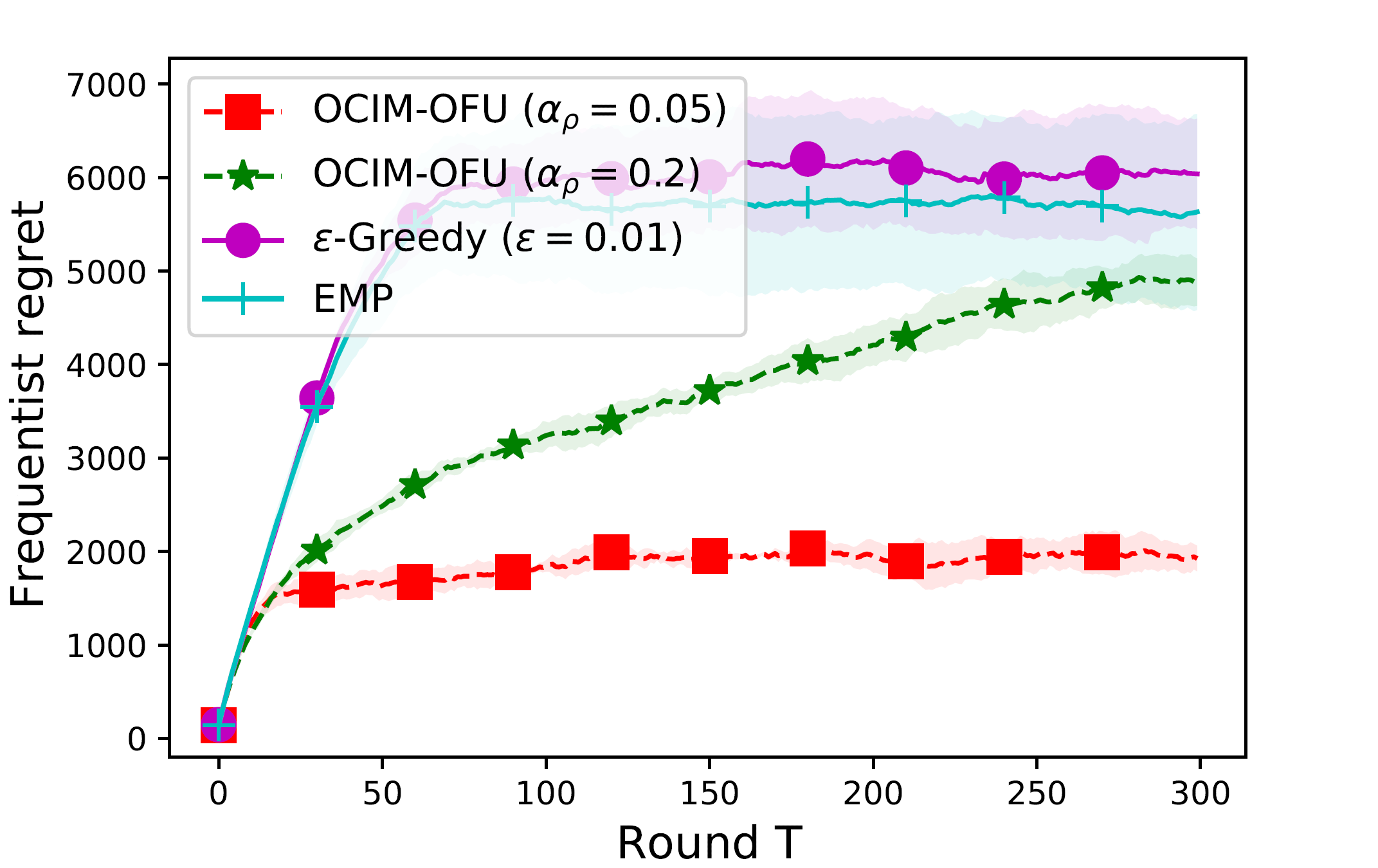}
		\caption{DM}
	\end{subfigure}
	\caption{Frequentist regrets for Yahoo-Ad and DM with unknown fixed competitor's seed distribution.}
	\label{fig:prob_seed}
\end{figure}

\clearpage
\section{Contextual combinatorial multi-armed bandit framework C$^2$MAB-T}
\subsection{General Framework}
We propose a general framework of contextual combinatorial multi-armed bandit with probabilistically triggered arms (C$^2$MAB-T), which is a contextual extension of CMAB-T in~\citep{wang2017improving}.

C$^2$MAB-T is a learning game between a learning player and an environment. The environment consists of $m$ random variables $X_1, \cdots, X_m$ called \emph{base arms} following a joint distribution $D$ over $[0,1]^m$. Distribution $D$ is chosen by the environment from a class of distributions $\mathcal{D}$ before the game starts. The player knows $\mathcal{D}$ but not the actual distribution $D$ in advance. Different from that in CMAB-T, the environment in C$^2$MAB-T also provides \emph{contexts} for the learning agent, which will be discussed in detail later.

The learning process runs in discrete rounds. In round $t$, the environment first provides a context, $\bm{\mathcal{S}}^{(t)} \subseteq \bm{\mathcal{S}}$, to the player, where $\bm{\mathcal{S}}$ is the full action space and $\bm{\mathcal{S}}^{(t)}$ is a subset of it, representing the current action space in round $t$. The player then chooses an action $S^{(t)} \in \bm{\mathcal{S}}^{(t)}$ based on the feedback history from previous rounds. The environment also draws an independent sample $X^{(t)} = (X_1^{(t)}, \cdots, X_m^{(t)})$ from the joint distribution $D$. When action $S^{(t)}$ is played on the environment outcome $X^{(t)}$, a random subset of arms $\tau_t\in[m]$ are triggered, and the outcomes of $X_i^{(t)}$ for all $i\in\tau_t$ are observed as the feedback to the player. 
$\tau_t$ may have additional randomness beyond the randomness of $X^{(t)}$. Let $D_{\text{trig}}(S,X)$ denote a distribution of the triggered subset of $[m]$ for a given action $S$ and an environment outcome $X$. We assume $\tau_t$ is drawn independently from $D_{\text{trig}}(S^{(t)}, X^{(t)})$. The player obtains a reward $R(S^{(t)}, X^{(t)}, \tau_t)$ fully determined by $S^{(t)}, X^{(t)}$ and $\tau_t$. A learning algorithm aims at selecting actions $S^{(t)}$'s over time based on the past feedback to accumulate as much reward as possible.

For each arm $i$, let $\mu_i = \mathbb{E}_{X\sim D}[X_i]$. Let $\bm{\mu}= (\mu_1, \cdots, \mu_m)$ denote the expectation vector of arms. We assume that the expected reward $\mathbb{E}[R(S, X, \tau)]$, where the expectation is taken over $X \sim D$ and $\tau \sim D_{\text{trig}}(S,X)$, is a function of
action $S$ and the expectation vector $\bm{\mu}$ of the arms. Thus, we denote $r_{S}(\bm{\mu}) := \mathbb{E}[R(S, X, \tau)]$. We assume the outcomes of arms do not depend on whether they are triggered, i.e., $\mathbb{E}_{X \sim D, \tau \sim D_{\text{trig}}(S,X)}[X_i \mid i \in \tau] = \mathbb{E}_{X \sim D}[X_i]$.

The performance of a learning algorithm $\mathcal{A}$ is measured by its expected regret, which is the difference in expected cumulative reward between always playing the best action and playing actions selected by algorithm $\mathcal{A}$. Let $\text{opt}^{(t)}(\bm{\mu}) = \sup_{S^{(t)}\in \bm{\mathcal{S}}^{(t)}}(\bm{\mu})$ denote the expected reward of the optimal action in round $t$. We assume that there exists an offline oracle $\mathcal{O}$, which takes context $\bm{\mathcal{S}}^{(t)}$ and $\bm{\mu}$ as inputs and outputs an action $S^{{\mathcal{O}},(t)}$ such that $\text{Pr}\{r_{S^{{\mathcal{O}},(t)}}(\bm{\mu}) \geq \alpha \cdot \text{opt}^{(t)}(\bm{\mu})) \} \geq \beta$, where $\alpha$ is the approximation ratio and $\beta$ is the success probability. Instead of comparing 
with the exact optimal reward, we take the $\alpha \beta$ fraction of it and use the following $(\alpha,\beta)$-approximation {\em frequentist regret} for $T$ rounds:
\begin{equation}\textstyle
    Reg^{\mathcal{A}}_{\alpha,\beta}(T;\bm{\mu}) = \sum_{t=1}^T  \alpha \cdot \beta \cdot \text{opt}^{(t)}(\bm{\mu}) - \sum_{t=1}^T r_{S^{\mathcal{A},(t)}}(\bm{\mu}),
\end{equation}
where $S^{\mathcal{A},(t)}$ is the action chosen by algorithm $\mathcal{A}$ in round $t$.

Another way to measure the performance of the algorithm $\mathcal{A}$ is using {\em Bayesian regret}. Denote the prior distribution of $\bm{\mu}$ as $\mathcal{Q}$. When the prior $\mathcal{Q}$ is given, the corresponding Bayesian regret is defined as:
\begin{equation}\textstyle
    BayesReg^{\mathcal{A}}_{\alpha,\beta}(T) = \mathbb{E}_{\bm{\mu}\sim\mathcal{Q}}  Reg^{\mathcal{A}}_{\alpha,\beta}(T;\bm{\mu}).
\end{equation}
Note that the contextual combinatorial bandit problem is also studied in~\citep{chen2018contextual,qin2014contextual}. They consider the context features of all bases arms, which can affect their expected outcomes in each round, and assume the action space of super arms is a subset of $[m]$. However, we do not bond the context with base arms and consider the feasible set of super arms,  $\bm{\mathcal{S}}^{(t)}$, as the context, which is more flexible than a subset of $[m]$. Besides, we are the first to consider probabilistically triggered arms in the contextual combinatorial bandit problem.

% \wei{Need to add a paragraph comparing with the SDM'14 paper on the contextual CMAB. There they use context features and the context determine the base arm
% 	expected outcome, but in our case,
% 	we assume that the base arms' expect outcomes remain the same, and only the feasible set of super arms may change I guess. 
% 	Please verify and add this part.}

\subsection{Monotonicity and Triggering Probability Modulated Condition}
In order to guarantee the theoretical regret bounds, we consider two conditions given in~\citep{wang2017improving}. The first one is monotonicity, which is stated below.
\begin{restatable}{condition}{condOverMonotonicity}(Monotonicity). \label{cond:Monotonicity}
We say that a C$^2$MAB-T problem instance satisfies monotonicity, if for any action $S$, for any two expectation vectors $\bm{\mu} = (\mu_1,\dots,\mu_m)$ and $\bm{\mu}' = (\mu_1',\dots,\mu_m')$, we have $r_S(\bm{\mu}) \leq r_S(\bm{\mu}')$ if $\mu_i \leq \mu_i'$ for all $i \in [m]$.
\end{restatable}
The second condition is Triggering Probability Modulated (TPM) Bounded Smoothness. We use $p_{i}^{S}(\bm{\mu})$ to denote the probability that the action $S$ triggers arm $i$ when the expectation vector is $\bm{\mu}$. The TPM condition in C$^2$MAB-T is given below.
% \begin{condition}
% We say that an OCIM problem instance satisfies 1-norm TPM bounded smoothness, if there exists $B\in R^+$ (referred as the bounded smoothness constant) such that, for any two environment distributions with expectation vectors $\bm{\mu}$ and $\bm{\mu}'$, and any joint action $S$, we have $|r_{S}(\bm{\mu}) - r_{S}(\bm{\mu}')| \leq B \sum_{i\in[m]}p_i^{\bm{\mu},S}|\mu_i-\mu_i'|$,
% \end{condition}
\begin{restatable}{condition}{condOverTPM}(1-Norm TPM bounded smoothness). \label{cond:generalTPM}
We say that a C$^2$MAB-T problem instance satisfies 1-norm TPM bounded smoothness, if there exists $C\in \mathbb{R}^+$ (referred as the bounded smoothness coefficient) such that, for any two expectation vectors $\bm{\mu}$ and $\bm{\mu}'$, and any action $S$, we have $|r_{S}(\bm{\mu}) - r_{S}(\bm{\mu}')| \leq C \sum_{i\in[m]}p_i^{S}(\bm{\mu})|\mu_i-\mu_i'|$.
\end{restatable}

\subsection{Regret Bounds with Monotonicity} \label{general_regret_monotone}
\begin{algorithm}[t]
 \caption{Contextual CUCB with offline oracle $\mathcal{O}$, C$^2$-UCB}\label{alg:C2UCB}
 \begin{algorithmic}[1]
 \STATE \textbf{Input}: $m$, Oracle $\mathcal{O}$.
 \STATE For each arm $i\in [m]$, $T_i\leftarrow 0$. \{maintain the total number of times arm $i$ is played so far.\}
 \STATE For each arm $i \in [m]$, $\hat{\mu}_i \leftarrow 1$. \{maintain the empirical mean of $X_i$.\}
 \FOR{$t = 1,2,3,\dots$}
    \STATE For each arm $i\in[m], \rho_i \leftarrow \sqrt{\frac{3\ln t}{2 T_i}}$. \{the confidence radius, $\rho_i = +\infty$ if $T_i = 0$.\}
    \STATE For each arm $i\in[m], \bar{\mu}_i = \min\{\hat{\mu}_i + \rho_i, 1\}$. \{the upper confidence bound.\}
    \STATE Obtain context $\bm{\mathcal{S}}^{(t)}$.
    \STATE $S^{(t)} \leftarrow \mathcal{O}(\bm{\mathcal{S}}^{(t)}, \bar{\mu}_1, \bar{\mu}_2, \dots, \bar{\mu}_m)$.
    \STATE Play action $S^{(t)}$, which triggers a set $\tau \subseteq [m]$ of base arms with feedback $X_i^{(t)}$'s, $i\in \tau$.
    \STATE For every $i\in \tau$ update $T_i$ and $\hat{\mu}_i$: $T_i = T_i + 1, \hat{\mu}_i = \hat{\mu}_i + (X_i^{(t)}-\hat{\mu}_i) / T_i$.
 \ENDFOR
 \end{algorithmic} 
\end{algorithm}
For the general C$^2$MAB-T problem that satisfies both monotonicity (Condition \ref{cond:Monotonicity}) and TPM bounded smoothness (Condition \ref{cond:generalTPM}), we introduce a contextual version of the CUCB algorithm~\citep{wang2017improving}, which is described in Algorithm~\ref{alg:C2UCB}. Recall that $\bm{\mathcal{S}}^{(t)}$ is the action space in round $t$. We define the reward gap $\Delta_{S}^{(t)}{=}\max(0, \alpha \cdot \text{opt}^{(t)}(\bm{\mu}) - r_S(\bm{\mu}))$ for all actions $S \in \bm{\mathcal{S}}^{(t)}$. 
For each arm $i$, we define $\Delta^{i,T}_{\min} = \min_{t\in[T]}\inf_{S\in\mathcal{S}^{(t)}:p_i^S(\bm{\mu}) > 0, \Delta_{S}^{(t)} > 0} \Delta_{S}^{(t)}$ and $\Delta^{i,T}_{\max} = \max_{t\in[T]}\sup_{S\in\bm{\mathcal{S}}^{(t)}:p_i^S(\bm{\mu}) > 0, \Delta_{S}^{(t)} > 0} \Delta_{S}^{(t)}$. 
If there is no action $S$ such that $p_i^S(\bm{\mu}) > 0$ and $\Delta_{S}^{(t)} > 0$, we define $\Delta^{i,T}_{\min} = +\infty$ and $\Delta^{i,T}_{\max} = 0$. We define $\Delta^{(T)}_{\min} = \min_{i \in [m]}\Delta^{i,T}_{\min}$ and $\Delta^{(T)}_{\max} = \max_{i \in [m]}\Delta^{i,T}_{\max}$.
Let $\widetilde{S} = \{i\in [m] \mid p_i^S(\bm{\mu}) > 0\}$ be the set of arms that can be triggered by $S$. We define $K = \max_{S\in\bm{\mathcal{S}}^{(t)}}|\widetilde{S}|$ as the largest number of arms could be triggered by a feasible action. We use $\lceil x \rceil_0$ to denote $\max\{\lceil x \rceil, 0\}$. Contextual CUCB (C$^2$-UCB) has the following regret bounds.
\begin{restatable}{theorem}{thmOverCUCB}\label{thm:CUCB}
For the Contextual CUCB algorithm C$^2$-UCB (Algorithm~\ref{alg:C2UCB}) on an C$^2$MAB-T problem satisfying 1-norm TPM bounded smoothness (Condition \ref{cond:generalTPM}) with bounded smoothness constant $C$, (1) if $\Delta^{(T)}_{\min} > 0$, we have a distribution-dependent bound
\begin{align}\textstyle
    \text{Reg}_{\alpha,\beta}(T; \bm{\mu}) \leq \sum_{i \in [m]} \frac{576C^2K\ln T}{\Delta^{i,T}_{\min}}  + 4Cm +\sum_{i \in [m]} \left(\left\lceil \log_2 \frac{2CK}{\Delta^{i,T}_{\min}}\right\rceil_0 + 2\right)\cdot \frac{\pi^2}{6}\cdot\Delta^{(T)}_{\max},
\end{align}
and (2) we have a distribution-independent bound
\begin{align}\textstyle
\nonumber
    \text{Reg}_{\alpha,\beta}(T; \bm{\mu}) \leq 12C\sqrt{mKT\ln T}+ 2Cm + \left(\left\lceil \log_2 \frac{T}{18\ln T}\right\rceil_0 + 2\right)\cdot m \cdot \frac{\pi^2}{6}\cdot\Delta^{(T)}_{\max}.
\end{align}\end{restatable}
\begin{proof}
We first show that Lemma 5 in~\citep{wang2017improving} still holds for Contextual CUCB algorithm in the C$^2$MAB-T problem. Let $\mathcal{N}_{t}^{\text{s}}$ be the event that at the beginning of round $t$, for every arm $i \in [m]$, $|\hat{\mu}_{i,t} - \mu_{i}| \leq \rho_{i,t}$. Let $\mathcal{H}_t$ be the event that at round $t$ oracle $\mathcal{O}$ fails to output an $\alpha$-approximate solution. In Lemma 5 from~\citep{wang2017improving}, it assumes that $\mathcal{N}_{t}^{\text{s}}$ and $\neg \mathcal{H}_t$ hold, then we have
\begin{equation}
    r_{S^{(t)}}(\bar{\bm{\mu}}_{t}) \geq \alpha \cdot \text{opt}^{(t)}(\bar{\bm{\mu}}_{t}) \geq  \alpha \cdot \text{opt}^{(t)}(\bm{\mu}) = r_{S^{(t)}}(\bm{\mu}) + \Delta^{(t)}_{S^{(t)}}.
\end{equation}
By the TPM condition, we have
\begin{equation}
\Delta^{(t)}_{S^{(t)}} \leq  r_{S^{(t)}}(\bar{\bm{\mu}}_{t})  -  r_{S^{(t)}}(\bm{\mu})  \leq C \sum_{i\in [m]} p_i^{S^{(t)}}(\bm{\mu}) |\bar{\mu}_{i,t} - \mu_{i}|,
\end{equation}
which is in the same form of Eq.(10) in~\citep{wang2017improving}. Hence, we can follow the remaining proof of its Lemma 5. With Lemma 5, we can follow the proof of Lemma 6 in~\citep{wang2017improving} to bound the regret when $\Delta^{(t)}_{S^{(t)}}\geq M_{S^{(t)}}$, where $M_{S^{(t)}} = \max_{i\in\tilde{S}^{(t)}} M_i$ and $M_i$ is a positive real number for each arm $i$. Finally, we take $M_i = \Delta^{i,T}_{\min}$. If $\Delta^{(t)}_{S^{(t)}}< M_{S^{(t)}}$, then $\Delta^{(t)}_{S^{(t)}} = 0$, since we have either $\tilde{S}^{(t)} =  \emptyset$ or $\Delta^{(t)}_{S^{(t)}} < M_{S^{(t)}} \leq M_i$ for some $i \in \tilde{S}^{(t)}$. Thus, no regret is accumulated when $\Delta^{(t)}_{S^{(t)}}< M_{S^{(t)}}$. Following Eq.(17)-(22) in~\citep{wang2017improving}, we can derive the distribution-dependent and distribution-independent regret bounds shown in the theorem. 
\end{proof}

\subsection{Regret Bounds without Monotonicity}
\begin{algorithm}[t]
 \caption{C$^2$-TS with offline oracle $\mathcal{O}$}\label{alg:C2-TS}
 \begin{algorithmic}[1]
 \STATE \textbf{Input}: $m$, Prior $\mathcal{Q}$, Oracle $\mathcal{O}$.
  \STATE \textbf{Initialize} Posterior $\mathcal{Q}_1 = \mathcal{Q}$
%  \STATE For all arm $i\in [m]$, $\mu_i \sim Beta(a_i, b_i)$.
 \FOR{$t = 1,2,3,\dots$}
    \STATE Draw a sample $\bm{\mu}^{(t)}$ from $\mathcal{Q}_t$.
    \STATE Obtain context $\bm{\mathcal{S}}^{(t)}$
    \STATE $S^{(t)} \leftarrow {\mathcal{O}}(\bm{\mathcal{S}}^{(t)}, \bm{\mu}^{(t)})$.
    \STATE Play action $S^{(t)}$, which triggers a set $\tau \subseteq [m]$ of base arms with feedback $X_i^{(t)}$'s, $i\in \tau$.
    \STATE Update posterior $\mathcal{Q}_{t+1}$ using $X_i^{(t)}$ for all $i\in \tau$.
 \ENDFOR
 \end{algorithmic} 
\end{algorithm}

\begin{algorithm}[t]
 \caption{C$^2$-OFU with offline oracle $\widetilde{\mathcal{O}}$}\label{alg:C2-OFU}
 \begin{algorithmic}[1]
 \STATE \textbf{Input}: $m$, Oracle $\widetilde{\mathcal{O}}$.
 \STATE For each arm $i\in [m]$, $T_i\leftarrow 0$. \{maintain the total number of times arm $i$ is played so far.\}
 \STATE For each arm $i \in [m]$, $\hat{\mu}_i \leftarrow 1$. \{maintain the empirical mean of $X_i$.\}
 \FOR{$t = 1,2,3,\dots$}
    \STATE For each arm $i\in[m], \rho_i \leftarrow \sqrt{\frac{3\ln t}{2 T_i}}$. \{the confidence radius, $\rho_i = +\infty$ if $T_i = 0$.\}
    \STATE For each arm $i\in[m], c_i \leftarrow \left[(\hat{\mu}_i - \rho_i)^{0+}, (\hat{\mu}_i + \rho_i)^{1-}\right]$. \{the estimated range of $\mu_i$.\}
    \STATE Obtain context $\bm{\mathcal{S}}^{(t)}$.
    \STATE $S^{(t)} \leftarrow \widetilde{\mathcal{O}}(\bm{\mathcal{S}}^{(t)}, c_1, c_2, \dots, c_m)$.
    \STATE Play action $S^{(t)}$, which triggers a set $\tau \subseteq [m]$ of base arms with feedback $X_i^{(t)}$'s, $i\in \tau$.
    \STATE For every $i\in \tau$ update $T_i$ and $\hat{\mu}_i$: $T_i = T_i + 1, \hat{\mu}_i = \hat{\mu}_i + (X_i^{(t)}-\hat{\mu}_i) / T_i$.
 \ENDFOR
 \end{algorithmic} 
\end{algorithm}

\begin{algorithm}
 \caption{C$^2$-ETC with offline oracle $\mathcal{O}$}
 \begin{algorithmic}[1]\label{alg:C2-ETC}
 \STATE \textbf{Input}: $m$, $k$, $N$, $T$, Oracle $\mathcal{O}$.
 \STATE For each arm $i$, $T_i\leftarrow 0$. \{maintain the total number of times arm $i$ is played so far.\}
 \STATE For each arm $i$, $\hat{\mu}_i \leftarrow 0$. \{maintain the empirical mean of $X_i$.\}
 \STATE \textbf{Exploration phase}:
  \FOR{$t = 1,2,3,\dots, \lceil mN/k \rceil$}
    \STATE Obtain context $\bm{\mathcal{S}}^{(t)}$.
    \STATE Play action $S^{(t)} \in \bm{\mathcal{S}}^{(t)}$, which contains $k$ base arms that have not been chosen for $N$ times.
    \STATE Observe the feedback $X_i^{(t)}$ for each base arm in $S^{(t)}$, $i \in \tau_{\text{direct}}$.
    \STATE For each arm $i\in \tau_{\text{direct}}$ update $T_i$ and $\hat{\mu}_i$: $T_i = T_i + 1, \hat{\mu}_i = \hat{\mu}_i + (X_i^{(t)}-\hat{\mu}_i) / T_i$.
 \ENDFOR 
    \STATE \textbf{Exploitation phase}:
 \FOR{$t =  \lceil mN/k \rceil+1,\dots, T$}
    \STATE Obtain context $\bm{\mathcal{S}}^{(t)}$.
    \STATE $S^{(t)} \leftarrow \mathcal{O}(\bm{\mathcal{S}}^{(t)}, \hat{\mu}_1, \hat{\mu}_2, \dots, \hat{\mu}_m)$.
    \STATE Play action $S^{(t)}$.
 \ENDFOR
 \end{algorithmic} 
\end{algorithm}

As discussed in Section~\ref{sec:properties}, OCIM is an example of C$^2$MAB-T that satisfies the TPM condition but not monotonicity. For the general C$^2$MAB-T problem without monotonicity, we proposed two algorithms, C$^2$-TS, C$^2$-OFU, that can still achieve logarithmic Bayesian and frequentist regrets respectively. 
We also present C$^2$-ETC that has a tradeoff between feedback requirement and regret bound.

C$^2$-TS is described in Algorithm~\ref{alg:C2-TS}. Different from OCIM-TS, we input a general prior $\mathcal{Q}$ (which depends on $\mathcal{D}$ and might not be Beta distributions anymore) and update the posterior distribution $\mathcal{Q}_t$ accordingly. With the same definitions in \ref{general_regret_monotone} and $\delta^{(T)}_{\max} = \max_{\bm{\mu}} \Delta^{(T)}_{\max}$, it has the following Bayesian regret bound.
\begin{restatable}{theorem}{thmOverC2TS}\label{thm:C2-TS}
For the C$^2$-TS (Algorithm~\ref{alg:C2-TS}) on an C$^2$MAB-T problem satisfying 1-norm TPM bounded smoothness (Condition \ref{cond:generalTPM}) with bounded smoothness constant $C$, we have the Bayesian regret bound
\begin{align}\textstyle
    \text{BayesReg}_{\alpha,\beta}(T) \leq 12C\sqrt{mKT\ln T}+ 2Cm +\left(\left\lceil \log_2 \frac{T}{18\ln T}\right\rceil_0 + 4\right)\cdot m \cdot \frac{\pi^2}{6}\cdot\delta^{(T)}_{\max},
\end{align}.
\end{restatable}

C$^2$-OFU is described in Algorithm~\ref{alg:C2-OFU}. Similar to OCIM-OFU, it requires an offline oracle $\tilde{\mathcal{O}}$ that takes the context $\bm{\mathcal{S}}^{(t)}$ and $c_i$'s (ranges of $\mu_i$'s) as inputs and outputs an approximate solution $S^{(t)}$. With such an oracle, C$^2$-OFU has the following frequentist regret bounds. 
\begin{restatable}{theorem}{thmOverC2OFU}\label{thm:C2-OFU}
For the C$^2$-OFU (Algorithm~\ref{alg:C2-OFU}) on an C$^2$MAB-T problem satisfying 1-norm TPM bounded smoothness (Condition \ref{cond:generalTPM}) with bounded smoothness constant $C$, (1) if $\Delta^{(T)}_{\min} > 0$, we have a distribution-dependent bound
\begin{align}\textstyle
    \text{Reg}_{\alpha,\beta}(T; \bm{\mu}) \leq \sum_{i \in [m]} \frac{576C^2K\ln T}{\Delta^{i,T}_{\min}}  + 4Cm +\sum_{i \in [m]} \left(\left\lceil \log_2 \frac{2CK}{\Delta^{i,T}_{\min}}\right\rceil_0 + 2\right)\cdot \frac{\pi^2}{6}\cdot\Delta^{(T)}_{\max},
\end{align}
and (2) we have a distribution-independent bound
\begin{align}\textstyle
\nonumber
    \text{Reg}_{\alpha,\beta}(T; \bm{\mu}) \leq 12C\sqrt{mKT\ln T}+ 2Cm + \left(\left\lceil \log_2 \frac{T}{18\ln T}\right\rceil_0 + 2\right)\cdot m \cdot \frac{\pi^2}{6}\cdot\Delta^{(T)}_{\max}.
\end{align}\end{restatable}

Besides C$^2$-TS and C$^2$-OFU, we also provide a general explore-then-commit algorithm C$^2$-ETC, as described in Algorithm~\ref{alg:C2-ETC}. 
In the general setting, $\tau_{\text{direct}}$ is defined as the set of base arms that is deterministically triggered by the action in question.
C$^2$-ETC is simple and only requires feedback from directly triggered arms, but it has a worse regret bound and requires the following condition besides 
	Condition~\ref{cond:generalTPM}. 
%	it requires that every $S^{(t)}$ can directly trigger any $k$ different base arms, which is defined as the following condition.
% \wei{Reworded this paragraph and the condition below.
% 	I think C$^2$-ETC does not require monotonicity, right? please check.}
\begin{restatable}{condition}{condOverETC} \label{cond:ETC}
For some $k\ge 1$, given any context $\bm{\mathcal{S}}^{(t)}$ and any set $S'\subseteq [m]$ with $|S'| = k$, there exists $S \in \bm{\mathcal{S}}^{(t)}$ such that $p_i^{S}(\bm{\mu}) = 1$ for every $i\in |S'|$.
\end{restatable}
With such a condition, C$^2$-ETC has the following frequentist regret bounds.
\begin{restatable}{theorem}{thmOverC2ETC}\label{thm:C2-ETC}
For the C$^2$-ETC (Algorithm~\ref{alg:C2-ETC}) on an C$^2$MAB-T problem satisfying Condition~\ref{cond:ETC} and 1-norm TPM bounded smoothness (Condition \ref{cond:generalTPM}) with bounded smoothness constant $C$, (1) if $\Delta^{(T)}_{\min} > 0$, when $N = \max\left\{1, \frac{2 C^2 m^2}{(\Delta^{(T)}_{\min})^2}\ln(\frac{kT(\Delta^{(T)}_{\min})^2}{C^3 m})\right\}$, we have a distribution-dependent bound
\begin{align}\textstyle
    \text{Reg}_{\alpha,\beta}(T; \bm{\mu}) \leq \frac{m}{k} \Delta^{(T)}_{\max} + \frac{2C^2 m^3 \Delta^{(T)}_{\max}}{k (\Delta^{(T)}_{\min})^2}\left(\max\left\{ \ln\left(\frac{kT(\Delta^{(T)}_{\min})^2}{C^2 m^2}\right), 0\right\} + 1\right)
\end{align}
and (2) when $N = (Ck)^{\frac{2}{3}} m^{-\frac{2}{3}} T^{\frac{2}{3}} (\ln T)^{\frac{1}{3}}$, we have a distribution-independent bound
% \begin{equation}
%     \text{Reg}_{\bm{\mu},\alpha,\beta}(T) \leq 3.9 B^{\frac{2}{3}} m^{\frac{4}{3}} n^{\frac{2}{3}} k^{-\frac{1}{3}} T^{\frac{2}{3}} + 1
% \end{equation}
\begin{equation}\textstyle
     Reg_{\alpha,\beta}(T;\bm{\mu}) \leq O(C^{\frac{2}{3}} m^{\frac{4}{3}} k^{-\frac{1}{3}} T^{\frac{2}{3}} (\ln T )^{\frac{1}{3}}).
\end{equation}
\end{restatable}
The proofs of Theorem~\ref{thm:C2-TS}, \ref{thm:C2-OFU} and~\ref{thm:C2-ETC}  generally follow the same steps in Appendix~\ref{appendix:TS}, \ref{appendix:TS:OFU} and \ref{appendix:ETC}.

% \wei{The above paragraph still gives the impression that we are using OCIM algorithms, and discuss within the OCIM context. I understand that the general algorithms are almost
% 	the same as these OCIM versions, but it would be strange for later work to refer to say OCIM-OFU to solve their different problems. Moreover, there are still minor differences:
% 	1. For TS, we need a general prior distribution and a general procedure to compute the posterior.
% 	2. For OFU and ETC, we need a general context, not $S_B$ in our OCIM algorithms. 
% 	Moreover, the theorem statements would be different. For example for OFU, our current Theorem~\ref{thm:OIFU} uses the constant $\tilde{C}$, which is specific to our OCIM problem, but the general theorem requires the general bounded smoothness constant.
	
% 	My suggestion is follows:
% 	1. List all three algorithm pseudo codes again here, with the above general modifications. We can name the algorithms as C$^2$-TS, C$^2$-OFU, C$^2$-ETC.
% 	2. List all general theorems again for the three general algorithms. In the theorem statements, explicitly refer to the assumptions (including the bounded smoothness assumptions, and the assumptions on the oracles) in the statement.
% 	Then later papers could refer to these general algorithms and general theorems, not the theorems on OCIM. 
% 	I think it is ok if most part of the algorithms and theorems are repetitive.
% }

\end{document}